%% file: main.tex
\documentclass[10pt, hidelinks, letterpaper]{article}
\usepackage[margin=1in]{geometry}

\usepackage{amsmath}
\usepackage{amssymb}
\usepackage{amsthm}
\usepackage{setspace}
\usepackage{bm, bbm}
\usepackage{enumitem}
\usepackage{mathrsfs}
\usepackage{hyperref}
\usepackage{xcolor}

\usepackage{amsmath}
\usepackage{algorithm}
\usepackage{algpseudocode}
\usepackage{amsfonts}

\newcommand{\R}{\mathbb{R}}

\newcommand{\rd}{\mathrm d}

\newcommand{\bR}{\mathbb R}
\newcommand{\bE}{\mathbb E}
\newcommand{\bP}{\mathbb P}
\newcommand{\Var}{\mathrm{Var}}

\let \phi \varphi

\DeclareMathOperator{\Uniform}{Uniform}

\DeclareMathOperator{\dTV}{d_{TV}}
\DeclareMathOperator{\poly}{poly}
\DeclareMathOperator{\supp}{supp}

\DeclareMathOperator*{\dKL}{d_{KL}}

\newcommand{\vv}[1]{{\boldsymbol{#1}}}
\newcommand{\diff}{\mathrm{d}}
\newcommand{\E}{\mathbb{E}}
\def\bbb#1\eee{\begin{align}#1\end{align}}
\def\bb#1\ee{\begin{align*}#1\end{align*}}

\newtheorem{theorem}{Theorem}
\newtheorem{lemma}{Lemma}

\newtheorem{proposition}{Proposition}

\theoremstyle{definition}
\newtheorem{remark}{Remark}

\hypersetup{
    colorlinks=true,
    linkcolor=blue,
    filecolor=blue,      
    urlcolor=cyan,
    citecolor=red
}

\title{From optimal score matching to optimal sampling}
\author{Zehao Dou\footnote{Department of Statistics and Data Science, Yale University} \and Subhodh Kotekal\footnote{Department of Statistics, University of Chicago} \and Zhehao Xu\(^*\) \and Harrison H. Zhou\(^*\)}
\date{}

\begin{document}
    \maketitle

    \begin{abstract}
    \input{abstract}
    \end{abstract}
    
    \input{intro.tex}

    \section{Methodology}\label{section:upper_bound}
    \input{upper/upper_bound_methodology_text.tex}
    \input{density_estimation/met}

    \section{Main results}
    Upper bounds for the minimax risk of score estimation are presented in Section \ref{section:main_result_upper_bound}. Section \ref{section:main_result_lower_bound} states matching lower bounds. Finally, Section \ref{section:main_result_density_estimation} states that our resulting estimator for the density \(f\) indeed achieves the optimal rate under total variation and Wasserstein-\(1\) losses. 
    \subsection{Score estimation upper bound}
    \label{section:main_result_upper_bound}
    \input{upper/upper_bound_main_result_text}

    \subsection{Score estimation lower bound}
    \label{section:main_result_lower_bound}
    \input{lower/lower_bound_section3_text}

    \subsection{Distribution estimation}\label{section:main_result_density_estimation}

\input{density_estimation/DE}

    \section{Proof sketch}
    Sections \ref{section:proof_sketch_upper_bound} and \ref{section:proof_sketch_lower_bound} give the proof arguments for the upper and lower bounds respectively. Finally, Section \ref{section:proof_sketch_density_estimation} describes some finer points associated to deriving error bounds for the density estimator obtained from our score estimator.

    \subsection{Upper bound}
    \label{section:proof_sketch_upper_bound}
    \input{upper/upper_bound_sketch_text}

    \subsection{Lower bound}
    \label{section:proof_sketch_lower_bound}
    \input{lower/lower_bound_sketch_text}

    \subsection{Distribution estimation}
    \label{section:proof_sketch_density_estimation}
    \input{density_estimation/proof_sketch}

    \input{discussion}

    \bibliographystyle{skotekal.bst}
    \bibliography{score-matching2}

    \newpage
    \appendix 
    \section{Proof details for upper bound}\label{appendix:upper_bound_proofs}
    \input{setting}
    \subsection{Some useful lemmas}
    \label{sec:useful}
    \input{upper/useful-lemmas}
    \subsection{Very high noise regime}\label{section:very-high-noise}
    \input{upper/very-high-noise}
    \subsection{High noise regime}\label{section:high-noise}
    \input{upper/high-noise}
    \subsection{Low noise regime: kernel based estimator}\label{section:low-noise_kernel}
    \input{upper/low-noise}
    \subsection{Low noise regime: data-free estimator}\label{section:low-noise_datafree}
    \label{sec:A.4}
    \input{upper/data-free}

    \section{Proof details for lower bounds}\label{appendix:lower_bound_proofs}
    \input{lower/lower_bound_proof}

    \input{lower/lower_bound_large_t}

    \input{density_estimation/density_estimation_proof}

    \input{multivariate}

\end{document}

%% file: abstract.tex
The recent, impressive advances in algorithmic generation of high-fidelity image, audio, and video are largely due to great successes in score-based diffusion models. A key implementing step is score matching, that is, the estimation of the score function of the forward diffusion process from training data. As shown in earlier literature, the total variation distance between the law of a sample generated from the trained diffusion model and the ground truth distribution can be controlled by the score matching risk. 

Despite the widespread use of score-based diffusion models, basic theoretical questions concerning exact optimal statistical rates for score estimation and its application to density estimation remain open. We establish the sharp minimax rate of score estimation for smooth, compactly supported densities. Formally, given \(n\) i.i.d. samples from an unknown \(\alpha\)-H\"{o}lder density \(f\) supported on \([-1, 1]\), we prove the minimax rate of estimating the score function of the diffused distribution \(f * \mathcal{N}(0, t)\) with respect to the score matching loss is \(\frac{1}{nt^2} \wedge \frac{1}{nt^{3/2}} \wedge (t^{\alpha-1} + n^{-2(\alpha-1)/(2\alpha+1)})\) for all \(\alpha > 0\) and \(t \ge 0\). As a consequence, it is shown the law \(\hat{f}\) of a sample generated from the diffusion model achieves the sharp minimax rate \(\bE(\dTV(\hat{f}, f)^2) \lesssim n^{-2\alpha/(2\alpha+1)}\) for all \(\alpha > 0\) without any extraneous logarithmic terms which are prevalent in the literature, and without the need for early stopping which has been required for all existing procedures to the best of our knowledge. 

Optimal score estimation is shown to be achieved by kernel-based estimation, and the minimax lower bound is proved through an implementation of Fano's method with a subtle, nonstandard construction to facilitate an application of the heat equation, which may be of independent interest.  

%% file: intro.tex
\section{Introduction}
In the recent machine learning literature, score-based generative algorithms have been empirically shown to produce state-of-the-art samples from highly structured, complex probability distributions across a range of different modalities including natural images \cite{song_generative_2019,song_improved_2020,ho_denoising_2020,dhariwal_diffusion_2021} and audio \cite{kong_diffwave_2021,chen_wavegrad_2020}. At a high level, the core idea of this type of generative modeling is to incrementally corrupt clean training data through some noisy mechanism with a gradually increasing noise scale at each forward step, and to progressively estimate a denoiser at each backward step. Ultimately the fitted denoiser is deployed to convert a fresh draw of pure noise into a new sample of an approximate distribution of the target probability distribution. 

A popular choice of noise mechanism is a diffusion process (as in the paradigm of denoising diffusion probabilistic modeling \cite{sohl-dickstein_deep_2015,ho_denoising_2020}), in which case the optimal denoiser is obtained by solving the stochastic differential equation corresponding to the reverse diffusion process. The drift term of this reverse diffusion involves the \emph{score function} (that is, the gradient of the logarithm of the probability density function) of the marginal distribution of the diffused data at each noise increment. A particularly popular approach for estimating the score function is \emph{score matching} \cite{hyvarinen_estimation_2005,vincent_connection_2011,song_generative_2019}. Though diffusion models trained via score matching have seen much empirical success as attested by the literature, theoretical optimalities are scarce. In this paper, we establish sharp minimax rates of score estimation with respect to the score matching loss for \(\alpha\)-H\"{o}lder densities. As a consequence, we show diffusion models can achieve, without extraneous logarithmic factors and without early stopping, the sharp minimax rate of density estimation and are thus, in this sense, optimal for sampling. 

\subsection{Background on diffusion models}
Consider a probability density function \(f\) on \(\R\) representing the target distribution from which we wish to generate a sample. In diffusion modeling, a forward diffusion process is posited as the solution to a particular stochastic differential equation (SDE). Consider a generic SDE for the forward process with an unknown initialization with density \(f\),  
\begin{equation}\label{eqn:forward_process}
    \rd X_t = g(X_t, t) \, \rd t + \sigma(t) \, \rd W_t, \;\;\; X_0 \sim f
\end{equation}
where \(g : \R \times [0, \infty) \to \R\) is the drift, \(\sigma : [0, \infty) \to [0, \infty)\) is the diffusion coefficient, and \(\{W_t\}_{t \ge 0}\) is the standard Wiener process. The drift and diffusion coefficient are chosen by the practitioner. For fixed \(T > 0\), it is known \cite{anderson_reverse-time_1982} that, under some conditions, the reverse process \(\tilde{Y}_t := X_{T-t}\) is a solution to the following SDE for \(0 \le t \le T\),  
\begin{equation}\label{eqn:backward_process}
    \rd\tilde{Y}_t = \left(-g(\tilde{Y}_t, T-t) + \sigma^2(T-t) s(\tilde{Y}_t, T-t)\right) \, \rd t + \sigma(T-t) \, \rd W_t, \;\;\; \tilde{Y}_0 \sim p(\cdot, T)
\end{equation}
where \(p(\cdot, t)\) denotes the density of the marginal distribution of \(X_t\), and the function
\begin{equation}\label{def:score}
    s(x, t) := \frac{\partial}{\partial x}\log p(x, t)
\end{equation}
is referred to as the \emph{score function} of the density \(p(\cdot, t)\). The methodological utility of diffusion modeling is that, for various choices of \(g\) and \(\sigma\) in (\ref{eqn:forward_process}), the distribution \(p(\cdot, T)\) for large \(T\) is close to a known distribution regardless of the unknown density \(f\). For example, in an Ornstein-Uhlenbeck process with the choice \(g(x, t) = -x\) and \(\sigma(t) = \sqrt{2}\) (also referred to as a \emph{variance preserving SDE} \cite{song_score-based_2020}), the distribution \(p(\cdot, t)\) converges to \(\mathcal{N}(0, 1)\) as \(t \to \infty\). Therefore, to produce a sample that is approximately distributed according to an unknown \(f\), the practitioner can generate a fresh draw from \(\mathcal{N}(0, 1)\) for initialization and run the backwards process (\ref{eqn:backward_process}) with an estimator \(\hat{s}(x, t)\) of the score function constructed from existing training data. An alternative to the variance preserving SDE is a \emph{variance exploding SDE} \cite{song_score-based_2020} which corresponds to the choice \(g \equiv 0\). Thus, the statistical task is to estimate the score function \(s(\cdot, t)\) at every time \(t > 0\) given access to i.i.d. training data drawn from \(f\). This is typically done with respect to the score matching loss \cite{song_generative_2019,ho_denoising_2020,sohl-dickstein_deep_2015,hyvarinen_estimation_2005,vincent_connection_2011},  
\begin{equation}\label{def:score_matching_loss}
    \int_{-\infty}^{\infty} \left|\hat{s}(x, t) - s(x, t)\right|^2 \, p(x, t)\, \rd x.
\end{equation}
The score matching loss is not only natural but also implies an error bound in total variation distance for the density estimator associated to diffusion modelling. For \(T > 0\), let \(\{\hat{Y}_t\}_{t \in [0, T]}\) denote the process solving the reverse SDE (\ref{eqn:backward_process}) with \(\hat{s}\) plugged in to the drift term and with some known initialization \(\pi_0\) instead of \(p(\cdot, T)\). Let \(\hat{X}_0 = \hat{Y}_T\). As argued in \cite{oko2023diffusion,chen_sampling_2023}, the error in total variation distance can be bounded, under some conditions, via an application of Girsanov's theorem,
\begin{equation}\label{eqn:dtv_score_bound}
    \dTV(\hat{X}_0, X_0)^2 \lesssim \dTV\left(p(\cdot, T), \pi_0\right)^2 + \int_{0}^{T} \int_{-\infty}^{\infty} |\hat{s}(x, t) - s(x, t)|^2 \, p(x, t) \, \rd x \, \rd t.  
\end{equation}
Since we can directly obtain samples with law \(\hat{X}_0\), we thus naturally speak of optimal sampling when referring to optimal density estimation via diffusion modeling. The score matching loss is also relevant to convergence results for sampling algorithms involving discretizations of various dynamics \cite{chen_sampling_2023,lee_convergence_2022,block_generative_2022,lee_convergence_2023}. For a detailed review of diffusion models and recent advances, we refer the reader to \cite{chen_overview_2024,tang_score-based_2024,yang_diffusion_2023}.

\subsection{Related work}
Given the theoretical focus of this article, the discussion in this section is limited to the recent progress in formal, provable guarantees for score estimation. The results of the early work \cite{block_generative_2022} on score matching are applicable to the variance exploding SDE corresponding to \(g \equiv 0\) and \(\sigma \equiv 1\) in (\ref{eqn:forward_process}). Block et al. \cite{block_generative_2022} invoke a relationship between score matching and denoising auto-encoders due to \cite{vincent_connection_2011} which equivalently casts the score matching problem as a supervised regression (i.e. denoising) problem. This correspondence can also be viewed as a consequence of Tweedie's well-known formula \cite{efron_tweedies_2011}. With this relationship in hand, Block et al. \cite{block_generative_2022} consider empirical risk minimization using a large, suitably expressive function class \(\mathcal{F}\), such as a class of neural networks. Assuming the score function \(\frac{\rd}{\rd x} \log f\) of the data generating distribution is Lipschitz (and satisfies some additional regularity conditions), they obtain the upper bound \( \frac{\poly(\log n)}{t^2} \left(\mathscr{R}^2_n(\mathcal{F}) + \frac{1}{n}\right)\) with high probability on the score matching loss (\ref{def:score_matching_loss}), provided \(t\) is bounded. Here, \(\mathscr{R}_n(\mathcal{F})\) denotes the Rademacher complexity of the function class \(\mathcal{F}\). 

Though the work of Block et al. \cite{block_generative_2022} is very relevant to practice as it concerns the empirical risk minimization paradigm, their result suffers from some drawbacks. The appearance of the Rademacher complexity can misleadingly suggest slow rates; it may imply a nonparametric rate (e.g. if \(\mathcal{F}\) is chosen to be a H\"{o}lder ball) even though the target, diffused density \(p(\cdot, t)\) might be an analytic function and should be, in principle, estimable at nearly parametric rates. Similarly, it is quite undesirable to assume \(f\) has a Lipschitz score. Already as \(t\) increases, the forward process (\ref{eqn:forward_process}) increasingly smooths \(f\) to yield the marginal density \(p(\cdot, t)\), and understanding this regularization effect is a major goal.

If \(f\) is assumed to have a Lipschitz score function, then one can look to directly estimate the score \(\frac{\rd}{\rd x} \log f\) of the data generating distribution. Wibisono et al. \cite{wibisono_optimal_2024}, pointing out the connection between the score matching loss and regret in empirical Bayes, consider an estimator inspired by smoothing in empirical Bayes methodology. They show, assuming \(f\) is subgaussian in addition to having a Lipschitz score, the minimax rate is given (up to a logarithmic factor) by \(n^{-2/5}\). As the authors of \cite{wibisono_optimal_2024} note, the rate \(n^{-2/5}\) is quite intuitive as it matches the rate \(n^{-\frac{2(\alpha - r)}{2\alpha+1}}\) of estimating the \(r\)th derivative of an \(\alpha\)-smooth density with \(r = 1\) and \(\alpha = 2\) as shown in \cite{stone_optimal_1982,stone_optimal_1983}. As a consequence of their result, Wibisono et al. \cite{wibisono_optimal_2024} obtain the following error bound for score estimation in the Ornstein-Uhlenbeck diffusion model with \(g(x, t) = -x\) and \(\sigma \equiv \sqrt{2}\) in (\ref{eqn:forward_process}),
\begin{equation}\label{eqn:wibisono_score}
    \int_{-\infty}^{\infty} |\hat{s}(x, t) - s(x, t)|^2 \, p(x, t) \, \rd x \lesssim \frac{\poly(\log n, \log t^{-1})}{nt^{3/2}}, 
\end{equation}
which holds with high probability when \(\frac{\poly(\log n)}{n^2} \lesssim t \lesssim 1\). Notably, a near parametric rate (with respect to \(n\)) of convergence can be achieved in comparison to the result of \cite{block_generative_2022}. However, no matching lower bound (even up to logarithmic factors) is offered in \cite{wibisono_optimal_2024}. Wibisono et al. also obtain results in the case where \(f\) admits a \(\beta\)-H\"{o}lder score function with \(\beta \le 1\). 

The work \cite{zhang_minimax_2024} (concurrent to \cite{wibisono_optimal_2024}) dispenses with the assumption of \cite{wibisono_optimal_2024} that \(f\) has a Lipschitz score function. Noting \(s(x, t) = \frac{\frac{\partial}{\partial x} p(x, t)}{p(x, t)}\), they estimate the numerator and denominator separately by kernel density estimators. However, there is a possibility with kernel-based density estimators of taking the value zero at some points, which is problematic when plugging into the denominator to form a score estimator. Consequently, Zhang et al. \cite{zhang_minimax_2024} introduce a truncation parameter and trivially estimate the score by zero at any point where  the kernel density estimator is below the truncation level; Zhang et al. \cite{zhang_minimax_2024} obtain an error bound similar to that of \cite{wibisono_optimal_2024}. Additionally, they consider the case \(1 \lesssim t \lesssim \poly(n)\) and obtain an upper bound of order \(\frac{\poly(\log n)}{nt}\). However, they also do not offer any minimax lower bound, let alone a matching one. Moreover, under the assumption \(f\) is an \(\alpha\)-Sobolev density, their score estimation result implies the nearly optimal upper bound \(\dTV(\hat{f}, f)^2 \lesssim \poly(\log n) n^{-\frac{2\alpha}{2\alpha+1}}\), but only when \(\alpha \le 2\) (see Theorem 3.7 in \cite{zhang_minimax_2024}). 

However, they rely on a popular early stopping trick in which the the reverse process \(\tilde{X}_t\) from (\ref{eqn:backward_process}) is sampled only up until \(T - \underline{T}\) instead of \(T\), for some choice of early stopping time \(\underline{T}\). Consequently, the lower limit of integration \(t = 0\) in (\ref{eqn:dtv_score_bound}) is replaced by the early stopping time \(t = \underline{T}\). To obtain the near optimal upper bound, Zhang et al. \cite{zhang_minimax_2024} must select \(\underline{T} \asymp n^{-\frac{2}{2\alpha+1}}\). In our view, early stopping is quite ad hoc and unprincipled; the cynic might put forth the charge that the choice of \(\underline{T}\) made by \cite{zhang_minimax_2024} was entirely driven by whatever value ensures \(\int_{\underline{T}}^{1} \frac{1}{nt^{3/2}} \,\rd t\lesssim n^{-\frac{2\alpha}{2\alpha+1}}\) so that (near) optimality can be claimed from (\ref{eqn:dtv_score_bound}). To satisfy the cynic, a quality affirmative answer to whether score-based diffusion models can be optimal should thus avoid early stopping, or at least avoid ad hoc choices. Furthermore, early stopping is typically adopted in the literature due to the suboptimality of existing score estimators for small \(t\); consequently, extraneous logarithmic terms appear in the resulting density estimation bounds. 

The impressive work \cite{oko2023diffusion} investigates score estimation by empirical risk minimization under the assumption \(f\) has its support contained in \([-1, 1]\), is bounded below by a universal constant on its support, is infinitely smooth at the support's boundary, and lies in a Besov space with smoothness index \(\alpha > 0\). Oko et al. \cite{oko2023diffusion} study a large class of deep neural networks \(\mathcal{F}\) and an Ornstein-Uhlenbeck diffusion model with \(g(x, t) = -\beta_t x\) and \(\sigma(t) = \sqrt{2\beta_t}\) in (\ref{eqn:forward_process}) for some variance schedule \(\left\{\beta_t\right\}_{t \ge 0}\). Without imposing additional conditions on the score function, they obtain the upper bound 
\begin{equation}\label{eqn:oko_score}
    \bE \left(\int_{-\infty}^{\infty} |\hat{s}(x, t) - s(x, t)|^2 \, p(x, t) \, \rd x\right) \lesssim \frac{\poly(\log n) \cdot n^{-\frac{2\alpha}{2\alpha+1}}}{t}.
\end{equation}
Oko et al. \cite{oko2023diffusion} also employ early stopping to obtain the bound \(\dTV(\hat{f}, f)^2 \lesssim \poly(\log n) \cdot n^{-\frac{2\alpha}{2\alpha+1}}\), but their choice \(\underline{T} = \poly(n^{-1})\) is more palatable as it is less ad hoc than the choice of \cite{zhang_minimax_2024}. However, no matching lower bound for (\ref{eqn:oko_score}) is obtained in \cite{oko2023diffusion}.

\subsection{Main contributions}
This article studies score estimation in the variance exploding SDE given by \(g \equiv 0\) and \(\sigma \equiv 1\) in (\ref{eqn:forward_process}), in which case the density of \(X_t\) is given by \(p(x, t) = (\varphi_t * f)(x)\) where \(\varphi_t(x) = \frac{1}{\sqrt{2\pi t}}e^{-\frac{x^2}{2t}}\) is the density of \(\mathcal{N}(0, t)\). The data generating density \(f\) is assumed to be supported on \([-1, 1]\) and, on its support, is assumed \(\alpha\)-H\"{o}lder and bounded below by a constant. To develop the minimax theory, the parameter space is formally defined as follows. As in \cite{tsybakov_introduction_2009}, for \(\alpha, L > 0\), define the H\"{o}lder class of probability density functions, 
\begin{align}
    \begin{split}\label{def:Holder}
    \mathcal{H}_\alpha(L) = &\left\{ f : [-1, 1] \to [0, \infty) : \int_{-1}^{1} f(x) \, \rd x = 1, f \text{ is continuous}, \right. \\
    &\;\;\;\left. f \text{ is } \lfloor \alpha \rfloor \text{-times differentiable on } (-1, 1)\text{,} \right. \\
    &\;\;\; \left. \text{and } \left|f^{\lfloor \alpha\rfloor}(x) - f^{\lfloor \alpha \rfloor}(y)\right| \le L |x-y|^{\alpha - \lfloor \alpha \rfloor} \text{ for all } x, y \in (-1, 1) \right\}.
    \end{split}
\end{align}
This article focuses on the case where \(L > 0\) is some large universal constant, and the notational dependence on \(L\) will be frequently suppressed by writing \(\mathcal{H}_\alpha\). Define the parameter space
\begin{equation}\label{def:param}
    \mathcal{F}_\alpha := \left\{ f : \R \to [0, \infty) : \supp(f) \subset [-1, 1], f|_{[-1, 1]}\in \mathcal{H}_\alpha, \text{ and } c_d \le f(x) \le C_d \text{ for } |x| \le 1 \right\}    
\end{equation}
where \(C_d, c_d > 0\) are some universal constants. Our main contribution is to establish the sharp minimax rate of score estimation,
\begin{equation}\label{eqn:rate}
    \inf_{\hat{s}} \sup_{f \in \mathcal{F}_\alpha} \bE \left(\int_{-\infty}^{\infty} |\hat{s}(x, t) - s(x, t)|^2 \, p(x, t) \, \rd x\right) \asymp \frac{1}{nt^2} \wedge \frac{1}{nt^{3/2}} \wedge \left(n^{-\frac{2(\alpha-1)}{2\alpha+1}} + t^{\alpha - 1}\right). 
\end{equation}
A number of remarks are in order. 

\begin{remark}
    The rate (\ref{eqn:rate}) exhibits three regimes. In the very high noise regime \(t \gtrsim 1\), the rate is of order \(\frac{1}{nt^2}\), which is faster than the upper bound \(\frac{1}{nt}\) obtained by \cite{zhang_minimax_2024}. In the high noise regime \(n^{-\frac{2}{2\alpha+1}} \lesssim t \lesssim 1\), the rate is of order \(\frac{1}{nt^{3/2}}\). This result implies that the upper bound (\ref{eqn:wibisono_score}) obtained by \cite{wibisono_optimal_2024,zhang_minimax_2024} is nearly sharp. A faster rate \(n^{-\frac{2(\alpha-1)}{2\alpha+1}} + t^{\alpha-1}\) than (\ref{eqn:oko_score}) obtained by \cite{oko2023diffusion} can be achieved in the low noise regime \(t \lesssim n^{-\frac{2}{2\alpha+1}}\) by exploiting the smoothness of \(f\).
\end{remark}

\begin{remark}
    Our result applies for all \(\alpha > 0\), including the case \(\alpha < 1\) when the underlying density \(f\) need not even be differentiable and thus need not have a score function. The diffusion (\ref{eqn:forward_process}) bestows regularity on \(p(x, t) = (\varphi_t * f)(x)\), which is infinitely differentiable regardless of the number, or absence, of derivatives of \(f\). When \(\alpha < 1\), the minimax rate (\ref{eqn:rate}) reduces to \(\frac{1}{nt^2} \wedge \frac{1}{nt^{3/2}} \wedge t^{\alpha - 1}\), which specializes to \(t^{\alpha - 1}\) in the low noise regime \(t \lesssim n^{-\frac{2}{2\alpha+1}}\) and is achieved by a trivial estimator which ignores the available data. However, right when the H\"{o}lder exponent crosses the critical threshold \(\alpha > 1\), the rate (\ref{eqn:rate}) in the low noise regime specializes to \(n^{-\frac{2(\alpha-1)}{2\alpha+1}}\), implying consistent estimation is immediately possible once \(f\) admits a score function. 
\end{remark}

\begin{remark}\label{remark:boundary}
    From (\ref{eqn:dtv_score_bound}), the diffusion model estimator \(\hat{f}\) furnished from our score estimator achieves, provided the initialization \(\pi_0\) satisfies \(\dTV(p(\cdot, T), \pi_0)^2 \lesssim n^{-\frac{2\alpha}{2\alpha+1}}\) and the choice \(T \gtrsim 1\) is made,
    \begin{align*}
        \bE \left(\dTV(\hat{f}, f)^2\right) &\lesssim \int_{0}^{n^{-\frac{2}{2\alpha+1}}} \left(n^{-\frac{2(\alpha - 1)}{2\alpha+1}} + t^{\alpha - 1}\right) \, \rd t +  \int_{n^{-\frac{2}{2\alpha+1}}}^{T \wedge 1} \frac{1}{nt^{3/2}} \, \rd t + \int_{T \wedge 1}^{T \vee 1} \frac{1}{nt^2} \, \rd t\asymp  n^{-\frac{2\alpha}{2\alpha+1}}. 
    \end{align*}
    Therefore, \(\hat{f}\) achieves the sharp minimax rate of density estimation \emph{without} any extraneous logarithmic terms in contrast to previous work. Furthermore, early stopping as in \cite{oko2023diffusion,zhang_minimax_2024,wibisono_optimal_2024} is not necessary for diffusion models to be optimal. Note further that one can take \(T\) arbitrarily large since \(\int_{1}^{T} \frac{1}{nt^2} \, \rd t\lesssim \frac{1}{n}\). To the best of our knowledge, the existing state-of-the-art bound in the literature for unbounded \(t\), namely \(\frac{\poly(\log n)}{nt}\) due to \cite{zhang_minimax_2024}, does not allow us to take arbitrarily large \(T\) since \(\int_{1}^{T} \frac{1}{nt} \, \rd t = \frac{\log T}{n}\). In particular, the existing bound requires the constraint \(T \lesssim e^{Cn^{\frac{2}{2\alpha+1}}}\) in order for \(\hat{f}\) to achieve the sharp rate.
\end{remark}

\begin{remark}
    A density \(f \in \mathcal{F}_\alpha\) is discontinuous at the boundary of \([-1, 1]\) since \(f(1), f(-1) \ge c_d\) and \(f(x) = 0\) for \(|x| > 1\). Specifically, \(f\) is not assumed to be smooth at the boundary like \cite{oko2023diffusion}, who require a high order of smoothness, higher than \(\alpha\), at the boundary. Consequently, \(p(x, t)\) exhibits a comparative lack of regularity near the boundary compared to the internal region of \([-1, 1]\). Thus, the expert reader might expect this to be reflected in the minimax rate (\ref{eqn:rate}) by the inclusion of some additional term. However, as discussed in Section \ref{section:proof_sketch_upper_bound}, the score function is conveniently self-normalized, \(s(x, t) = \frac{\frac{\partial}{\partial x} p(x, t)}{p(x, t)}\), and it turns out this property prevents the lack of regularity at the boundary to cause a slowdown in the estimation rate. In contrast, the boundary can be shown to negatively impact the estimation rate of the unnormalized derivative \(\frac{\partial}{\partial x} p(x, t)\). 
\end{remark}

%% file: upper/upper_bound_methodology_text.tex
To describe our methodology, first note the score function can be written as $s(x,t)=\frac{\psi(x,t)}{p(x,t)}$ where $\psi(x,t) = \frac{\partial}{\partial x} p(x,t)$ is the spatial derivative of the density. Note \(p(x, t) = (\varphi_t * f)(x)\) and so \(\psi(x, t) = (\varphi_t' * f)(x)\). We employ different methodology depending on the magnitude of the noise level \(t\). Sections \ref{section:methodology_very_high_noise}, \ref{section:methodology_high_noise}, and \ref{section:methodology_low_noise} address score estimation in the very high noise, high noise, and low noise regimes respectively. Section \ref{section:methodology_density_estimation} describes our estimator for the data-generating distribution \(f\) constructed from our score estimators. Throughout this section, denote the available data as \(\mu_1,...,\mu_n \overset{iid}{\sim} f\).

\subsection{Very high noise regime}\label{section:methodology_very_high_noise}
In this section, we define a score estimator to achieve the rate \(\frac{1}{nt^2}\) which dominates in (\ref{eqn:rate}) in the very high noise regime \(t \gtrsim 1\). A natural idea is to define estimators for \(\psi\) and \(p\) separately, and then plug them into the definition of \(s\) to obtain an estimator \(\hat{s}\). Since \(p\) appears in the denominator, the loss can be too large if the estimator of \(p\) is too close to \(0\), which needs to be avoided in the score estimation.

To overcome this technical obstacle, a regularization trick is widely applied in various existing works, namely forcing the denominator to always be larger than some regularization level \(\varepsilon\), which is typically chosen as $\varepsilon = n^{-2}$ (e.g. \cite{wibisono_optimal_2024,jiang_general_2009}). Through this regularization, Wibisono et al. \cite{wibisono_optimal_2024} further control the score matching risk by the Hellinger distance but incur additional $\poly(\log n)$ factors. To avoid these possibly unnecessary logarithmic terms, we employ a regularizer which is a function of $x$ instead of one which is constant in $x$. Specifically, we use \(\varepsilon(x, t) := c_d \int_{-1}^{1} \varphi_t(x-\mu) \, \rd \mu\). 

We now define the density estimator \(\hat{p}(x, t) := \varepsilon(x, t) \vee \left(\frac{1}{n} \sum_{j=1}^{n} \varphi_t(x-\mu_j)\right)\). Next, observe 
\begin{equation*}
    \psi(x, t) = (\varphi_t' * f)(x) = \int_{-1}^{1} -\frac{x-\mu}{t} \varphi_t(x-\mu) f(\mu) \, \rd \mu = - \frac{x}{t} p(x, t) + \frac{1}{t} \int_{-1}^{1} \mu\varphi_t(x-\mu) f(\mu) \, \rd \mu. 
\end{equation*}
Therefore, a natural estimator for \(\psi\) is \(\hat{\psi}(x, t) = -\frac{x}{t} \hat{p}(x, t) + \frac{1}{nt} \sum_{j=1}^{n} \mu_j \varphi_t(x-\mu_j)\). We employ the score estimator
\begin{equation}\label{def:s_hat_large_t}
    \hat{s}(x, t) = \frac{\hat{\psi}(x, t)}{\hat{p}(x, t)}
\end{equation}
in the very high noise regime.

\subsection{High noise regime}\label{section:methodology_high_noise}
In the high noise regime \(n^{-\frac{2}{2\alpha+1}} \lesssim t \lesssim 1\), a different idea for constructing a score estimator is to consider an unbiased estimator of $\psi(x,t)$. Unbiased estimators $\hat{\psi}(x,t)$ and $\hat{p}(x,t)$ can be obtained by replacing $f$ in \(\psi\) and \(p\) with the empirical distribution $\frac1n\sum_{j=1}^n \delta_{\mu_j}$, namely \(\hat{\psi}(x, t) = \frac1n \sum_{j=1}^n \phi_t'(x-\mu_j)\), and \(\hat{p}(x, t) = \frac1n \sum_{j=1}^n \phi_t(x-\mu_j)\). Again, regularization is needed, so we define the score estimator 
\begin{equation}\label{def:shat_high}
    \hat{s}(x,t) = \frac{\hat{\psi}(x,t)}{\hat{p}(x,t)\vee \varepsilon(x,t)}
\end{equation}
with the choice $\varepsilon(x,t) = c_d\int_{-1}^1 \phi_t(x-\mu) \,\rd \mu $. Though this estimator turns out to achieve the sharp rate \(\frac{1}{nt^{3/2}}\) in the high noise regime as stated in Section \ref{section:main_result_upper_bound}, this estimation strategy is not completely satisfactory since the bound diverges as the noise level $t\rightarrow 0$. A naive application of (\ref{eqn:dtv_score_bound}), even with early stopping incorporated, will yield a suboptimal convergence rate instead of the sharp \(n^{-\frac{2\alpha}{2\alpha+1}}\) scaling without any logarithmic terms. Indeed, this is the essence of the shared shortcomings of existing works \cite{oko2023diffusion,block_generative_2022,wibisono_optimal_2024,zhang_minimax_2024}. A different strategy must be employed for the low noise regime.

\subsection{Low noise regime}\label{section:methodology_low_noise}
In the low noise regime \(t \lesssim n^{-\frac{2}{2\alpha+1}}\), the optimal rate when \(\alpha \ge 1\) is \(n^{-\frac{2(\alpha-1)}{2\alpha+1}}\) as seen in (\ref{eqn:rate}). To achieve this rate, we develop a kernel-based score estimator. As before, the derivative \(\psi\) and the density \(p\) are separately estimated, but now the smoothness of \(f\) is explicitly exploited by way of kernel smoothing. Define the estimator

\begin{equation}\label{def:kde}
\hat{f}(x) = \frac{1}{nh} \sum_{j=1}^n \left[K_{\lfloor\alpha\rfloor}((\mu_j-x)/h)\cdot \mathbbm {1}_{\{x<0\}} + K_{\lfloor\alpha\rfloor}((x-\mu_j)/h)\cdot \mathbbm 1_{\{x\ge0\}}\right]
\end{equation}
where $K_{\lfloor\alpha\rfloor}$ is an $\lfloor\alpha\rfloor$-order kernel \cite{li2005convergence}. The score estimator is inspired by the direct plugin of $\hat{f}$, with some regularization applied. The denominator poses more of a challenge in the low noise regime than it does in the high noise regime; consequently, we employ different estimation strategies depending on the location of \(x\). 

We consider three different regions: the internal part \(D_1 := \{x \in \R : |x| < 1 - \sqrt{C t\log(1/t)}\}\), the boundary part \(D_2 := \{x \in \R : 1 - \sqrt{C t\log(1/t)} \le |x| \le 1 + C\sqrt{t}\}\), and the external part \(D_3 := \{x \in \R : |x| > 1 + C\sqrt{t}\}\), where \(C > 0\) is a constant. When $x \in D_1$, we define
\begin{equation}\label{def:shat_D1}
\hat{s}(x,t) = \frac{\hat{\psi}(x,t)}{\hat{p}(x,t) \vee \varepsilon(x,t)}
\end{equation}
where we have the density derivative estimator \(\hat{\psi}(x,t) = \phi_t(x+1)\hat{f}(-1)-\phi_t(x-1)\hat{f}(1)+\int_{-1}^{1}\phi_t(x-\mu)\hat{f}'(\mu) \,\rd\mu\), the density estimator $\hat{p}(x,t)=\phi_t* \hat{f}(x)$, and the regularizer $\varepsilon(x,t)=c_d\int_{-1}^{1} \phi_t(x-\mu)\,\rd \mu$. When $x \in D_2 \cup D_3$, we define
\begin{equation}\label{def:shat_D2}
\hat{s}(x, t) = \frac{\hat{\psi}(x,t)}{\phi_t* \hat{f}(x)} \mathbbm{1}_{\Omega} + \frac{\phi_t' * u(x)}{\phi_t * u(x)} \mathbbm{1}_{\Omega^c}
\end{equation}
where the event \(\Omega := \left\{ \hat{f}(x) \ge \frac{c_d}{2} \text{ for all } x \in [-1, 1]\right\}\), $u(x) = \frac12 \mathbbm{1}_{\left\{|x|\le 1\right\}},$ and \(\hat{\psi}(x,t) = \phi_t(x+1)\hat{f}(-1)-\phi_t(x-1)\hat{f}(1)+\int_{-1}^{1} \phi_t(x-\mu)\hat{f}'(\mu) \,\rd\mu\). Note $\|\hat{f}-f\|_\infty < \frac{c_d}{2}$ with high probability \cite{gine2021mathematical}, and so it follows from \(f \ge c_d\) that \(\Omega^c\) is a low-probability event. This strategy is employed since the data-generating density \(f\) is discontinuous at the boundary of its support; consequently, \(\varphi_t * f\) exhibits comparatively less regularity on \(D_2\) compared to, say, \(D_1\). As hinted in Remark \ref{remark:boundary}, it turns out the self-normalizing property of the score function enables fast estimation at the boundary without imposing additional smoothness assumptions as done in previous work \cite{oko2023diffusion}. We construct (\ref{def:shat_D2}) to exploit this convenient property; see Section \ref{section:proof_sketch_upper_bound} for an explanation.

When $\alpha < 1$, the data generating density $f$ is not necessarily differentiable, which precludes the use of arguments from the \(\alpha \ge 1\) case which make use \(f\)'s differentiability. In this case, we can show optimal rates can be achieved by using a trivial density estimator with $\hat{f}(x)=\frac12 \mathbbm{1}_{\{|x|\le 1\}}$ and its corresponding score estimator with $\hat{s}(x,t) = \frac{\left(\phi_t'* \hat{f}\right)(x)}{\left(\phi_t* \hat{f}\right)(x)}$.

%% file: density_estimation/met.tex
\subsection{Distribution estimation}\label{section:methodology_density_estimation}

Having constructed a score estimator, we now develop an estimator of \(f\) by solving for the backward process in (\ref{eqn:backward_process}) corresponding to our choice of a variance exploding forward process in (\ref{eqn:forward_process}). Recall, we consider the forward process $\vv X = \{X_t: t \in [0,T]\}$ which solves
\begin{align}
\label{eq:forward}
\diff X_t = \diff W_t, ~~ X_0 \sim f.
\end{align}
Denote its reverse process by $\vv Y = \{Y_t: t \in [0,T]\}$ which solves (\ref{eqn:backward_process}), that is,
\begin{align}
\label{eq:backward}
\diff Y_t = s(Y_t,T-t)\, \diff t+\diff W_t,~~Y_0 \sim p(\cdot,T).
\end{align}
It can be shown that $X_t = X_0 + \sqrt{t}Z$ and $Y_t = X_{T-t}$, where $Z \sim N(0, 1)$ is independent of $X_0$, are solutions to the forward and reverse SDEs.

Let \(T \asymp n\). Define the process \(\widehat{\vv{Y}} = \{\widehat{Y}_t : t \in [0,T]\}\), which satisfies
\begin{align}
\label{eq:hat_forward}
\diff \widehat{Y}_t = \widehat{s}(\widehat{Y}_t, T-t) \diff t + \diff W_t, \quad \widehat{Y}_0 \sim \mathcal{N}(0,T).
\end{align}
Our estimator for \(f\) (i.e., the distribution of \(X_0\)) is given by the law of \(\widehat{X}_0 := \widehat{Y}_T \mathbbm{1}_{\left\{|\widehat{Y}_T| \leq 1\right\}}\). In our estimation approach, we propose using \(\widehat{s}\), as discussed in previous subsections of Section~\ref{section:upper_bound}, to approximate the true score function \(s\). Additionally, we approximate the distribution \(p(\cdot,T)\) with \(N(0,T)\). It follows by Lemma~\ref{lem:KLappro} that the Kullback-Leibler divergence between these two distributions is upper bounded by \(1/(2T)\).

\begin{algorithm}
\caption{}\label{algo:DE}
\begin{algorithmic}[1]
\State \textbf{Input:} Data $\{\mu_i\}_{i=1}^{n}$ drawn i.i.d. from \(f\).

\State Calculate $\widehat{s}(\cdot,t)$ as proposed in Sections~\ref{section:methodology_very_high_noise}, \ref{section:methodology_high_noise}, and \ref{section:methodology_low_noise} for all \(t \in [0, T]\) with \(T := n\).

\State Solve the SDE,
\begin{equation*}
\diff \widehat{Y}_t = \widehat s(\widehat{Y}_t,T-t) \, \diff t + \diff W_t
\end{equation*}
with initialization $\widehat{Y}_0 \sim \mathcal{N}(0,T)$.
\State \textbf{Output:} $\widehat{X}_0 := \widehat Y_T \mathbbm{1}_{\left\{|\widehat Y_T| \le 1\right\}}$.
\end{algorithmic}
\end{algorithm}

Although our final estimator may not possess a density, the difference between our distribution estimation and density estimation is not fundamental. For example, convolving the estimator with a small Gaussian noise yields a smooth density without significantly altering the original distribution. The error introduced by the convolution can be made negligible, dominated by the distribution estimation error, ensuring the estimator remains consistent with traditional density estimation.

%% file: upper/upper_bound_main_result_text.tex
For score estimation with respect to the score matching loss (\ref{def:score_matching_loss}), we establish the following minimax upper bound for any H\"{o}lder exponent $\alpha > 0$ and time $t \ge 0$. 
\begin{theorem}
\label{thm:score_upperbound}
Let $\alpha > 0 $. There exists a constant $C = C(\alpha, L)$ depending only on $\alpha$ and $L$ such that
\begin{equation*}
\inf_{\hat{s}} \sup_{f \in \mathcal{F}_\alpha} \bE\left(\int_{-\infty}^{\infty} \left|\hat{s}(x, t) - s(x, t)\right|^2 \, p(x, t) \,\rd x \right) \le C\left(\frac{1}{nt^2} \wedge \frac{1}{nt^{3/2}} \wedge \left(n^{-\frac{2(\alpha-1)}{2\alpha+1}} + t^{\alpha-1}\right)\right)
\end{equation*}
for all $t \ge 0$. 
\end{theorem}
Theorem \ref{thm:score_upperbound} relies on employing the estimation methodology outlined in Section \ref{section:upper_bound}, and is a direct consequence of Theorems \ref{thm:upperbound_large_t}, \ref{thm:up-high-1}, \ref{thm:up-1}, and \ref{thm:up-2}. 

\subsubsection{Very high noise regime}
In the very high noise regime \(t \gtrsim 1\), Theorem \ref{thm:upperbound_large_t} establishes that the score estimator given by (\ref{def:s_hat_large_t}) achieves the rate \(\frac{1}{nt^2}\).
\begin{theorem}\label{thm:upperbound_large_t}
    Let \(c > 0\) and \(\hat{s}\) be given by (\ref{def:s_hat_large_t}). There exists a constant \(C\) depending only on \(c\) such that  
    \begin{equation*}
        \sup_{f \in \mathcal{F}_\alpha} \bE\left(\int_{-\infty}^{\infty} |\hat{s}(x, t) - s(x, t)|^2 \, p(x, t) \, \rd x\right) \le \frac{C}{nt^2}
    \end{equation*}
    for \(t \ge c\). 
\end{theorem}
\noindent The result follows from a straightforward examination of the score matching loss in Section \ref{section:very_high_noise_sketch}. 

\subsubsection{High noise regime}\label{section:main_high_noise_upperbound}
It can be shown the estimator \(\hat{s}\) given by (\ref{def:shat_high}) achieves the rate \(\frac{1}{nt^{3/2}}\), which is optimal in the high noise regime \(n^{-\frac{2}{2\alpha+1}} \lesssim t \lesssim 1\). 
\begin{theorem}
\label{thm:up-high-1}
If \(t \le 1\), the score estimator \(\hat{s}\) given by (\ref{def:shat_high}) achieves the error
\[\sup_{f \in \mathcal{F}_\alpha} \bE\left(\int_\bR \left|\hat{s}(x,t)-s(x,t)\right|^2  p(x,t) \,\rd x\right) \lesssim \frac{1}{nt^{3/2}}.\]
\end{theorem}
The error from estimating the derivative \(\psi\) dominates the error from estimating \(p\) in the rate. This can be explained by the mean squared errors \(\bE\left(\int_{-\infty}^{\infty} |\hat{\psi}(x, t) - \psi(x, t)|^2 \, \rd x\right) \lesssim \frac{1}{nt^{3/2}}\) and \(\bE\left(\int_{-\infty}^{\infty} |\hat{p}(x, t) - p(x, t)|^2 \, \rd x\right) \lesssim \frac{1}{nt^{1/2}}\). Its rigorous analysis involves some technical maneuvering due to the regularization in the denominator of (\ref{def:shat_high}). 

\subsubsection{Low noise regime}
In the low noise regime \(t \lesssim n^{-\frac{2}{2\alpha+1}}\) with \(\alpha \ge 1\), we use the estimator \(\hat{s}(x, t)\) given by (\ref{def:shat_D1}) for \(x \in D_1\) and (\ref{def:shat_D2}) for \(x \in D_2 \cup D_3\). 
\begin{theorem}
\label{thm:up-1}
If \(\alpha \ge 1\) and \(t < n^{-\frac{2}{2\alpha+1}}\), then 
\begin{equation*}
\hat{s}(x, t) := \hat{s}_1(x, t) \mathbbm{1}_{\{x \in D_1\}} + \hat{s}_2(x, t) \mathbbm{1}_{\{x \in D_2\}} + \hat{s}_3(x, t) \mathbbm{1}_{\{x \in D_3\}}
\end{equation*}
satisfies
\[\sup_{f \in \mathcal{F}_\alpha} \bE\left(\int_\bR \left|\hat{s}(x,t)-s(x,t)\right|^2 p(x,t) \, \rd x\right) \lesssim n^{-\frac{2(\alpha-1)}{2\alpha+1}}. \]
Here, \(\hat{s}_1\) is given by (\ref{def:shat_D1}), and both \(\hat{s}_2\) and \(\hat{s}_3\) are given by (\ref{def:shat_D2}). 
\end{theorem}

A key ingredient is that, according to \cite{gine2021mathematical, li2005convergence}, the kernel density estimator \(\hat{f}\) given by (\ref{def:kde}) with the bandwidth choice \(h \asymp n^{-\frac{1}{2\alpha+1}}\) satisfies \(\bE\left(|\hat{f}^{(k)}(x) - f^{(k)}(x)|^2\right) \lesssim n^{-2(\alpha-k)/(2\alpha+1)}\) for \(x \in [-1, 1]\) and \(k = 0,1,...,\lfloor \alpha \rfloor\). The proof to bound the score matching risk of \(\hat{s}\) on \(D_1\) is very similar to the high noise regime and the upper bound \(n^{-2(\alpha-1)/(2\alpha+1)}\) can be obtained since the density derivative estimation error dominates the other terms. Bounding the score matching risk on \(D_2 \cup D_3\) is substantially more involved due to the comparative lack of regularity from the discontinuity of \(f\) at its boundary. As noted in Section \ref{section:upper_bound}, we exploit the self-normalizing from of the score function to achieve the sharp rate without assuming additional smoothness, which is different from \cite{oko2023diffusion}. A technical discussion can be found in Section \ref{section:proof_sketch_upper_bound}.

Finally, if $0 < \alpha < 1$, it can be shown that a trivial estimator can achieve the rate \(t^{\alpha-1}\), which is the dominating term of the optimal rate (\ref{eqn:rate}) for \(0 < \alpha < 1\) and \(t < n^{-\frac{2}{2\alpha+1}}\).

\begin{theorem}
\label{thm:up-2}
If $\alpha < 1$ and \(t < n^{-\frac{2}{2\alpha+1}}\), then 
\[\sup_{f \in \mathcal{F}_\alpha} \bE \left(\int_\bR \left|\hat{s}(x,t)-s(x,t)\right|^2 p(x,t) \,\rd x \right) \lesssim t^{\alpha-1},\]
where \(\hat{s}(x, t) = \frac{(\varphi_t' * \hat{f})(x)}{(\varphi_t * \hat{f})(x)}\) with \(\hat{f}(x) = \frac{1}{2}\mathbbm{1}_{\{|x| \le 1\}}\).
\end{theorem}
\noindent Theorem \ref{thm:up-2} is proved in Appendix \ref{section:low-noise_datafree}.

%% file: lower/lower_bound_section3_text.tex
The following theorem establishes the minimax lower bound \(\frac{1}{nt^{3/2}} \wedge (n^{-\frac{2(\alpha-1)}{2\alpha+1}} + t^{\alpha-1})\) which dominates in the rate (\ref{eqn:rate}) for small \(t\). 
\begin{theorem}\label{thm:score_lowerbound}
    If \(\alpha > 0\), then there exists positive constants \(c_1 = c_1(\alpha, L)\) and \(c_2 = c_2(\alpha, L)\) depending only on \(\alpha, L\) such that 
    \begin{equation*}
        \inf_{\hat{s}} \sup_{f \in \mathcal{F}_\alpha} \bE \left(\int_{-\infty}^{\infty} \left|\hat{s}(x, t) - s(x, t)\right|^2 \, p(x, t) \, \rd x \right) \ge c_1\left(\frac{1}{nt^{3/2}} \wedge \left(n^{-\frac{2(\alpha-1)}{2\alpha+1}} + t^{\alpha-1}\right)\right)
    \end{equation*}
    for \(t \le c_2\). 
\end{theorem}
Our high-level strategy is to use Fano's method, which is a standard approach for proving minimax lower bounds in nonparametric statistics, particularly in density estimation. Our lower bound construction involves a subtle modification to the standard construction used to establish minimax lower bounds in that literature \cite{tsybakov_introduction_2009}. We consider densities of the form 
\begin{equation}\label{def:fano_submodel}
    f_b(\mu) = \frac{1}{2}\mathbbm{1}_{\{|\mu| \le 1\}} + \epsilon^\alpha \sum_{i=1}^{m} b_i w\left(\frac{\mu - x_i}{\rho}\right)
\end{equation}
where \(0 < \rho < 1\) is a parameter to be chosen, \(\{x_i\}_{i=1}^{m}\) are grid points in \([-1, 1]\) with \(2\rho\) spacing, and \(b \in \mathcal{B}\) where \(\mathcal{B} \subset \{0, 1\}^m\) is a subset obtained as a consequence of the Gilbert-Varshamov bound (see Lemma 2.9 in \cite{tsybakov_introduction_2009}), namely it satisfies \(\log |\mathcal{B}| \asymp m\) and \(\min_{b \neq b' \in \mathcal{B}} d_{\text{Ham}}(b, b') \gtrsim m\). Additionally, \(w : \R \to \R\) is a function supported on \([-1, 1]\) with \(w \in C^\infty(\R)\), \(\int_{-\infty}^{\infty} w(x) \, \rd x = 0\), and whose first \(\lfloor \alpha \rfloor\) derivatives are bounded. 

The expert reader will recall that the standard construction in the literature corresponds to the special case \(\rho = \epsilon\). However, it turns out to be quite convenient to instead choose \(\rho \asymp \sqrt{t} \vee \epsilon\) for interesting mathematical reasons explained in Section \ref{section:proof_sketch_lower_bound}. In the course of the argument to obtain the information-theoretic lower bound of Theorem \ref{thm:score_lowerbound}, we make the choice
\begin{equation*}
    \epsilon \asymp (n\sqrt{t})^{-\frac{1}{2\alpha}} \wedge \begin{cases}
        n^{-\frac{1}{2\alpha+1}} &\textit{if } \alpha \ge 1, \\
        \sqrt{t} &\textit{if } \alpha < 1,
    \end{cases}
\end{equation*}
instead of the typical choice \(\epsilon \asymp n^{-\frac{1}{2\alpha+1}}\) in the usual density estimation lower bound proof \cite{tsybakov_introduction_2009}. 

The term \(\frac{1}{nt^2}\) dominates in (\ref{eqn:rate}) for large \(t\), and the following theorem establishes a minimax lower bound which matches the upper bound of Section \ref{section:main_result_upper_bound} in this regime. 
\begin{theorem}\label{thm:score_lowerbound_large_t}
    If \(\alpha > 0\) and \(c_2 > 0\), then there exists a positive constant \(c_1 = c_1(\alpha, L, c_2)\) depending only on \(\alpha, L,\) and \(c_2\) such that 
    \begin{equation*}
        \inf_{\hat{s}} \sup_{f \in \mathcal{F}_\alpha} \bE\left(\int_{-\infty}^{\infty} |\hat{s}(x, t) - s(x, t)|^2 \, p(x, t) \, \rd x \right) \ge \frac{c_1}{nt^2}
    \end{equation*}
    for \(t \ge c_2\). 
\end{theorem}
The lower bound proof is different from the proof of Theorem \ref{thm:score_lowerbound} as a two-point construction turns out to suffice. The construction is given by the choice \(f_0(\mu) = \frac{1}{2}\mathbbm{1}_{\{|\mu| \le 1\}}\) and \(f_1(\mu) = f_0(\mu) + \epsilon^\alpha w\left(\frac{\mu}{\rho}\right) \mathbbm{1}_{\{|\mu| \le 1\}}\) where \(w(x) = x\) and \(\epsilon \le \rho\) are free parameters to be chosen. In the course of the argument, it turns out the choices \(\rho \asymp 1\) and \(\epsilon \asymp n^{-\frac{1}{2\alpha}}\) can be made to obtain the lower bound stated in Theorem \ref{thm:score_lowerbound_large_t}.

%% file: density_estimation/DE.tex
We aim to bound the total variation (TV) distance and Wasserstein (\(\mathrm{W_1}\)) distance between $\widehat X_0$ given by Algorithm \ref{algo:DE} and $X_0$. The Wasserstein-1 distance uses the \(L_1\) norm as the cost metric. This will be done by examining the score matching error. The following theorem gives the upper bound for the total variation error for our estimator, which coincides with known minimax rates \cite{tsybakov_introduction_2009,oko2023diffusion}.

\begin{theorem}
\label{thm:TV}
Let \(\alpha > 0\). There exists a constant \(C = C(\alpha, L)\), as discussed in Section~\ref{section:discussion}, depending only on \(\alpha\) and \(L\) such that 
\begin{equation*}
    \sup_{f \in \mathcal{F}_\alpha} \mathbb{E}\left(\dTV(X_0, \widehat{X}_0)\right) \le Cn^{-\frac{\alpha}{2\alpha +1}}
\end{equation*}
where $\widehat{X}_0$ is given by Algorithm~\ref{algo:DE}.
\end{theorem}
Define the process $\widetilde{\vv 
Y} = \{\widetilde Y_t: t \in [0,T]\}$ that satisfies
\begin{align}
\label{eq:tilde_forward}
\diff \widetilde Y_t = \widehat s(\widetilde{Y}_t,T-t)\, \diff t+\diff W_t,~~ \widetilde Y_0 \sim p(\cdot,T).
\end{align}
The estimation error arises mainly from the differences between processes $\vv Y$ and $\widetilde{\vv Y}$. To help the analysis, we introduce the auxiliary process (\ref{eq:tilde_forward}). By the data processing inequality, we can bound the total variation between the final steps of (\ref{eq:tilde_forward}) and (\ref{eq:hat_forward}). We know that their first steps are \( p(\cdot,T) \) and \( \mathcal{N}(0,T) \), respectively, and that \( \dKL(p(\cdot,T)||\mathcal{N}(0,T)) \lesssim \frac{1}{T}\). Essentially, we compare the two SDEs (\ref{eq:backward}) and (\ref{eq:hat_forward}), which can show the same initial distribution but differ in their drift terms. A corollary of Girsanov's Theorem (Lemma~\ref{lemma:Girsanov}) is then applied, providing a bound on the Kullback-Leibler divergence between the path measures of (\ref{eq:backward}) and (\ref{eq:hat_forward}) by \( \int_0^T \int_{-\infty}^{\infty} |\hat{s}(x, t) - s(x, t)|^2 p(x, t) \, \rd x\, \rd t \), for which, we have the upper bounds in Section~\ref{section:main_result_upper_bound} for analysis.  It is worth noting that for $d >1$, the bound of Theorem \ref{thm:TV} becomes $\E\left(\dTV(X_0,\widehat X_0)\right) \lesssim n^{-\frac{\alpha}{2\alpha +d}}$.


\begin{theorem}
\label{thm:Wasserstein}
If \(\alpha \ge 1\), then there exists a constant \(C = C(\alpha, L)\), as discussed in Section~\ref{section:discussion}, depending only on \(\alpha\) and \(L\) such that
\begin{align*}
    \sup_{f \in \mathcal{F}_\alpha} \mathbb{E}\left(\mathrm{W_1}(X_0, \widehat{X}_0)\right) \le Cn^{-\frac{1}{2}}
\end{align*}
where \(\widehat{X}_0\) is given by Algorithm \ref{algo:DE}.
\end{theorem}

We obtain a smaller bound for $\mathrm{W_1}\) compared to $\dTV$, primarily because the total variation distance only measures the mass that needs to be transported from one distribution to the other. However, we can design the coupling so that the transport distance is also small. In the multivariate setting, it is shown in Section~\ref{section:discussion} that the bounds become $\mathbb{E}\left(\mathrm{W_1}(X_0, \widehat{X}_0)\right) \lesssim \log(n) n^{-\frac{1}{2}}$ for $d=2$ and $\mathbb{E}\left(\mathrm{W_1}(X_0, \widehat{X}_0)\right) \lesssim n^{-\frac{\alpha+1}{2\alpha+d}}$ for $d>2$. For $\alpha < 1$, we have the same bounds up to logarithmic terms for the errors in Wasserstein distance.

There are results from the literature that establish the minimax rates for the total variation and the Wasserstein-1 distances \cite{liang2017well,oko2023diffusion,MR4441130}. Specifically, for functions possessing \(\alpha\)-Hölder continuity, the minimax rate for the TV distance is given by \( n^{-\frac{\alpha}{2\alpha + d}} \), where \( d \) is the dimensionality of the sample space. This rate captures the trade-off between the smoothness of the function and the curse of dimensionality. In contrast, the minimax rate for the Wasserstein-1 distance is sharper, characterized by \( n^{-\frac{\alpha+1}{2\alpha + d}} \), highlighting the greater sensitivity of the Wasserstein metric to the geometric properties of the underlying distribution. We present that our estimator from Algorithm~\ref{algo:DE} can achieve the minimax rate in total variation, and it can also achieve the minimax rate in Wasserstein-1 distance if \(\alpha \geq 1\) and \( d \neq 2 \). It is worth noting that our upper bound \( \log(n) n^{-\frac{1}{2}} \) for \( d=2 \) coincides with the upper bound of Theorem~1 in \cite{MR4441130}, which is the best result to our knowledge for now.

%% file: upper/upper_bound_sketch_text.tex
The analysis of the upper bound can be split into analyses based on the noise level. Full proof details can be found in Appendix \ref{appendix:upper_bound_proofs}. 

\subsubsection{\texorpdfstring{Very high noise regime: \(t > 1\)}{Very high noise regime}}\label{section:very_high_noise_sketch}
The very high noise regime \(t > 1\) is addressed by Theorem \ref{thm:upperbound_large_t}. The argument for Theorem \ref{thm:upperbound_large_t} is conceptually straightforward. A formal proof with full details can be found in Appendix \ref{section:very-high-noise}. To gain intuition, recall the estimator \(\hat{s}(x, t)\) is given by (\ref{def:s_hat_large_t}). Since \(s(x,t) = \frac{\psi(x, t)}{p(x, t)}\), direct calculation yields
\begin{equation*}
    \int_{\R} |\hat{s}(x, t) - s(x, t)|^2 \, p(x, t) \, \rd x \le \int_{\R} \frac{|\hat{\psi}(x, t)p(x, t) - \psi(x, t)\hat{p}(x, t)|^2}{p(x,t) \varepsilon(x, t)^2} \, \rd x. 
\end{equation*}
Examining the numerator, since \(\hat{\psi}(x, t)p(x, t) = -\frac{x}{t} \hat{p}(x, t) p(x, t) + \left(\frac{1}{nt}\sum_{i=1}^{n} \mu_i\varphi_t(x-\mu_i)\right)p(x,t)\) and \(\psi(x, t)\hat{p}(x, t) = -\frac{x}{t}p(x, t)\hat{p}(x, t) + \left(\frac{1}{t} \int_{-1}^{1} \mu\varphi_t(x-\mu)f(\mu) \, \rd \mu\right)\hat{p}(x, t)\), it follows 
\begin{align*}
    |\hat{\psi}(x, t)p(x, t) - \psi(x, t) \hat{p}(x, t)|^2 \nonumber &\lesssim \left|\frac{1}{nt}\sum_{i=1}^{n} \mu_i \varphi_t(x-\mu_i) - \frac{1}{t} \int_{-1}^{1} \mu \varphi_t(x-\mu) f(\mu) \, \rd\mu \right|^2 p(x, t)^2 \nonumber\\
    &~~~+ \left|\frac{1}{t} \int_{-1}^{1} \mu \varphi_t(x-\mu) f(\mu) \, \rd\mu \right|^2 |\hat{p}(x, t) - p(x, t)|^2.
\end{align*}
A direct calculation shows   
\begin{align*}
    \bE \left|\frac{1}{nt}\sum_{i=1}^{n} \mu_i \varphi_t(x-\mu_i) - \frac{1}{t} \int_{-1}^{1} \mu \varphi_t(x-\mu) f(\mu) \, \rd\mu \right|^2 \lesssim \frac{1}{nt^3} e^{-\frac{(|x| - 1)^2}{t}}. 
\end{align*}
Now consider \(p(x, t) \lesssim \frac{1}{\sqrt{t}} e^{-\frac{(|x|-1)^2}{2t}}\), and so the first term can be bounded. Let us turn our attention to the second term. Consider after some calculation
\begin{align*}
    \bE\left|\hat{p}(x, t) - p(x, t)\right|^2 \le \bE\left|\frac{1}{n} \sum_{i=1}^{n} \varphi_t(x-\mu_i) - p(x, t)\right|^2 \lesssim \frac{1}{nt} e^{-\frac{(|x| - 1)^2}{t}}. 
\end{align*}
Here, we have used that \(f \in \mathcal{F}_\alpha\) implies \(p(x, t) \ge \varepsilon(x, t)\) to obtain the first inequality. Furthermore, observe \(\left|\frac{1}{t} \int_{-1}^{1} \mu \varphi_t(x-\mu)f(\mu) \, \rd \mu\right|^2 \lesssim \frac{1}{t^{3}} e^{-\frac{(|x| - 1)^2}{t}}\), and so we can bound the second term. Putting together our bounds, it follows 
\begin{align*}
    \int_{\R} |\hat{s}(x, t) - s(x, t)|^2 p(x, t) \, \rd x &\lesssim \int_{\R} \frac{1}{nt^4} e^{-\frac{2(|x|-1)^2}{t}} \cdot \frac{1}{p(x, t) \varepsilon(x, t)^2} \, \rd x = \frac{e^{\frac{24}{t}} }{nt^2}\int_{\R} \frac{1}{\sqrt{t}} e^{-\frac{(|x| - 7)^2}{2t}} \, \rd x \lesssim \frac{e^{\frac{24}{c}}}{nt^2}.
\end{align*}
Here, we have used \(p(x, t), \varepsilon(x, t) \gtrsim \frac{1}{\sqrt{t}}e^{-\frac{(|x| + 1)^2}{2t}}\) and also \(t \ge c\), yielding the desired result of Theorem \ref{thm:upperbound_large_t}. 

\subsubsection{\texorpdfstring{High noise regime: \(n^{-\frac{2}{2\alpha+1}} \le t \le 1\)}{High noise regime}}
The high noise regime \(n^{-\frac{2}{2\alpha+1}} \le t \le 1\) is addressed by Theorem \ref{thm:up-high-1}, and the argument proceeds by first decomposing the score estimation error as follows. Recall the estimator \(\hat{s}(x, t)\) is given by (\ref{def:shat_high}). 
\begin{lemma}
\label{lemma:score-split-1}
If \(D \subseteq \R\), then 
\begin{align*}
&\int_{D} \bE \left(\hat{s}(x,t)-s(x,t)\right)^2 p(x,t) \, \rd x \\
&\quad\quad\quad \lesssim \int_D \frac{\bE (\hat{\psi}(x, t)-\psi(x,t))^2}{p(x,t)} \, \rd x + \int_D s(x,t)^2 \cdot \frac{\bE (\hat{p}(x, t)-p(x,t))^2}{p(x,t)} \,\rd x.
\end{align*}
\end{lemma}
From Lemma \ref{lemma:score-split-1}, to bound the score matching error in Theorem \ref{thm:up-high-1} it suffices to bound the following two terms,
\[I_1 := \int_\bR \frac{\bE (\hat{\psi}(x, t)-\psi(x,t))^2}{p(x,t)}\,\rd x, ~~~I_2 := \int_\bR s(x,t)^2 \cdot \frac{\bE (\hat{p}(x, t)-p(x,t))^2}{p(x,t)} \,\rd x. \]
These can be bounded by Lemmas \ref{lemma:up-6} and \ref{lemma:up-7} respectively, which are proved in Appendix \ref{section:high-noise}. 
\begin{lemma}
\label{lemma:up-6}
If \(t \le 1\), then \(I_1 \lesssim \frac{1}{nt^{3/2}}\).
\end{lemma}
\begin{lemma}
\label{lemma:up-7}
If \(t \le 1\), then \(I_2 \lesssim \frac{1}{nt^{3/2}}\).
\end{lemma}
Theorem \ref{thm:up-high-1} immediately follows as a consequence. The analyses of Lemmas \ref{lemma:up-6} and \ref{lemma:up-7} proceed by splitting the integral over the three regions \(D_1,D_2,\) and \(D_3\) defined in Section \ref{section:upper_bound}. For \(x \in D_1 \cup D_2\), it is easy to see \(p(x, t) \gtrsim 1\). It can be further shown \(|s(x, t)| \lesssim \frac{1}{\sqrt{t}}\). Therefore, it is clear that the integrands of \(I_1\) and \(I_2\), after integrating over \(D_1 \cup D_2\), are at most of order \(\frac{1}{nt^{3/2}}\). For the remaining part \(D_3\), we compare the exponential decay rates of the numerators and denominators, and show that the integrands of \(I_1\) and \(I_2\) exhibit exponential decay and thus yields a risk which is dominated by \(\frac{1}{nt^{3/2}}\). 

\subsubsection{\texorpdfstring{Low noise regime: \(t < n^{-\frac{2}{2\alpha+1}}\)}{Low noise regime}}
In this section, we discuss the low noise regime \(t < n^{-\frac{2}{2\alpha+1}}\). Attention will be focused on the case \(\alpha \ge 1\) addressed by Theorem \ref{thm:up-1}, since the optimal rate in the case \(\alpha < 1\) is achieved by a trivial estimator as stated in Theorem \ref{thm:up-2}. The proof of Theorem \ref{thm:up-2} is contained in Appendix \ref{appendix:upper_bound_proofs}. Examining Theorem \ref{thm:up-1}, recall we use the estimator (\ref{def:shat_D1}) on the region \(D_1\) and (\ref{def:shat_D2}) on the region \(D_2 \cup D_3\). 

The analysis of (\ref{def:shat_D1}) resembles the analysis in the high noise regime owing to the similar form of regularization. The following standard result concerning kernel density estimators is used. 
\begin{lemma}[\cite{li2005convergence, gine2021mathematical}]
Given $n$ i.i.d samples $\mu_1, \mu_2, \ldots, \mu_n\sim f$, there exists a kernel-based estimator $\hat{f}_n(x)$ with the following form
\[\hat{f}_n(x) = \frac{1}{nu} \sum_{j=1}^n \left[K_{\lfloor\alpha\rfloor}((\mu_j-x)/u)\cdot \mathbbm {1}_{\{x<0\}} + K_{\lfloor\alpha\rfloor}((x-\mu_j)/u)\cdot \mathbbm 1_{\{x\ge0\}}\right],\]
such that for all $x\in [-1,1]$, 
\[\bE \left(\hat{f}_n(x)-f(x)\right)^2 \lesssim n^{-\frac{2\alpha}{2\alpha+1}}, ~~~\bE \left(\hat{f}_n^{(k)}(x)-f^{(k)}(x)\right)^2 \lesssim n^{-\frac{2(\alpha-k)}{2\alpha+1}}~~\textit{for } k=1,2,\ldots,\lfloor \alpha \rfloor. \]
Here, the bandwidth $u$ is optimally chosen as $u\asymp n^{-\frac{1}{2\alpha+1}}$. 
\label{lemma:upper-1}
\end{lemma}
Lemma \ref{lemma:upper-1} can be used to establish the following bound on the estimation error of (\ref{def:shat_D1}) on \(D_1\).
\begin{proposition}
\label{prop:up-internal}
If \(\alpha \ge 1\) and \(t < n^{-\frac{2}{2\alpha+1}}\), then
\[\sup_{f \in \mathcal{F}_\alpha} \bE\left(\int_{D_1} \left|\hat{s}_{1,n}(x,t)-s(x,t)\right|^2 p(x,t)\, \rd x\right) \lesssim n^{-\frac{2(\alpha-1)}{2\alpha+1}}\]
where \(\hat{s}_{1, n}\) is given by (\ref{def:shat_D1}). 
\end{proposition}
\begin{proof}
Recall we use the density derivative estimator 
\begin{equation}
\label{eqn:up-internal-2}
\hat{\psi}_n(x,t) = \phi_t(x+1)\hat{f}_n(-1)-\phi_t(x-1)\hat{f}_n(1)+\int_{-1}^{1}\phi_t(x-\mu)\hat{f}'_n(\mu) \,\rd\mu,
\end{equation}
the density estimator $\hat{p}_n(x,t)=\phi_t* \hat{f}_n(x)$, and the regularizer $\varepsilon(x,t)=c_d\int_{-1}^{1} \phi_t(x-\mu)\,\rd \mu$. For the density derivative $\psi(x,t)=\phi_t'* f(x)$, we use integration by parts to rewrite this term as 
\begin{equation}
\label{eqn:up-internal-1}
\psi(x,t)=\int_\bR \phi_t'(x-\mu)f(\mu)\,\rd \mu
=\phi_t(x+1)f(-1)-\phi_t(x-1)f(1)+\int_{-1}^{1} \phi_t(x-\mu)f'(\mu) \,\rd \mu. 
\end{equation}
According to Lemma \ref{lemma:up-1} and Lemma \ref{lemma:score-up}, we know that $p(x,t)\gtrsim 1$ and $|s(x,t)| \lesssim 1$ when $|x| < 1-\sqrt{Ct\log(1/t)}$. Therefore, it follows from Lemma \ref{lemma:score-split-1},
\[\int_{D_1} \bE \left(\hat{s}_{1,n}(x,t)-s(x,t)\right)^2 \, p(x,t) \, \rd x \lesssim \int_{D_1} \bE (\hat{\psi}_n(x, t)-\psi(x,t))^2 \,\rd x + \int_{D_1}\bE (\hat{p}_n(x, t)-p(x,t))^2 \,\rd x.\]
For the density estimation error term, we have
\begin{equation}\label{eqn:D1_density}
\begin{aligned}
\int_{D_1} \bE (\hat{p}_n(x, t)-p(x,t))^2 \,\rd x &= \int_{D_1}\bE \left(\phi_t * (\hat{f}_n -f)(x)\right)^2 \,\rd x = \int_{D_1}\bE \left(\bE_{\mu\sim \mathcal N(x,t)} (\hat{f}_n -f)(\mu)\right)^2 \,\rd x \\
&\le \int_{D_1}\bE\left(\bE_{\mu\sim \mathcal N(x,t)} \left(\hat{f}_n(\mu) -f(\mu)\right)^2\right) \,\rd x\\
&\lesssim n^{-\frac{2\alpha}{2\alpha+1}} \cdot |D_1| \lesssim n^{-\frac{2\alpha}{2\alpha+1}}
\end{aligned}
\end{equation}
Here, we applied Jensen's inequality and Lemma \ref{lemma:upper-1}. Next, we bound the density derivative estimation error term. After comparing (\ref{eqn:up-internal-1}) and (\ref{eqn:up-internal-2}), it follows for any $x\in D_1$,
\begin{equation}\label{eqn:D1_psi}
\begin{aligned}
&~~~\bE (\hat{\psi}_n(x, t)-\psi(x,t))^2 \\
&\lesssim \phi_t(x+1)^2 \cdot \bE (\hat{f}_n(-1)-f(-1))^2 + \phi_t(x-1)^2 \cdot \bE (\hat{f}_n(1)-f(1))^2 + \bE \left[\bE_{\mu\sim \mathcal N(x,t)}(\hat{f}_n'-f')(\mu)\right]^2 \\
&\lesssim n^{-\frac{2\alpha}{2\alpha+1}}\cdot \left[\phi_t(x+1)^2+\phi_t(x-1)^2\right] + \bE_{\mu\sim \mathcal N(x,t)} \bE \left[(\hat{f}_n'-f')(\mu)\right]^2 \\
& \lesssim n^{-\frac{2\alpha}{2\alpha+1}}\cdot \left[\phi_t(x+1)^2+\phi_t(x-1)^2\right] + n^{-\frac{2(\alpha-1)}{2\alpha+1}}. 
\end{aligned}
\end{equation}
Since $|x| < 1-\sqrt{Ct\log(1/t)}$ with $C>0$, it holds that $\phi(x+1), \phi(x-1) < t^{C/2} < 1$, and so
\[\int_{D_1} \bE (\hat{\psi}_n(x, t)-\psi(x,t))^2 \,\rd x \lesssim n^{-\frac{2\alpha}{2\alpha+1}} + n^{-\frac{2(\alpha-1)}{2\alpha+1}} \lesssim n^{-\frac{2(\alpha-1)}{2\alpha+1}}. \]
Combining the two bounds (\ref{eqn:D1_density}) and (\ref{eqn:D1_psi}) yields our conclusion. 
\end{proof}

The more interesting phenomenon occurs in the analysis of (\ref{def:shat_D2}), as it has been noted in Remark \ref{remark:boundary} and Section \ref{section:upper_bound} that the self-normalization of the score function enables fast estimation at the boundary without imposing additional smoothness conditions. In particular, a faster rate than \(n^{-\frac{2(\alpha-1)}{2\alpha+1}}\) can be achieved on \(D_2\). 
\begin{proposition}
\label{prop:up-boundary}
If \(\alpha \ge 1\) and \(t < n^{-\frac{2}{2\alpha+1}}\), then 
\[\sup_{f \in \mathcal{F}_\alpha} \bE \left(\int_{D_2} \left|\hat{s}_{2,n}(x,t)-s(x,t)\right|^2 p(x,t) \, \rd x\right) \lesssim \sqrt{t} \log^{\alpha+1/2}(1/t)\cdot n^{-\frac{2(\alpha-1)}{2\alpha+1}} + e^{-cn^{\frac{2(\alpha-1)}{2\alpha+1}}}\]
where \(\hat{s}_{2, n}\) is given by (\ref{def:shat_D2}). Here, \(c > 0\) is some universal constant. 
\end{proposition}

Proposition \ref{prop:up-boundary} is proved in Appendix \ref{section:low-noise_kernel}. To illustrate the argument, let us assume the event \(\Omega\) is in force when examining (\ref{def:shat_D2}) since \(\Omega^c\) is a low-probability event as noted in Section \ref{section:methodology_low_noise} and the score-matching error is constant on \(\Omega^c\) (see Remark \ref{remark:alpha}). To bound the score matching loss (\ref{def:score_matching_loss}), we directly operate  
\begin{equation}
\label{eqn:intro-2}
\hat{s}(x,t)-s(x,t) = \frac{(\phi_t'* \hat{f})(x)\cdot\left(\phi_t * f\right)(x) - (\phi_t * \hat{f})(x)\cdot\left(\phi_t'* f\right)(x)}{(\phi_t * \hat{f})(x)\cdot \left(\phi_t * f\right)(x)}.
\end{equation}
For \(x \in D_2\), the denominator in (\ref{eqn:intro-2}) is lower bounded by a constant, and so it suffices to focus on the numerator. In other words, letting \(J(x, t)\) denote the numerator in (\ref{eqn:intro-2}), it suffices to show \(\int_{D_2} \bE(J(x, t)^2\mathbbm{1}_{\Omega}) \, \rd x \lesssim n^{-\frac{2(\alpha-1)}{2\alpha+1}}\). Observe that integration by parts yields 
\begin{align*}
    (\varphi_t' * f)(x) &= \varphi_t(x+1)f(-1) - \varphi_t(x-1)f(1) + \int_{-1}^{1} \varphi_t(x-\mu) f'(\mu) \,\rd \mu, \\
    (\varphi_t' * \tilde{f})(x) &= \varphi_t(x+1)\hat{f}(-1) - \varphi_t(x-1)\hat{f}(1) + \int_{-1}^{1} \varphi_t(x-\mu)\hat{f}'(\mu) \, \rd \mu. 
\end{align*}
Therefore, we have the bound \(J(x, t)^2 \lesssim J_1(x, t) + J_2(x, t) + J_3(x, t)\) where 
\begin{align*}
    J_1(x, t) &= |\varphi_t(x-1)|^2|f(1) (\varphi_t * \hat{f})(x) - \hat{f}(1) (\varphi_t * f)(x)|^2, \\
    J_2(x, t) &= |\varphi_t(x+1)|^2 |f(-1)(\varphi_t * \hat{f})(x) - \hat{f}(-1) (\varphi_t * f)(x)|^2, \\
    J_3(x, t) &= \left((\varphi_t * f)(x) \int_{-1}^{1} \varphi_t(x-\mu) \hat{f}'(\mu) \,\rd \mu - (\varphi_t * \hat{f})(x)\int_{-1}^{1} \varphi_t(x-\mu) f'(\mu)\, \rd\mu\right)^2.
\end{align*}
The following three lemmas, Lemmas \ref{lemma:J2}-\ref{lemma:J1}, provide an upper bound for \(J_2\), \(J_3\), and \(J_1\) respectively. It can be shown immediately the expectation of \(J_2\) is exponentially small and thus negligible. 
\begin{lemma}\label{lemma:J2}
    If \(x \in D_2\), then \(\bE\left(J_2(x, t) \mathbbm{1}_{\Omega}\right)\lesssim \frac{e^{-\frac{1}{t}}}{t^2}\). 
\end{lemma}
\begin{proof}
    Without loss of generality, assume \(x\) is positive (a similar argument holds for negative \(x\)). The result immediately follows from the fact \(|\varphi_t(x+1)|^2 \lesssim \frac{1}{t} e^{-\frac{1}{t}}\) and \(\varphi_t(x-\mu) \lesssim \frac{1}{\sqrt{t}}\) for any \(\mu\). 
\end{proof}
Unsurprisingly, \(J_3\) can also be bounded in a straightforward way to obtain the desired \(n^{-\frac{2(\alpha-1)}{2\alpha+1}}\) rate since we have already passed to the derivatives of \(f\) and \(\hat{f}\). 
\begin{lemma}\label{lemma:J3}
    If \(\alpha \ge 1\), then \(\bE\left(J_3(x, t)\mathbbm{1}_{\Omega}\right) \lesssim n^{-\frac{2(\alpha-1)}{2\alpha+1}}\). 
\end{lemma}
\begin{proof}
By Jensen's inequality and Lemma \ref{lemma:upper-1}, we have
\begin{align*}
\bE\left(J_3(x,t)\mathbbm{1}_{\Omega}\right) &\le \bE \left(\bE_{\mu,\nu\sim \mathcal N(x,t)} \left(f'(\mu)\hat{f}_n(\nu)-\hat{f}_n'(\mu)f(\nu)\right)^2\right) \\
&\lesssim \bE_{\mu,\nu\sim \mathcal N(x,t)} \bE \left(f'(\mu)\hat{f}_n(\nu)-\hat{f}_n'(\mu)f(\nu)\right)^2 \lesssim n^{-\frac{2(\alpha-1)}{2\alpha+1}}.
\end{align*}
\end{proof}

The worrisome term is \(J_1\), whose existence is due to the discontinuity of \(f\) and \(\hat{f}\) at the point \(1\). As pointed out in Remark \ref{remark:boundary}, the expert reader might expect the discontinuity to cause some slowdown in the rate. Specifically, one might expect that \(\hat{f}(1)\) can not estimate \(f(1)\) at a (squared) rate faster than \(n^{-\frac{2\alpha}{2\alpha+1}}\), and so it might be hopeless to improve upon the bound \(\bE(J_1(x, t)\mathbbm{1}_{\Omega}) \lesssim \frac{1}{t} \cdot n^{-\frac{2\alpha}{2\alpha+1}}\). This bound is not quite good enough since it would only deliver \(\int_{D_2 \cap \R^+} \bE \left(J_1(x, t)\mathbbm{1}_{\Omega}\right) \, \rd x \lesssim n^{-\frac{2\alpha}{2\alpha+1}} \cdot \sqrt{\frac{\log(1/t)}{t}}\), which is larger than the bound \(n^{-\frac{2(\alpha-1)}{2\alpha+1}}\) claimed in (\ref{eqn:rate}) for sufficiently small \(t\). However, it turns out it is not at all necessary for \(\hat{f}(1)\) to be a good estimator of \(f(1)\) in order to achieve good score estimation, as the following lemma asserts (proved in Appendix \ref{section:low-noise_kernel}). 

\begin{lemma}\label{lemma:J1}
    If \(x \in D_2\), then \(\bE\left(J_1(x, t)\mathbbm{1}_{\Omega}\right) \lesssim (\log(1/t))^\alpha n^{-\frac{2(\alpha-1)}{2\alpha+1}}\).
\end{lemma}
The intuition for Lemma \ref{lemma:J1} is as follows. Consider that for any \(\mu \in (-1, 1)\), Taylor expansion around the point \(1\) yields \(f(\mu) = \sum_{k=0}^{\lfloor \alpha \rfloor} \frac{f^{(k)}(1)}{k!} (\mu-1)^k + O(|\mu-1|^{\alpha - \lfloor \alpha \rfloor})\), and so 
\begin{align*}
    (\varphi_t * f)(x) &= \int_{-1}^{1} \varphi_t(x-\mu) f(\mu) \, \rd \mu \\
    &= \sum_{k=0}^{\lfloor \alpha \rfloor} \frac{f^{(k)}(1)}{k!} \int_{-1}^{1} (\mu-1)^k \varphi_t(x - \mu) \, \rd \mu + O\left(\int_{-1}^{1} |\mu-1|^{\alpha} \varphi_t(x-\mu) \, \rd\mu \right). 
\end{align*}
Since \(|x-1| < \sqrt{Ct \log(1/t)}\), it can be shown that the error term is dominated by \((t \log(1/t))^{\alpha/2}\). The same expansion can be applied to \(\varphi_t * \hat{f}\), and so 
\begin{align*}
    &\bE \left(|f(1) (\varphi_t * \hat{f})(x) - \hat{f}(1) (\varphi_t * f)(x)|^2\right) \\
    &\lesssim \bE\left(|f(1)\hat{f}(1) - \hat{f}(1) f(1)|^2\right) \\
    &\;\;\; + \sum_{k=1}^{\lfloor\alpha\rfloor} \frac{\int_{-1}^{1} |\mu-1|^{2k} \varphi_t(x-\mu) \,\rd \mu}{k!} \bE\left(|f(1)\hat{f}^{(k)}(1) - \hat{f}(1)f^{(k)}(1)|^2\right) + (t\log(1/t))^{\alpha} \\
    &\lesssim \sum_{k=1}^{\lfloor \alpha \rfloor} (t\log(1/t))^{k} \cdot n^{-\frac{2(\alpha - k)}{2\alpha+1}} + (t\log(1/t))^{\alpha} \\
    &\lesssim t n^{-\frac{2(\alpha-1)}{2\alpha+1}} \cdot \log^{\alpha}(1/t). 
\end{align*}
Then, since \(\varphi_t(x-1) \lesssim t^{-1/2}\), it follows \(\int_{D_2 \cap \R^+} \bE \left(J_1(x, t)\mathbbm{1}_{\Omega}\right) \, \rd x \lesssim n^{-\frac{2(\alpha-1)}{2\alpha+1}} \cdot \sqrt{t} \log^{\alpha+1/2}(1/t) \lesssim n^{-\frac{2(\alpha-1)}{2\alpha+1}}\). The slowdown has been circumvented! The key phenomenon is the cancellation of the first order terms (i.e. corresponding to \(k = 0\)). This extremely convenient cancellation occurs precisely because of the self-normalization in the score function which delivers the nice form of the numerator in (\ref{eqn:intro-2}). No additional smoothness at all needs to be assumed at the boundary of \(f\)'s support, unlike previous work \cite{oko2023diffusion}. Proposition \ref{prop:up-boundary} follows as a consequence of Lemmas \ref{lemma:J2}, \ref{lemma:J3}, and \ref{lemma:J1}. 

The following proposition (proved in Appendix \ref{section:low-noise_kernel}) establishes the error achieved on \(D_3\) by (\ref{def:shat_D2}).
\begin{proposition}
\label{prop:up-external}
If \(\alpha \ge 1\) and \(t < n^{-\frac{2}{2\alpha+1}}\), then 
\[\sup_{f \in \mathcal{F}_\alpha} \bE \left(\int_{D_3} \left|\hat{s}_{3, n}(x,t)-s(x,t)\right|^2 p(x,t) \, \rd x\right) \lesssim \sqrt{t}\cdot n^{-\frac{2(\alpha-1)}{2\alpha+1}}\]
where \(\hat{s}_{3, n}\) is given by (\ref{def:shat_D2}).
\end{proposition}
\noindent The analysis on the external part $D_3$ is a simple extension of the analysis of $D_2$. The only difficulty is to compare the exponential decay rates of the numerator and denominator of (\ref{eqn:intro-2}), but it can be shown the error is still dominated by the error on \(D_2\).

Theorem \ref{thm:up-1} can be directly concluded from combining the results of Propositions \ref{prop:up-internal}, \ref{prop:up-boundary}, and \ref{prop:up-external}, namely 
\begin{equation*}
    \sup_{f \in \mathcal{F}_\alpha} \bE\left(\int_\R |\hat{s}(x, t) - s(x, t)|^2 \, p(x, t) \, \rd x\right) \lesssim n^{-\frac{2(\alpha-1)}{2\alpha+1}}\left(1  + \sqrt{t}\log^{\alpha+1/2}(1/t) + n^{\frac{2(\alpha-1)}{2\alpha+1}}e^{-cn^{\frac{2(\alpha-1)}{2\alpha+1}}} + \sqrt{t} \right). 
\end{equation*}
Since \(t < n^{-\frac{2}{2\alpha+1}}\), it is immediately clear that the term \(n^{-\frac{2(\alpha-1)}{2\alpha+1}}\) dominates in the rate. Recall from Proposition \ref{prop:up-internal} that this term arises from the region \(D_1\).

%% file: lower/lower_bound_sketch_text.tex
In this section, we present the arguments for the lower bounds. Full details of the proofs can be found in Appendix \ref{appendix:lower_bound_proofs}.
\subsubsection{\texorpdfstring{Regime \(t \gtrsim 1\)}{Regime t >= 1}}
As pointed out in Section \ref{section:main_result_lower_bound}, the proof of Theorem \ref{thm:score_lowerbound_large_t} involves a two-point construction. Namely, we consider the densities \(f_0(\mu) = \frac{1}{2}\mathbbm{1}_{\{|\mu| \le 1\}}\) and \(f_1(\mu) = f_0(\mu) + \epsilon^{\alpha} w\left(\frac{\mu}{\rho}\right)\mathbbm{1}_{\{|\mu| \le 1\}}\) with \(w(x) = x\) and \(0 < \epsilon \le \rho\) to be set. Writing \(p_b(x, t) = (\varphi_t * f_b)(x)\), \(\psi_b(x, t) = (\varphi_t * f_b)'(x)\) and \(s_b(x, t) = \frac{\psi_b(x, t)}{p_b(x, t)}\) for \(b \in \{0, 1\}\), a standard two-point lower bound reduction yields 
\begin{align*}
    &\inf_{\hat{s}} \sup_{f \in \mathcal{F}_\alpha} \bE\left(\int_{-\infty}^{\infty} |\hat{s}(x, t) - s(x, t)|^2 \, p(x, t) \, \rd x \right) \\
    &\gtrsim \frac{e^{-n\dKL(f_0 \,||\, f_1)}}{\sqrt{t} \vee 1} \left(\int_{I} |s_0(x, t) - s_1(x, t)|^2 \, \rd x\right) \\
    &\gtrsim e^{-n\dKL(f_0 \,||\, f_1)}(t^{3/2} \vee 1) \left( \int_{I} |\psi_1(x, t)p_0(x, t) - \psi_0(x, t)p_1(x, t)|^2 \, \rd x\right) \\
    &\asymp \frac{\epsilon^{2\alpha}e^{-n\dKL(f_0 \,||\, f_1)}}{(\sqrt{t} \vee 1)\rho^2}\left(\int_{I} \left|\left(\int_{-1}^{1}\mu^2 \varphi_t(x-\mu) \, \rd \mu\right)\left(\int_{-1}^{1} \varphi_t(x-\mu) \, \rd \mu\right) - \left(\int_{-1}^{1} \mu\varphi_t(x-\mu) \, \rd \mu \right)^2 \right|^2 \, \rd x\right).
\end{align*}
Here, we have reduced to the interval \(I = [-1-\sqrt{t}, 1+\sqrt{t}]\) and have used \(p_b(\cdot, t) \gtrsim \frac{1}{\sqrt{t} \vee 1}\) for \(x \in I\). The following lemma (proved in Appendix \ref{appendix:lower_bound_proofs}) provides a lower bound for the integrand.
\begin{lemma}\label{lemma:g_sep}
    If \(x \in I\), then 
    \begin{equation}\label{eqn:g_sep}
        \left|\left(\int_{-1}^{1}\mu^2 \varphi_t(x-\mu) \, \rd \mu\right)\left(\int_{-1}^{1} \varphi_t(x-\mu) \, \rd \mu\right) - \left(\int_{-1}^{1} \mu\varphi_t(x-\mu) \, \rd \mu \right)^2 \right|^2 \gtrsim \frac{1}{t^2} e^{-\frac{2(2+\sqrt{t})^2}{t}}.
    \end{equation}
\end{lemma}
Since the length of \(I\) is of order \(\sqrt{t} \vee 1\), the integral can be lower bounded directly using (\ref{eqn:g_sep}). Furthermore, it is straightforward to obtain the bound \(n\dKL(f_0\,||\,f_1) \lesssim \frac{n\epsilon^{2\alpha}}{\rho^2}\), which yields the minimax lower bound 
\begin{equation*}
    \inf_{\hat{s}}\sup_{f \in \mathcal{F}_\alpha} \bE\left(\int_{-\infty}^{\infty} |\hat{s}(x, t) - s(x, t)|^2 \, p(x, t) \, \rd x\right)\gtrsim \frac{\epsilon^{2\alpha}}{(t^2 \vee 1) \rho^2} e^{-\frac{2(2+\sqrt{t})^2}{t}} e^{-\frac{Cn\epsilon^{2\alpha}}{\rho^2}}. 
\end{equation*}
Selecting \(\rho \asymp 1\) and \(\epsilon \asymp n^{-\frac{1}{2\alpha}}\), and noting \(t \geq c_2\), the rate \(\frac{1}{nt^2}\) claimed in Theorem \ref{thm:score_lowerbound_large_t} follows.

\subsubsection{\texorpdfstring{Regime \(t \lesssim 1\)}{Regime t <= 1}}
Our discussion in this section will focus on the difficulties and proof methods for Theorem \ref{thm:score_lowerbound}. At a high level, since the score function \(s(x, t) = \frac{\partial}{\partial x}\log p(x, t) = \frac{\frac{\partial}{\partial x}p(x, t)}{p(x, t)}\) involves the derivative of the density \(p(\cdot, t)\), it is plausible the difficulty of score estimation is determined by the difficulty of estimating the derivative of the density. Indeed, the proof of Theorem \ref{thm:score_lowerbound} first proceeds by showing score estimation is no easier (up to constant factors) than estimating \(\frac{\partial}{\partial x} p(x, t) = (\varphi_t * f)'(x)\) on a subinterval\footnote{Specifically, we use the interval \(I = [-1+\sqrt{C t \log(1/t)}, 1-\sqrt{Ct\log(1/t)}]\) which is contained in and just slightly smaller than \([-1, 1]\). This is done for technical reasons, but it suffices for intuition for the reader to think about \([-1, 1]\).} \(I \subset [-1, 1]\). From here, we use Fano's method as noted in Section \ref{section:main_result_lower_bound}.

Since the target is the derivative \((\varphi_t * f)'(x)\) of the ``noisy" density and not that of the data-generating distribution, one needs to carefully track the dependence on \(t\) of the problem's statistical difficulty; intuitively, estimation ought to become easier as \(t\) increases. The parameter \(\rho\) in our submodel construction (\ref{def:fano_submodel}) is introduced for this reason. The standard construction corresponding to \(\rho = \epsilon\) happens to give rise to a serious calculation difficulty. 

To illustrate this difficulty, suppose \(\rho = \epsilon\). An important ingredient in Fano's method is to compute the separations \(||(\varphi_t * f_b)'(x) - (\varphi_t * f_{b'})'(x)||_{L^2(I)}^{2}\). In the case \(t = 0\), the target is indeed the derivative of the data-generating density, and it follows immediately from the fact \(w\) is compactly supported that
\begin{equation}\label{eqn:dgp_sep}
    ||f_{b}' - f_{b'}'||_{L^2(I)}^2 \asymp \epsilon^{2\alpha-1} d_{\text{Ham}}(b, b').
\end{equation}
With this separation of the derivatives in hand, it remains to bound the pairwise Kullback-Leibler divergences for use in Fano's method; a simple calculation yields \(\dKL(f_b^{\otimes n} \,||\, f_{b'}^{\otimes n}) \lesssim n ||f_b - f_{b'}||_{2}^2 \lesssim n\epsilon^{2\alpha}\). Fano's method (e.g. see \cite{tsybakov_introduction_2009}) yields,
\begin{equation*}
    \inf_{\widehat{f'}} \sup_{f \in \mathcal{F}_\alpha} \bE \left(\left|\left|\widehat{f'} - f'\right|\right|_{L^2([-1, 1])}^2\right) \gtrsim  \epsilon^{2\alpha-1} m \left(1 - \frac{C n\epsilon^{2\alpha} + \log 2}{m} \right).
\end{equation*}
Since \(m \asymp \frac{1}{\epsilon}\), optimizing over \(\epsilon\) forces the choice \(\epsilon \asymp n^{-\frac{1}{2\alpha+1}}\), yielding the lower bound \(n^{-\frac{2(\alpha-1)}{2\alpha+1}}\). 

The difficulty arises when examining the case \(t > 0\). In this case, the relevant separation to calculate is no longer (\ref{eqn:dgp_sep}), but rather \(||(\varphi_t * f_b)' - (\varphi_t * f_{b'})'||_{L^2([-1, 1])}^2\). Due to the convolution, the collection \(\left\{\varphi_t' * w\left(\frac{\cdot - x_i}{\epsilon}\right)\right\}_{i=1}^{m}\) no longer exhibits the extremely convenient pairwise \(L^2\)-orthogonality enjoyed by \(\left\{w\left(\frac{\cdot - x_i}{\epsilon}\right) \right\}_{i=1}^{m}\), and so a direct calculation of the separation is not at all straightforward. Generally speaking, this poses a major technical hurdle in proving minimax lower bounds for estimation of Gaussian mixtures, and other works \cite{kim_minimax_2014,kim_minimax_2022,polyanskiy_sharp_2021} investigating this problem have gone down the route of constructing collections which satisfy some sort of approximate pairwise \(L^2\)-orthogonality. However, in our view, it appears quite challenging to successfully mimic such arguments for computing pairwise separation of derivatives. Furthermore, these existing approaches cannot be directly ported without further thought since our data are drawn i.i.d. from \(f\) and not the Gaussian mixture \(\varphi_t * f\); there is more information available in our setting which needs to be accounted for in the lower bound.

It is precisely at this stage where introducing a parameter \(\rho\) to be chosen in (\ref{def:fano_submodel}) is quite convenient. Before discussing the intuition, we first state the separation we can achieve for the derivatives when \(t > 0\). 
\begin{proposition}\label{prop:derivative_separation}
    There exists a universal constant \(C_3 > 0\) such that if \(\rho = C_3 \sqrt{t} \vee \epsilon\), then 
    \begin{equation*}
        \int_{I} |\psi_b(x, t) - \psi_{b'}(x, t)|^2 \, \rd x \gtrsim d_{\text{Ham}}(b, b') \epsilon^{2\alpha} \rho^{-1}
    \end{equation*}
    for \(b, b' \in \{0, 1\}^m\) where \(\psi_b(\cdot, t) = (\varphi_t * f_b)'\) and \(\psi_{b'}(\cdot, t) = (\varphi_t * f_{b'})'\). 
\end{proposition}
The proof of Proposition \ref{prop:derivative_separation} appears in Appendix \ref{appendix:lower_bound_proofs}. To calculate the separation, we essentially express \(||(\varphi_t * f_b)' - (\varphi_t * f_{b'})'||_{L^2(I)}^2 = ||(\varphi_t * f_b)' - (\varphi_t * f_{b'})'||_2^2 - ||(\varphi_t * f_b)' - (\varphi_t * f_{b'})'||_{L^2(I^c)}^2\). It turns out the first term on the right hand side is the dominating term, so our discussion will focus there. It is very well known that convolution with the Gaussian density solves the initial value problem for the heat equation. Specifically, for a function \(h : \R \to \R\), we have the identity \(\frac{\partial}{\partial t} (\varphi_t * h)(x) = \frac{1}{2} \cdot \frac{\partial^2}{\partial^2 x} (\varphi_t * h)(x)\). Since \(||(\varphi_t * f_b)' - (\varphi_t * f_{b'})'||_2^2 = ||\varphi_t * (f_b - f_{b'})'||_2^2\) and \(f_b - f_{b'} \in C^\infty(\R)\) with compact support, it follows by integration by parts and the heat equation that \(\frac{d}{dt} ||(\varphi_t * f_b)' - (\varphi_t * f_{b'})'||_2^2 = - ||\varphi_t * (f_b - f_{b'})''||_2^2 \ge - ||(f_b - f_{b'})''||_2^2\). Therefore, a first-order Taylor expansion yields 
\begin{equation}\label{eqn:heat_sep}
    ||(\varphi_t * f_b)' - (\varphi_t * f_{b'})'||_2^2 \ge ||(f_b - f_{b'})'||_2^2 - t ||(f_b - f_{b'})''||_2^2 = \epsilon^{2\alpha} \rho^{-1} d_{\text{Ham}}(b, b') \left(c_1 - t\rho^{-2} c_2\right). 
\end{equation}
where \(c_1 = ||w'||_2^2\) and \(c_2 = ||w''||_2^2\) are constants. Since we pick \(\rho \asymp \sqrt{t} \vee \epsilon\) in Proposition \ref{prop:derivative_separation}, the claimed separation follows. 

A comparison of (\ref{eqn:heat_sep}) to (\ref{eqn:dgp_sep}) reveals the utility in introducing \(\rho\) as a free parameter. If \(\rho = \epsilon \asymp n^{-\frac{1}{2\alpha+1}}\) as in (\ref{eqn:dgp_sep}) was decided at the outset, then (\ref{eqn:heat_sep}) is only meaningful when \(t \lesssim \epsilon^2 \asymp n^{-\frac{2}{2\alpha+1}}\) and only delivers the lower bound \(n^{-\frac{2(\alpha-1)}{2\alpha+1}}\) in this regime. Now with \(\rho\) as a free parameter, we have the flexibility to optimize in the last step of Fano's method subject to the constraints \(\epsilon \le \rho\) (to ensure \(\{f_b\}_{b \in \mathcal{B}} \subset \mathcal{F}_\alpha\)) and \(t\rho^{-2} \lesssim 1\). Armed with (\ref{eqn:heat_sep}) and the pairwise Kullback-Leibler divergence bound \(\dKL(f_b^{\otimes n} \,||\, f_{b'}^{\otimes n}) \lesssim n ||f_b - f_{b'}||_2^2 \lesssim n\epsilon^{2\alpha}\), it follows from Fano's method 
\begin{equation}\label{eqn:heat_fano}
    \inf_{\hat{\psi}} \sup_{f \in \mathcal{F}_\alpha} \bE \left(\left|\left|\hat{\psi} - (\varphi_t * f)'\right|\right|_2^2\right) \gtrsim \epsilon^{2\alpha}\rho^{-1} m \left(1 - \frac{Cn\epsilon^{2\alpha} + \log 2}{m}\right)
\end{equation}
under the constraint \(\epsilon \le \rho\) and \(t\rho^{-2} \lesssim 1\). Select \(\rho \asymp \sqrt{t} \vee \epsilon\) and 
\begin{equation*}
    \epsilon \asymp \left(n\sqrt{t}\right)^{-\frac{1}{2\alpha}} \wedge 
    \begin{cases}
        n^{-\frac{1}{2\alpha+1}} &\textit{if } \alpha \ge 1, \\
        \sqrt{t} &\textit{if } \alpha < 1. 
    \end{cases}
\end{equation*}
Noting \(m \asymp \frac{1}{\rho}\) in (\ref{eqn:heat_fano}), the lower bound \(\frac{1}{nt^{3/2}} \wedge \left(n^{-\frac{2(\alpha-1)}{2\alpha+1}} + t^{\alpha-1}\right)\) follows. Interestingly, we only have \(\rho \asymp \epsilon\) as in the traditional construction yielding (\ref{eqn:dgp_sep}) when the regime \(t \lesssim n^{-\frac{2}{2\alpha+1}}\) and \(\alpha \ge 1\) is in force. Otherwise, the choice \(\rho \asymp \sqrt{t}\) is made to yield the sharp minimax lower bound.

%% file: density_estimation/proof_sketch.tex
In this section, we present the arguments for the density estimation. To improve readability, we have postponed the proofs of Lemma~\ref{lem:9} and Lemma~\ref{lem:10} to Appendix~\ref{app:DE}.

\subsubsection{Proof of Theorem~\ref{thm:TV}}
The proofs of Theorems~\ref{thm:TV} and \ref{thm:Wasserstein} involve controlling the total variation and Wasserstein distances between our estimator $\widehat{X}_0$ and the target $X_0 \sim f$. The estimator $\widehat{X}_0$ is derived from a designed diffusion process $\vv{\widehat{Y}}$, which approximates the true reverse process $\vv{Y}$ of $\vv{X}$. This is achieved by substituting the true score function \(s\) in (\ref{eqn:backward_process}) with \(\widehat{s}\) and the initial distribution \(p(\cdot, T)\) with $\mathcal{N}(0,T)$. To compare \(\vv{Y}\) and \(\vv{\widehat{Y}}\), we introduce an auxiliary backward process \(\vv{\widetilde{Y}}\) that satisfies (\ref{eqn:backward_process}) with the initialization of the ground truth \(p(\cdot, T)\) but with the estimator \(\widehat{s}\) as the drift term. 

\begin{table}[h]
    \centering
    \begin{tabular}{lll}  
        \hline
        & \textbf{INIT} & \textbf{DRIFT} \\
        \hline
        $\vv Y$ & $p(\cdot,T)$ & $s(Y_t,T-t)$ \\
        $\widetilde{\vv Y}$ & $p(\cdot,T)$ & $\widehat{s}(\widetilde{Y}_t,T-t)$ \\
        $\widehat{\vv Y}$ & $\mathcal{N}(0,T)$ & $\widehat{s}(\widehat{Y}_t,T-t)$ \\
        \hline
    \end{tabular}
    \caption{Comparison of initializations and drift terms for different processes.}
    \label{tab:1}
\end{table}


Firstly, we recall the Data Processing Inequality (DPI), which plays a crucial role in our subsequent proof. The Data Processing Inequality \cite{cover2006elements,polyanskiy2024information} asserts that when two random variables \(X\) and \(Y\) are processed through the same channel, the divergence between their resulting distributions cannot increase. Specifically, if \(X' = f(X)\) and \(Y' = f(Y)\) for some function \(f\) (which may include independent randomness for \(X\) and \(Y\)), then for any divergence measure \(D\), we have 
\[
D(X' \| Y') \le D(X \| Y).
\]
This inequality underscores the fact that information cannot be amplified through a noisy channel, thereby ensuring that the divergence, whether it be Kullback-Leibler, Total Variation, or another measure, does not increase during the transformation.

We use $\widehat{Y}_T \mathbbm{1}_{\left\{| \widehat{Y}_T| \le 1\right\}}$ defined in Algorithm~\ref{algo:DE} as our estimator. Noting that $Y_T = Y_T \mathbbm{1}_{\left\{| Y_T| \le 1\right\}}$ since $Y_T \sim f$ is supported on $[-1,1]$, we have $\dTV(Y_T,\widehat Y_T \mathbbm{1}_{\left\{|\widehat Y_T| \le 1\right\}}) = \dTV(Y_T \mathbbm{1}_{\left\{| Y_T| \le 1\right\}},\widehat Y_T \mathbbm{1}_{\left\{|\widehat Y_T| \le 1\right\}}) \le \dTV(Y_T,\widehat{Y}_T)$, by the DPI. Therefore, we have $\E \left( \dTV(Y_T,\widehat Y_T \mathbbm{1}_{\left\{|\widehat Y_T| \le 1\right\}})\right) \lesssim \E \left( \dTV(Y_T, \widehat{Y}_T) \right)$, which implies we only need to bound $\E \left( \dTV(Y_T, \widehat{Y}_T) \right)$.  The triangle inequality implies,
\bb
\dTV(Y_T, \widehat{Y}_T) \le \dTV(Y_T, \widetilde{Y}_T) + \dTV(\widehat{Y}_T, \widetilde{Y}_T).
\ee
By applying the DPI, we have the bound \(\dTV(\widehat{Y}_T, \widetilde{Y}_T) \le \dTV(p(\cdot, T), \mathcal{N}(0,T))\), where the stochastic differential equation (SDE) with drift term \(\widehat{s}\) is considered as the channel. Additionally, by Lemma~\ref{lem:KLappro} and Pinsker's inequality \cite{tsybakov_introduction_2009}, which states that for any two distributions \( P \) and \( Q \), \( 2\dTV(P,Q)^2 \leq \dKL(P,Q) \), we obtain \(\dTV(p(\cdot, T), \mathcal{N}(0,T)) \lesssim 1/\sqrt{T}\). We can make it \( 1/\sqrt{n} \) by taking \( T = n \).


The remaining task is to bound \(\dTV(Y_T, \widetilde{Y}_T)\). By Lemma~\ref{lemma:Girsanov}, which follows from Girsanov's Theorem, the Kullback-Leibler divergence between the path measures of \(\vv{\widetilde{Y}}\) and \(\vv{Y}\) can be bounded by the accumulated score matching error (up to a constant), \(\int_{0}^T \int_{\mathbb{R}} |s(x, t) - \widehat{s}(x, t)|^2 \, p(x, t) \, \diff x \, \diff t\). Applying Pinsker's inequality, we obtain the following bound:
\bbb
\label{eq:TV2SME}
&~~\E \left(\dTV(Y_T, \widetilde{Y}_T) \right)\lesssim\E \left(\dKL(Y_T, \widetilde{Y}_T)^{1/2} \right) \lesssim \E \left(\dKL(Y_T, \widetilde{Y}_T) \right)^{1/2}\nonumber \\
&\lesssim \sqrt{\int_{0}^T \int_{\mathbb{R}} \E[|s(x, t) - \widehat{s}(x, t)|^2] \, p(x, t) \, \diff x \, \diff t}.
\eee

As is discussed in Section~\ref{section:discussion}, we have
\[
\int_{\mathbb{R}} \mathbb{E}(|s(x, t) - \widehat{s}(x, t)|^2) \, p(x, t) \, \mathrm{d}x \lesssim (t^{\alpha-1}+n^{-\frac{2(\alpha-1)}{2\alpha+1}}) \wedge \frac{1}{nt^{\frac{3}{2}}} \wedge \frac{1}{nt^2}.
\]
Therefore, we have

\bbb
\label{eq:TV}
\int_{0}^T \int_{\mathbb{R}} \E(|s(x, t) - \widehat{s}(x, t)|^2) \, p(x, t) \, \diff x \, \diff t \lesssim \int_0^{t_*} (t^{\alpha-1}+n^{-\frac{2(\alpha-1)}{2\alpha+1}}) \diff t + \int_{t_*}^1 \frac{1}{nt^{\frac{3}{2}}} \diff t + \int_{1}^T \frac{1}{nt^2} \diff t,
\eee
for any \( t_* \). We choose \( t_* = n^{-\frac{2}{2\alpha+1}} \) to optimize this upper bound, which gives

\[
\int_{0}^T \int_{\mathbb{R}} \E(|s(x, t) - \widehat{s}(x, t)|^2) \, p(x, t) \, \diff x \, \diff t \lesssim n^{-\frac{2\alpha}{2\alpha+1}},
\]
which concludes the proof.


\subsubsection{Proof of Theorem~\ref{thm:Wasserstein}}
\label{sec:PFWASS}
Note that the support of \(f\) is \([-1,1]\),  and we estimate it by the truncated version $\widehat{X}_0 = \widehat{Y}_T \mathbbm{1}_{\left\{|\widehat{Y}_T| \le 1\right\}}$. A quick observation is that $\mathrm{W_1}(\widehat{X}_0, X_0) \lesssim \dTV(\widehat{X}_0, X_0)$ since both \(X_0\) and \(\widehat{X}_0\) are random variables with supports having constant diameter \cite{villani2009optimal}. But that upper bound is not sufficient, because for the estimation with respect to Wasserstein distance, we expect to obtain the faster rate $n^{-\frac{1}{2}}$ \cite{MR4441130} for $d=1$. 



Similar to \cite{MR4441130, oko2023diffusion}, we introduce intermediate processes for a better coupling under the Wasserstein distance. For any $t>0$, consider the process \(\vv{Y}^t = \{Y^t_\tau: \tau \in [0,T]\}\), which satisfies (\ref{eqn:backward_process}) with the true score \(s\) as the drift function for times less than \(T-t\), and the estimator \(\hat{s}\) for times after \(T-t\) (with a truncation at time \(T\)). That is,
\begin{align*}
&\diff Y^t_\tau = s(Y^t_\tau,T-\tau)\, \diff \tau + \diff W_\tau,~~\tau \in [0,T-t],\\
&\diff Y^t_\tau = \widehat{s}(Y^t_\tau,T-\tau)\, \diff \tau + \diff W_\tau,~~\tau \in (T-t,T],
\end{align*}
where \(Y^t_0 \sim p(\cdot,T)\). Let \(\widehat{Y}^t_T = Y^t_T \mathbbm{1}_{\left\{|Y^t_T| \le 1\right\}}\). We introduce \(0 = t_0 < t_1 < \ldots < t_N \le t_{N+1} = 1\), where \(t_i = \frac{2^i}{n}\) with \(i = 1, \ldots, \lfloor \log_2(n) \rfloor =: N\). By the triangle inequality, to bound \(\mathbb{E}(\mathrm{W_1}(Y^0_T, \widehat{Y}^T_T))\), we only need to show that the Wasserstein-1 distance between \( Y^{t_i}_T \) and \( Y^{t_{i+1}}_T \), and the Wasserstein-1 distance between \(Y^1_T\) and \(Y^T_T\) are small. Noting that \(Y^0_T = X_0\) and \(\widehat{Y}^T_T = \widehat{X_0}\), \(\mathbb{E}(\mathrm{W_1}(Y^0_T, \widehat{Y}^T_T))\) is the needed bound.

Given two time points \(0 \le a < b \le 1\), we can bound the Wasserstein-1 distance between \(\widehat{Y}^{a}_T\) and \(\widehat{Y}^{b}_T\) in terms of the Kullback-Leibler divergence between the path measures of \(\vv{Y}^a\) and \(\vv{Y}^b\) by the following lemma.
\begin{lemma}
\label{lem:9}
    Given the estimator \(\widehat{Y}^t_T = Y^t_T \mathbbm{1}_{\left\{|Y^t_T| \le 1\right\}}\), where \(\vv{Y}^t\) is defined above, for \(0 \le a < b \le 1\), we have
    \begin{equation*}
    \mathrm{W_1}(\widehat{Y}^{a}_T, \widehat{Y}^{b}_T) \lesssim \sqrt{b + \dKL(\vv{Y}||\vv{Y}^{b})} \cdot \sqrt{\dKL(\vv{Y}^{a}||\vv{Y}^{b})},
    \end{equation*}
where $\dKL(\vv{Y}^{a}||\vv{Y}^{b})$ \(\left(\dKL(\vv{Y}||\vv{Y}^{b})\right)\) denotes the Kullback-Leibler divergence between the path measures of the processes $\vv{Y}^{a}$ and $\vv{Y}^{b}$ \(\left(\vv{Y} \text{ and } \vv{Y}^{b}\right)\).
\end{lemma}

The lemma shows that the Wasserstein-1 distance between \(\widehat{Y}^{t_1}_T\) and \(\widehat{Y}^{t_2}_T\) is bounded by the product of the Kullback-Leibler distance between \(\vv{Y}^{t_1}\) and \(\vv{Y}^{t_2}\) (transport mass) and \(\sqrt{t_2+ \dKL(\vv{Y}||\vv{Y}^{t_2})}\) (transport distance), where  the term $\dKL(\vv{Y}||\vv{Y}^{t_2})$ can be shown to be at most an order of $\sqrt{t_2}$ under the expectation over the randomness of our samples of $X_0$ for $\alpha \ge 1$. We defer the proof to Appendix~\ref{app:DE}. 

We aim to obtain a bound for times \(t_0\) and \(t_{N+1}\), noting that
\begin{equation*}
\E(\mathrm{W_1}(Y^0_T, \widehat{Y}^T_T)) - \E(\mathrm{W_1}(\widehat{Y}^1_T, \widehat{Y}^T_T)) \leq \sum_{i=0}^{N} \E(\mathrm{W_1}(\widehat{Y}^{t_i}_T, \widehat{Y}^{t_{i+1}}_T)).
\end{equation*}
Firstly, we have \(\E(\mathrm{W_1}(\widehat{Y}^1_T, \widehat{Y}^T_T)) \lesssim \E(\mathrm{\dTV}(\widehat{Y}^1_T, \widehat{Y}^T_T))\) because their supports are compact. Combining (\ref{eq:TV2SME}) with the third term on the right-hand side of (\ref{eq:TV}), we can bound it by \(O(n^{-\frac{1}{2}})\).
 Applying the upper bounds for the score matching errors (Theorem~\ref{thm:score_upperbound}), we have the following lemma.
\begin{lemma}
\label{lem:10}
    For \(\alpha \ge 1\), we have
    \begin{align*}
    \E(\mathrm{W_1}(Y_T^{t_i}, Y_T^{t_{i+1}})) &\lesssim \sqrt{t_{i+1}} \cdot \sqrt{\int_{t_i}^{t_{i+1}} \int_{\mathbb{R}} \E(|s(x, t) - \widehat{s}(x, t)|^2) \, p(x,t) \, \diff x \, \diff t} \lesssim \mathrm{E}_i 
    \end{align*}
    where \(\mathrm{E}_i = n^{-\frac{\alpha-1}{2\alpha+1}} t_{i+1}\) when \(i \in [0, \lfloor \log_2(nt_*) \rfloor - 1]\) and \(\mathrm{E}_i = t_{i+1}^{\frac{1}{4}} n^{-\frac{1}{2}}\) when \(i \in [\lfloor \log_2(nt_*) \rfloor, N]\), and \(t_* = n^{-\frac{2}{2\alpha+1}}\).
\end{lemma}
The lemma implies that, for \(t_i = \frac{2^i}{n}\), the summation of the geometric series is dominated by the largest terms. We defer the proof of Lemma~\ref{lem:10} to Appendix~\ref{app:DE}. Therefore, we have
\begin{equation}
\sum_{i=0}^N \mathrm{E}_i \le n^{-\frac{\alpha-1}{2\alpha+1}} t_* + n^{-\frac{1}{2}} \lesssim n^{-\frac{1}{2}},
\end{equation}
which concludes the proof.

%% file: discussion.tex
\section{Discussion of the multivariate setting}\label{section:discussion}

We have developed the score matching methodology and the corresponding theory in the univariate setting for convenience. All of the results and ideas can be generalized in a straightforward way to the multivariate setting. Concisely, it can be shown that the sharp minimax rate of score estimation is 
\begin{equation*}
    \inf_{\hat{s}} \sup_{f \in \mathcal{F}_\alpha} \bE\left(\int_{\R^d} ||\hat{s}(x, t) - s(x, t)||^2\, p(x, t) \, \rd x\right) \asymp 
    \begin{cases}
        \frac{1}{nt^2} &\textit{if } t > 1, \\
        \frac{1}{nt^{d/2 + 1}} \wedge \left(n^{-\frac{2(\alpha-1)}{2\alpha+d}} + t^{\alpha-1}\right) &\textit{if } t \le 1.
    \end{cases}
\end{equation*}
Here, \(d \ge 1\) denotes the dimension and is fixed. The parameter space \(\mathcal{F}_\alpha\) denotes the multivariate H\"{o}lder class on \([-1, 1]^d\) defined in Appendix \ref{appendix:multivariate}, \(p(\cdot, t)\) is the density \(f * \mathcal{N}(0, tI_d)\), and \(s(x, t) = \nabla_x \log p(x, t)\) is the score function. This score estimation result can be used to show that diffusion model continues to achieve the sharp rate for density estimation in dimension \(d\). The extension to the multivariate case for \(t \gtrsim 1\) follows immediately from the univariate case addressed in Appendix \ref{section:very-high-noise}. For \(t \lesssim 1\), the extension is straightforward and is sketched in Appendix \ref{appendix:multivariate}.

It is worth noting that for \(d > 1\), the error bounds in total variation distance and Wasserstein-1 distance derived in the univariate setting extend naturally to higher dimensions. Specifically, Theorem~\ref{thm:TV2} shows that for the total variation distance, the bound becomes \(\mathbb{E}\left(\mathrm{TV}(X_0, \widehat{X}_0)\right) \lesssim n^{-\frac{\alpha}{2\alpha + d}}\) in dimension \(d\). Similarly, for the Wasserstein-1 distance, Theorem~\ref{thm:Wasserstein2} establishes that the minimax rate for \(d > 2\) is \(\mathbb{E}\left(\mathrm{W_1}(X_0, \widehat{X}_0)\right) \lesssim n^{-\frac{\alpha+1}{2\alpha + d}}\). For \(d = 2\), a logarithmic factor appears, leading to a rate of \(\mathbb{E}\left(\mathrm{W_1}(X_0, \widehat{X}_0)\right) \lesssim \log(n) n^{-\frac{1}{2}}\). To the best of our knowledge, this represents the most optimal result currently available \cite{MR4441130}.

\section*{Acknowledgements}
The authors would like to thank Prof. Yihong Wu and Prof. Zhou Fan for helpful discussions.  

%% file: setting.tex
In this section, we provide the detailed proof for the upper bound of integrated squared error of score estimator all noise levels $t > 0$. 
\subsection{Basic Settings}
For the original density $f$, recall we denote its noisy version by $p(x, t)=f * \varphi_t(x)$ where \(\varphi_t(x) = \frac{1}{\sqrt{2\pi t}}e^{-\frac{x^2}{2t}}\) is the density of \(\mathcal{N}(0, t)\). Throughout, we assume \(f \in \mathcal{F}_\alpha\) where \(\mathcal{F}_\alpha\) is given by (\ref{def:param}). The density \(p(\cdot, t)\) and its derivative $\psi(x,t) := \frac{\partial}{\partial x}p(x,t)$ have the following expressions,
\[p(x,t) = \phi_t * f(x) =  \int_U \phi_t(x-\mu) f(\mu) \,\rd \mu. \]
\[\psi(x,t) = \phi_t'*f(x) = \int_U \frac{\mu-x}{t}\phi_t(x-\mu) f(\mu) \,\rd \mu,\]
where \(U = [-1, 1]\). In this section, we propose estimators for the score function $s(x,t) := \psi(x,t)/p(x,t)$ given $n$ observations $\mu_1, \mu_2,
\ldots, \mu_n \overset{iid}{\sim} f$. The very high noise (\(t \gtrsim 1\)) and the high noise ($n^{-\frac{2}{2\alpha+1}} \lesssim t \lesssim 1$) regimes are addressed in Appendices \ref{section:very-high-noise} and \ref{section:high-noise} respectively. The analysis of the low noise regime (\(t \lesssim n^{-\frac{2}{2\alpha+1}}\)) is split into two pieces and discussed in Appendices \ref{section:low-noise_kernel} and \ref{section:low-noise_datafree}.

%% file: upper/useful-lemmas.tex
\begin{lemma}
\label{lemma:up-1}
For any universal constant $C > 0$, it holds that for all $|x|\le 1+C\sqrt{t}$:
\[p(x, t) \begin{cases} \ge c ~~~~~~\mbox{when $t\le 1$} \\ \asymp \frac{1}{\sqrt{t}}~~~~\mbox{when $t > 1$} \end{cases}\]
where $c>0$ is another universal constant. 
\end{lemma}
\begin{proof}[Proof of Lemma \ref{lemma:up-1}]
When variance $t\le 1$, since $f(x)\ge c_d$ for all \(x\in [-1,1]\), we have,
\begin{equation}
\label{eqn:lemma-11}
\begin{aligned}
p(x, t) &= \int_U \frac{1}{\sqrt{2\pi t}} \exp\left(-\frac{(x-\mu)^2}{2t}\right)\cdot  f(\mu) \,\rd \mu \ge c_d \cdot \int_U \frac{1}{\sqrt{2\pi t}}\exp\left(-\frac{(x-\mu)^2}{2t}\right) \,\rd \mu \\
&= c_d \cdot \bP\left\{-1\le \mathcal N(x,t) \le 1\right\} \ge c_d \cdot \bP_{z\sim \mathcal N(0,1)}\left\{\frac{x-1}{\sqrt{t}}\le z \le \frac{x+1}{\sqrt{t}}\right\} \\ 
&\ge c_d \cdot \bP_{z\sim \mathcal N(0,1)}\left\{C\le z \le C+\frac{2}{\sqrt{t}}\right\} \ge c_d \cdot \bP_{z\sim \mathcal N(0,1)}\left\{C\le z \le C+2\right\} := c. 
\end{aligned}
\end{equation}
Here $c$ is also a universal constant. In another case where $t > 1$, we have
\[p(x,t) = \int_U \frac{1}{\sqrt{2\pi t}}\exp\left(-\frac{(x-\mu)^2}{2t}\right)\cdot f(\mu)\,\rd\mu \le \int_U \frac{1}{\sqrt{2\pi t}}\cdot C_d \,\rd \mu = \frac{C_d\sqrt{2/\pi}}{\sqrt{t}}, \]
\[\text{and}~~p(x,t)
\overset{(a)}{\ge} c_d \cdot \bP_{z\sim \mathcal N(0,1)}\left\{C\le z \le C+\frac{2}{\sqrt{t}}\right\} = c_d \int_C^{C+2/\sqrt{t}}\phi(z)\,\rd z \ge \frac{2c_d}{\sqrt{t}}\cdot\phi(C+2).\]
Here, (a) follows \eqref{eqn:lemma-11}. To sum up, we have $p(x,t)\asymp \frac{1}{\sqrt{t}}$ when $t > 1$. 
\end{proof}
\begin{lemma}[Following Exercise 6.1 of \cite{shorack2000probability}]
\label{lemma:up-2}
For any $x > 0$, the Gaussian tail satisfies
\[\bP_{z\sim \mathcal N(0,1)} \{z \ge x\} \asymp \frac{1}{x \vee 1}\exp\left(-x^2/2\right). \]
More precisely, when $x \ge 1$, we have:
\[\frac{1}{\sqrt{2\pi}}\cdot \frac{1}{2x}\exp\left(-x^2/2\right) \le \bP_{z\sim \mathcal N(0,1)} \{z \ge x\} \le \frac{1}{\sqrt{2\pi}}\cdot \frac{1}{x}\exp\left(-x^2/2\right).\]
When $x < 1$, we have:
\[\bar{\Phi}(1)\exp\left(-x^2/2\right) \le \bP_{z\sim \mathcal N(0,1)} \{z \ge x\} \le \frac{\sqrt{e}}{2} \exp\left(-x^2/2\right).\]
After combining these two parts, we conclude that
\begin{equation}
\label{eqn:Gaussian-toy1}
\bar{\Phi}(1)\frac{1}{x \vee 1}\exp\left(-x^2/2\right) \le \bP_{z\sim \mathcal N(0,1)} \{z \ge x\} \le \frac{\sqrt{e}}{2} \frac{1}{x \vee 1}\exp\left(-x^2/2\right).
\end{equation}
Here $\bar{\Phi}(1)=\mathbb P_{z\sim\mathcal N(0,1)}[z>1]$. 
\end{lemma}
\begin{proof}[Proof of Lemma \ref{lemma:up-2}]
For $x\ge 1$, it can be directly obtained from the inequality 
\[\frac{x}{x^2+1}\frac{1}{\sqrt{2\pi}}\exp(-x^2/2)\le \bP_{z\sim \mathcal N(0,1)} \{z \ge x\} \le \frac{1}{\sqrt{2\pi} x} \exp(-x^2/2).\]
For $x < 1$, the conclusion is trivial. Similarly, for $x < 0$, the left Gaussian tail satisfies
\[\mathbb P_{z\sim \mathcal N(0,1)} \left\{z\le x\right\} \asymp \frac{1}{-x \vee 1}\exp\left(-x^2/2\right).\]
\end{proof}
\begin{lemma}
\label{lemma:up-3}
When $t\le 1$, it holds that for all $x>1$,
\[p(x, t) \asymp \bP_{z\sim \mathcal N(0,1)} \left\{\frac{x-1}{\sqrt{t}}\le z \le \frac{x+1}{\sqrt{t}}\right\} \asymp \bP_{z\sim \mathcal N(0,1)}\left\{z\ge \frac{x-1}{\sqrt{t}}\right\} \asymp \frac{\sqrt{t}}{x-1}\cdot \exp\left(-\frac{(x-1)^2}{2t}\right). \]
As a byproduct, we have $p(x,t)\lesssim \sqrt{t}\phi_t(x-1)$ when $x>1$. It also holds that there exists a universal constant $c(=0.67)$ such that
\begin{equation*}
    \frac{\bP_{z\sim \mathcal N(0,1)}\left\{z\ge \frac{x+1}{\sqrt{t}}\right\}}{\bP_{z\sim \mathcal N(0,1)}\left\{z\ge \frac{x-1}{\sqrt{t}}\right\}} \le c < 1
\end{equation*}
\end{lemma}
\begin{proof}[Proof of Lemma \ref{lemma:up-3}]
Notice that 
\begin{equation*}
\begin{aligned}
p(x, t) &= \int_U \frac{1}{\sqrt{2\pi t}} \exp\left(-\frac{(x-\mu)^2}{2t}\right)\cdot  f(\mu) \,\rd \mu \asymp \int_U \frac{1}{\sqrt{2\pi t}} \exp\left(-\frac{(x-\mu)^2}{2t}\right)\,\rd \mu \\
&= \bP_{z\sim \mathcal N(0,1)} \left\{\frac{x-1}{\sqrt{t}}\le z \le \frac{x+1}{\sqrt{t}}\right\},
\end{aligned}
\end{equation*}
which proves the first claim. Then the third claim is based on Lemma \ref{lemma:up-2}. For the second claim, notice that
\begin{equation*}
\frac{\bP_{z\sim \mathcal N(0,1)}\left\{z\ge \frac{x+1}{\sqrt{t}}\right\}}{\bP_{z\sim \mathcal N(0,1)}\left\{z\ge \frac{x-1}{\sqrt{t}}\right\}}  \le \frac{\frac{\sqrt{e}}{2}\cdot \frac{\sqrt{t}}{x+1}\exp\left(-\frac{(x+1)^2}{2t}\right)}{\bar{\Phi}(1)\cdot \left(\frac{\sqrt{t}}{x-1} \wedge 1\right)\exp\left(-\frac{(x-1)^2}{2t}\right)} = \frac{\sqrt{e}}{2\bar{\Phi}(1)}\exp\left(-\frac{2x}{t}\right)< \frac{\sqrt{e}}{2\bar{\Phi}(1)}\cdot e^{-2} < 0.67
\end{equation*}
when $t \le 1$. Here we applied \eqref{eqn:Gaussian-toy1} and $\frac{x-1}{\sqrt{t}} > 0$. 
Therefore,
\[\bP_{z\sim \mathcal N(0,1)} \left\{\frac{x-1}{\sqrt{t}}\le z \le \frac{x+1}{\sqrt{t}}\right\} \asymp \bP_{z\sim \mathcal N(0,1)}\left\{z\ge \frac{x-1}{\sqrt{t}}\right\},\]
which comes to our conclusion. Similarly, for all $x < -1$, we also have
\[p(x, t)\asymp\bP_{z\sim \mathcal N(0,1)} \left\{\frac{x-1}{\sqrt{t}}\le z \le \frac{x+1}{\sqrt{t}}\right\} \asymp \bP_{z\sim \mathcal N(0,1)}\left\{z\le \frac{x+1}{\sqrt{t}}\right\} \asymp \frac{\sqrt{t}}{-1-x}\cdot \exp\left(-\frac{(x+1)^2}{2t}\right). \]
\end{proof}
\begin{lemma}
\label{lemma:score-up}
We can upper bound the score function $s(x,t)=\psi(x,t)/p(x, t)$ as
\[s(x,t)^2 \le \frac{2}{t}\log \frac{1}{\sqrt{2\pi t}\cdot p(x, t)} ~~~\forall x\in \bR, t > 0.\]
In addition, when $|x| < 1+C\sqrt{t}$, it holds that $|s(x,t)| \lesssim \frac{1}{\sqrt{t}}$.
When $|x| < 1-\sqrt{t\log (1/t)}$ and $\alpha \ge 1$, we have $|s(x,t)| \lesssim 1$. 
\end{lemma}
\begin{proof}[Proof of Lemma \ref{lemma:score-up}]
Following Lemma 5 of \cite{wibisono_optimal_2024} and \cite{jiang_general_2009, saha2020nonparametric}, let $\theta\sim f$ and $X = \theta + \sqrt{t} Z \sim p(x, t)$ where $Z$ is a standard Gaussian random variable. Recall Tweedie's formula
\[s(x,t) = \frac{1}{t}\cdot \bE[\theta - X \mid X= x], \]
and note we have by Jensen's inequality,
\begin{equation*}
\begin{aligned}
s(x,t)^2 &\le \frac{1}{t^2} \bE[(\theta - X)^2 \mid X= x]\le \frac{2}{t} \log \bE[\exp\left((\theta - X)^2/2t \right)\mid X= x]\\
& =  \frac{2}{t}\cdot\log \int_{-1}^{1} \,\rd \theta \exp\left(\frac{(\theta-x)^2}{2t}\right)\cdot \frac{\exp\left(-\frac{(\theta-x)^2}{2t}\right) f(\theta)}{\int_{-1}^{1} \exp\left(-\frac{(\theta'-x)^2}{2t}\right) f(\theta') \,\rd \theta'}  = \frac{2}{t} \log \frac{\int_{-1}^{1} f(\theta) \,\rd \theta}{\sqrt{2\pi t} \cdot p(x,t)}\\
&=\frac{2}{t}\cdot \log \frac{1}{\sqrt{2\pi t} \cdot p(x, t)},
\end{aligned}
\end{equation*}
In addition, when $|x| < 1+C\sqrt{t}$, we know that $p(x, t)\ge c$ holds for a universal constant $c>0$ according to Lemma \ref{lemma:up-1}. For the density derivative $\psi(x,t)$, we have
\begin{equation*}
\begin{aligned}
|\psi(x,t)| &= \left|\phi_t' * f(x)\right| = \left|\int_\bR -\frac{x-\mu}{t}\phi_t(x-\mu) f(\mu) \,\rd\mu \right| \\
&\le \int_\bR \frac{|z|}{\sqrt{t}}\phi(z) f(x-\sqrt{t}z) \,\rd z \lesssim \frac{1}{\sqrt{t}} \int_\bR |z| \phi(z) \,\rd z \lesssim \frac{1}{\sqrt{t}}. 
\end{aligned}
\end{equation*}
Therefore, we have shown \(|s(x, t)| \lesssim \frac{1}{\sqrt{t}}\). When $|x|<1-\sqrt{t\log(1/t)}$ and $\alpha \ge 1$, it follows from integration by parts that 
\begin{equation*}
\begin{aligned}
\psi(x,t)&=\phi_t'* f(x)=\int_{-1}^{1} \phi_t'(x-\mu) f(\mu) \,\rd \mu \\
&= -f(1)\phi_t(x-1)+f(-1)\phi_t(x+1) +\int_{-1}^{1} \phi_t(x-\mu)f'(\mu) \,\rd \mu. 
\end{aligned}
\end{equation*}
Now, consider that 
\[\exp\left(-\frac{(x-1)^2}{2t}\right) \vee \exp\left(-\frac{(x+1)^2}{2t}\right) < \sqrt{t}~~~\Rightarrow~~~\phi_t(x-1) \vee \phi_t(x+1) \lesssim 1. \]
Then we can upper bound $|\psi(x,t)|$ as
\[|\psi(x,t)|\lesssim |f(-1)| + |f(1)| + \bE_{z\sim\mathcal N(0,1)} |f'(x+\sqrt{t} z)\mathbbm{1}_{\{|x+\sqrt{t}z| \le 1\}}|\lesssim 1,\]
which leads to $|s(x,t)| \lesssim 1$, and it comes to our conclusion.
\end{proof}
\begin{lemma}
\label{lemma:prob-1}
Suppose $0 < t < 1$. For $x > 1$, let \(q\) denote the conditional density of \(Y\) conditional on the event \(\{|Y| \leq 1\}\) where \(Y \sim N(x, t)\). Then for any positive constant $C' > 0$, we have
\[\mathbb P_{y\sim q} \left\{y < 1-\sqrt{2C't\log(1/t)}\right\} < t^{C'}. \]
\end{lemma}
\begin{proof}[Proof of Lemma \ref{lemma:prob-1}]
Notice that
\begin{equation*}
\begin{aligned}
&~~~\mathbb P_{y\sim q} \left\{y < 1-\sqrt{2C't\log(1/t)}\right\} = \frac{\mathbb P \left\{\mathcal N(x,t) \in [-1, 1-\sqrt{2C't\log(1/t)}]\right\}}{\mathbb P \left\{\mathcal N(x,t) \in [-1, 1]\right\}}\\
&< \frac{\mathbb P \left\{\mathcal N(x,t) \in [-1, 1-\sqrt{2C't\log(1/t)}]\right\} + \mathbb P\left\{\mathcal N(x,t) < -1\right\}}{\mathbb P \left\{\mathcal N(x,t) \in [-1, 1]\right\} + \mathbb P\left\{\mathcal N(x,t) < -1\right\}} = \frac{\mathbb P \left\{\mathcal N(x,t) < 1-\sqrt{2C't\log(1/t)}\right\}}{\mathbb P \left\{\mathcal N(x,t) < 1\right\}}\\
&= \frac{\mathbb P \left\{\mathcal N(0,1) > \frac{x-1}{\sqrt{t}}+\sqrt{2C'\log(1/t)}\right\}}{\mathbb P \left\{\mathcal N(0,1) > \frac{x-1}{\sqrt{t}}\right\}} \overset{(a)}{<} \exp\left(-\frac12\left(\frac{x-1}{\sqrt{t}}+\sqrt{2C'\log(1/t)}\right)^2 + \frac12\left(\frac{x-1}{\sqrt{t}}\right)^2\right)\\
&< \exp(-C'\log(1/t)) = t^{C'},
\end{aligned}
\end{equation*}
which comes to our conclusion. Here, inequality $(a)$ holds because of the following fact. For all $x_1 > x_2 > 0$, it holds that
\begin{align*}
\frac{\mathbb P\left\{\mathcal N(0,1) > x_1\right\}}{\mathbb P\left\{\mathcal N(0,1) > x_2\right\}} &= \frac{\int_{0}^{\infty} \frac{1}{\sqrt{2\pi}} \exp\left(-\frac{(x+x_1)^2}{2}\right) 
\,\rd x}{\int_{0}^{\infty} \frac{1}{\sqrt{2\pi}} \exp\left(-\frac{(x+x_2)^2}{2}\right) \,\rd x} \le \frac{\int_{0}^{\infty} \frac{1}{\sqrt{2\pi}} \exp\left(-\frac{(x+x_2)^2}{2}\right) \cdot \exp\left(-\frac{x_1^2-x_2^2}{2}\right) \,\rd x}{\int_{0}^{\infty} \frac{1}{\sqrt{2\pi}} \exp\left(-\frac{(x+x_2)^2}{2}\right) \,\rd x}\\
&= \exp\left(-\frac{x_1^2-x_2^2}{2}\right). 
\end{align*}
\end{proof}

%% file: upper/very-high-noise.tex
\begin{proof}[Proof of Theorem \ref{thm:upperbound_large_t}]
Since \(s(x,t) = \frac{\psi(x, t)}{p(x, t)}\), direct calculation yields
\begin{equation}\label{eqn:score_bound_large_t_I}
    \int_{\R} |\hat{s}(x, t) - s(x, t)|^2 \, p(x, t) \, \rd x \le \int_{\R} \frac{|\hat{\psi}(x, t)p(x, t) - \psi(x, t)\hat{p}(x, t)|^2}{p(x,t) \varepsilon(x, t)^2} \, \rd x. 
\end{equation}
Examining the numerator, consider
\[\hat{\psi}(x, t)p(x, t) = -\frac{x}{t} \hat{p}(x, t) p(x, t) + \left(\frac{1}{nt}\sum_{i=1}^{n} \mu_i\varphi_t(x-\mu_i)\right)p(x,t)\]
and 
\[\psi(x, t)\hat{p}(x, t) = -\frac{x}{t}p(x, t)\hat{p}(x, t) + \left(\frac{1}{t} \int_{-1}^{1} \mu\varphi_t(x-\mu)f(\mu) \, \rd \mu\right)\hat{p}(x, t)\]
hold according to the definition of $\hat{p}(x,t)$ and $\hat{\psi}(x,t)$ stated in Section \ref{section:methodology_very_high_noise}. We have
\begin{align}
    &~~~~|\hat{\psi}(x, t)p(x, t) - \psi(x, t) \hat{p}(x, t)|^2 \nonumber \\
    &= \left|\left(\frac{1}{nt}\sum_{i=1}^{n} \mu_i \varphi_t(x-\mu_i) \right)p(x, t) - \left(\frac{1}{t} \int_{-1}^{1} \mu\varphi_t(x-\mu)f(\mu) \, \rd \mu\right)\hat{p}(x, t)\right|^2 \nonumber \\
    &\lesssim \left|\frac{1}{nt}\sum_{i=1}^{n} \mu_i \varphi_t(x-\mu_i) - \frac{1}{t} \int_{-1}^{1} \mu \varphi_t(x-\mu) f(\mu) \, \rd\mu \right|^2 p(x, t)^2 \nonumber\\
    &~~~+ \left|\frac{1}{t} \int_{-1}^{1} \mu \varphi_t(x-\mu) f(\mu) \, \rd\mu \right|^2 |\hat{p}(x, t) - p(x, t)|^2. \label{eqn:large_t_numerator}
\end{align}
Recalling that \(\mu_1, \mu_2, \ldots, \mu_n \sim f\), it is clear  
\begin{align*}
    &\bE \left|\frac{1}{nt}\sum_{i=1}^{n} \mu_i \varphi_t(x-\mu_i) - \frac{1}{t} \int_{-1}^{1} \mu \varphi_t(x-\mu) f(\mu) \, \rd\mu \right|^2 = \frac{1}{nt^2} \Var\left(\mu_1 \varphi_t(x-\mu_1)\right) \\
    &~~~\le \frac{1}{nt^2} \bE\left(\mu_1^2 \varphi_t(x-\mu_1)^2\right) = \frac{1}{nt^2} \int_{-1}^{1} \frac{\mu^2}{2\pi t} e^{-\frac{(x-\mu)^2}{t}} \,\rd\mu \overset{(a)}{\lesssim} \frac{1}{nt^3} e^{-\frac{(|x| - 1)^2}{t}}. 
\end{align*}
Here, (a) holds because when \(|x|\ge 1\): \(\int_{-1}^1 \mu^2 e^{-\frac{(x-\mu)^2}{t}} \,\rd\mu \le \int_{-1}^1 \mu^2 e^{-\frac{(|x|-1)^2}{t}}\,\rd\mu \lesssim e^{-\frac{(|x|-1)^2}{t}}\). When \(|x| < 1\): \(\int_{-1}^1 \mu^2 e^{-\frac{(x-\mu)^2}{t}} \,\rd\mu \lesssim 1 \le e^{\frac1t}\cdot e^{-\frac{(|x|-1)^2}{t}} \lesssim e^{-\frac{(|x|-1)^2}{t}}\).
Now consider \(p(x, t) \lesssim \frac{1}{\sqrt{t}} e^{-\frac{(|x|-1)^2}{2t}}\), and so the first term in (\ref{eqn:large_t_numerator}) can be bounded. Let us turn our attention to the second term. Note since \(f \in \mathcal{F}_\alpha\) that we have \(p(x, t) \ge \varepsilon(x, t)\). Combined with $\hat{p}(x,t):=\varepsilon(x,t) \vee \frac{1}{n} \sum_{i=1}^{n} \varphi_t(x-\mu_i)$, it follows
\begin{align*}
    \bE\left|\hat{p}(x, t) - p(x, t)\right|^2 &\le \bE\left|\frac{1}{n} \sum_{i=1}^{n} \varphi_t(x-\mu_i) - p(x, t)\right|^2 = \frac{1}{n} \Var\left(\varphi_t(x-\mu_1)\right) \le \frac{1}{n} \bE\left(\varphi_t(x-\mu_1)^2\right) \\
    &= \frac{1}{n} \int_{-1}^{1} \frac{1}{2\pi t} e^{-\frac{(x - \mu)^2}{t}} f(\mu) \, \rd \mu \lesssim \frac{1}{nt} e^{-\frac{(|x| - 1)^2}{t}}. 
\end{align*}
Furthermore, observe \(\left|\frac{1}{t} \int_{-1}^{1} \mu \varphi_t(x-\mu)f(\mu) \, \rd \mu\right|^2 \lesssim \frac{1}{t^{3}} e^{-\frac{(|x| - 1)^2}{t}}\), and so we can bound the second term in (\ref{eqn:large_t_numerator}). Putting together our bounds, it follows from (\ref{eqn:large_t_numerator}), 
\begin{align*}
    \bE |\hat{\psi}(x, t)p(x, t) - \psi(x, t) \hat{p}(x, t)|^2 &\lesssim \frac{1}{nt^4} e^{-\frac{2(|x|-1)^2}{t}}. 
\end{align*}
Consequently from (\ref{eqn:score_bound_large_t_I}), we have
\begin{align*}
    &\bE \int_{\R} |\hat{s}(x, t) - s(x, t)|^2 p(x, t) \,\rd x \lesssim \int_{\R} \frac{1}{nt^4} e^{-\frac{2(|x|-1)^2}{t}} \cdot \frac{1}{p(x, t) \varepsilon(x, t)^2} \, \rd x \\
    &\lesssim \frac{1}{nt^4} \int_{\R} e^{-\frac{2(|x|-1)^2}{t}} \cdot t^{3/2} e^{\frac{3(|x| + 1)^2}{2t}} \, \rd x \lesssim \frac{1}{nt^2} \int_{\R} \frac{1}{\sqrt{t}} e^{-\frac{2(|x|-1)^2}{t}} \cdot e^{\frac{3(|x| + 1)^2}{2t}} \, \rd x \\
&= \frac{e^{\frac{24}{t}} }{nt^2}\int_{\R} \frac{1}{\sqrt{t}} e^{-\frac{(|x| - 7)^2}{2t}} \,\rd x\lesssim \frac{e^{\frac{24}{c}}}{nt^2}
\end{align*}
Here, we have used \(p(x, t)\ge \varepsilon(x, t) \gtrsim \frac{1}{\sqrt{t}}e^{-\frac{(|x| + 1)^2}{2t}}\) and we have used \(t \ge c\) to obtain the final line. The proof is complete.  
\end{proof}

%% file: upper/high-noise.tex
In this section, we examine the high-noise regime where $n^{-\frac{2}{2\alpha+1}} \le t \le 1$ and provide the proofs of Lemmas \ref{lemma:score-split-1}, \ref{lemma:up-6}, and \ref{lemma:up-7} stated in Section \ref{section:proof_sketch_upper_bound}. Theorem \ref{thm:up-high-1} is obtained as a direct consequence. Before doing so, some preliminary development is needed. Given the following unbiased estimators for $\psi(x,t)$ and $p(x,t)$, 
\begin{equation}
\label{eqn:unbiase-psi-p}
\hat{\psi}(x, t) := \frac{1}{n}\sum_{j=1}^n \frac{\mu_j-x}{t}\phi_t(x-\mu_j), ~~~\hat{p}(x, t) :=  \frac{1}{n}\sum_{j=1}^n \phi_t(x-\mu_j),
\end{equation}
the variance of $\hat{\psi}(x, t)$ can be bounded as
\begin{equation}
\label{eqn:6-1}
\begin{aligned}
\mathrm{Var}[\hat{\psi}(x, t)] &\le \frac{1}{2\pi t n} \int_U \frac{(\mu-x)^2}{t^2}\exp\left(-\frac{(\mu-x)^2}{t}\right)\cdot f(\mu) \,\rd \mu \\
&\overset{(a)}{=} \frac{1}{2\pi t n} \int_{(-1-x)/\sqrt{t}}^{(1-x)/\sqrt{t}} \frac{z^2}{t}\exp\left(-z^2\right) f(x+\sqrt{t} z) \cdot \sqrt{t}\,\rd z\\
&\lesssim \frac{1}{n t^{3/2}} \int_{\bR} \frac{1}{\sqrt{2\pi}} z^2 \exp\left(-\frac{z^2}{2}\right)\cdot \exp\left(-\frac{z^2}{2}\right) f(x+\sqrt{t} z) \,\rd z \\
&\overset{(b)}{\lesssim} \frac{1}{nt^{3/2}} \int_{\bR}\frac{1}{\sqrt{2\pi}}\exp\left(-\frac{z^2}{2}\right) f(x+\sqrt{t} z) \,\rd z = \frac{1}{nt^{3/2}} \cdot \bE_{z\sim\mathcal N(0,1)} f(x+\sqrt{t} z).
\end{aligned}
\end{equation}
Here, $(a)$ holds by letting $\mu=x+\sqrt{t} z$, $(b)$ holds because $z^2/2 < \exp(z^2/2)~\Rightarrow z^2 \exp(-z^2/2) < 2$ for all $z\in \bR$. Similarly, we can also bound the variance of $\hat{p}(x, t)$.
\begin{equation}
\label{eqn:6-2}
\begin{aligned}
\mathrm{Var}\left[\hat{p}(x, t)\right] &\le \frac{1}{2\pi tn} \int_U \exp\left(-\frac{(\mu-x)^2}{t}\right)\cdot f(\mu)\,\rd \mu \\
& \lesssim \frac{1}{n t^{1/2}}\cdot  \bE_{z\sim \mathcal N(0,1)} f(x+\sqrt{t/2}z). 
\end{aligned}
\end{equation}
Now, we estimate the score function $s(x,t) =\psi(x,t)/p(x,t)$ by using the regularized estimator (\ref{def:shat_high}) defined in Section \ref{section:methodology_high_noise}, namely 
\begin{equation}
\label{eqn: score-regularized}
\hat{s}(x,t) := \frac{\hat{\psi}(x, t)}{\hat{p}(x, t)\vee \varepsilon(x,t)} := \frac{\hat{\psi}(x, t)}{\hat{p}^\varepsilon_n(x,t)}.
\end{equation}
Recall the regularizer $\varepsilon(x,t)$ was chosen as
\begin{equation}
\label{eqn:regularizer}
\varepsilon(x,t) := c_d \int_U \frac{1}{\sqrt{2\pi t}} \exp\left(-\frac{(x-\mu)^2}{2t}\right) \,\rd \mu .
\end{equation}
Notice that we choose a function as the regularizer instead of a constant. Since we know that $c_d \le f(x) \le C_d$ holds for all $x\in U$, we conclude that
\begin{equation}
\label{eqn:regularizer-1}
\varepsilon(x,t) \le p(x,t) \le \frac{C_d}{c_d}\cdot \varepsilon(x,t),
\end{equation}
which directly leads to 
\begin{equation}
\label{eqn:regularized-density}
(\hat{p}^\varepsilon_n(x,t)-p(x,t))^2 \le (\hat{p}(x, t)-p(x,t))^2 ~~\text{and}~~ \hat{p}^\varepsilon_n(x,t) \ge \varepsilon(x,t)\gtrsim p(x,t).
\end{equation}
Then, the score integrated estimation error can be split according to Lemma \ref{lemma:score-split-1}, which is to say, 
\begin{align*}
&\int_{D} \bE \left(\hat{s}(x,t)-s(x,t)\right)^2\cdot p(x,t) \,\rd x \\
&\quad\quad\quad \lesssim \int_D \frac{\bE (\hat{\psi}(x, t)-\psi(x,t))^2}{p(x,t)}\,\rd x + \int_D s(x,t)^2 \cdot \frac{\bE (\hat{p}(x, t)-p(x,t))^2}{p(x,t)} \,\rd x.
\end{align*}
for any subset \(D \subseteq \R\). 

\begin{proof}[Proof of Lemma \ref{lemma:score-split-1}]
By using \eqref{eqn:regularizer-1} and \eqref{eqn:regularized-density}, we have
\begin{equation}
\label{eqn:score-1}
\begin{aligned}
&\int_\bR \bE \left(\hat{s}(x,t)-s(x,t)\right)^2\cdot p(x,t) \,\rd x = \int_\bR \bE \left(\frac{\hat{\psi}(x, t)}{\hat{p}^\varepsilon(x,t)}-\frac{\psi(x,t)}{p(x,t)}\right)^2\cdot p(x,t) \,\rd x\\
&~~~= \bE\int_\bR \frac{(\hat{\psi}(x, t)p(x,t)-\psi(x,t)p(x,t)+\psi(x,t)p(x,t)-\psi(x,t)\hat{p}^\varepsilon(x,t))^2}{\hat{p}^\varepsilon(x,t)^2 p(x,t)} \,\rd x \\
&~~~\lesssim \bE \int_\bR \frac{p(x,t)^2\cdot (\hat{\psi}(x, t)-\psi(x,t))^2 + \psi(x,t)^2\cdot (\hat{p}^\varepsilon(x,t)-p(x,t))^2}{p(x,t)^3}\,\rd x \\ 
&~~~\lesssim \int_\bR \frac{\bE (\hat{\psi}(x, t)-\psi(x,t))^2}{p(x,t)}\,\rd x + \int_\bR s(x,t)^2 \cdot \frac{\bE (\hat{p}(x, t)-p(x,t))^2}{p(x,t)} \,\rd x.
\end{aligned}
\end{equation}
The same argument goes through when the domain of integration is \(D\) instead of \(\R\).
\end{proof}
In order to bound the mean integrated squared score error, we need to upper bound the following two terms,
\[I_1 := \int_\bR \frac{\bE (\hat{\psi}(x, t)-\psi(x,t))^2}{p(x,t)}\,\rd x, ~~~I_2 := \int_\bR s(x,t)^2 \cdot \frac{\bE (\hat{p}(x, t)-p(x,t))^2}{p(x,t)} \,\rd x. \]
These can be bounded by Lemmas \ref{lemma:up-6} and \ref{lemma:up-7} respectively, which we prove now. 
\begin{proof}[Proof of Lemma \ref{lemma:up-6}]
By Lemma \ref{lemma:up-1}, we have $p(x,t)\gtrsim 1$ for $|x| < 1+C\sqrt{t}$, so
\begin{equation*}
\begin{aligned}
& \int_{|x|<1+C\sqrt{t}} \frac{\bE (\hat{\psi}(x, t)-\psi(x,t))^2}{p(x,t)}\,\rd x \lesssim \int_\bR \bE (\hat{\psi}(x, t)-\psi(x,t))^2\,\rd x = \int_\bR \mathrm{Var}\left[\hat{\psi}(x, t)\right] \,\rd x \\
&\lesssim \frac{1}{nt^{3/2}} \int_\bR \bE_{z\sim\mathcal N(0,1)}f(x+\sqrt{t}z) \,\rd x = \frac{1}{nt^{3/2}}. 
\end{aligned}
\end{equation*}
Next we need to consider the case where $|x|\ge 1+C\sqrt{t}$. According to Lemma \ref{lemma:up-3} and \eqref{eqn:6-1}, for $x > 1+C\sqrt{t}$, it holds that
\begin{equation}
\label{eqn:6-3}
\begin{aligned}
\bE \left(\hat{\psi}(x, t)-\psi(x,t)\right)^2 &\lesssim \frac{1}{nt^{3/2}} \int_{\bR} z^2\exp(-z^2)\cdot f(x-\sqrt{t} z) \,\rd z \lesssim \frac{1}{nt^{3/2}} \int_{(x-1)/\sqrt{t}}^{(x+1)/\sqrt{t}} z^2\exp(-z^2) \,\rd z\\
&\le \frac{1}{nt^{3/2}}
\exp\left(-\frac{2(x-1)^2}{3t}\right)\cdot \int_{(x-1)/\sqrt{t}}^{(x+1)/\sqrt{t}} z^2 \exp(-z^2/3)\,\rd z\\
&\lesssim \frac{1}{nt^{3/2}}
\exp\left(-\frac{2(x-1)^2}{3t}\right). 
\end{aligned}
\end{equation}
Therefore,
\begin{equation*}
\begin{aligned}
&\int_{1+C\sqrt{t}}^{\infty} \frac{\bE (\hat{\psi}(x, t)-\psi(x,t))^2}{p(x,t)}\,\rd x \lesssim \frac{1}{nt^{3/2}}\int_{1+C\sqrt{t}}^{\infty} \frac{\exp\left(-\frac{2(x-1)^2}{3t}\right)}{\frac{\sqrt{t}}{x-1}\exp\left(-\frac{(x-1)^2}{2t}\right)}\,\rd x \\
&~~~= \frac{1}{nt^{3/2}}\int_{1+C\sqrt{t}}^{\infty} \frac{x-1}{\sqrt{t}}\cdot \exp\left(-\frac{(x-1)^2}{6t}\right)\,\rd x = \frac{1}{nt^{3/2}}\int_C^{\infty} y\exp(-y^2/6)\cdot \sqrt{t} \,\rd y \lesssim \frac{1}{nt}.
\end{aligned}
\end{equation*}
Similarly, it holds that
\[\int_{-\infty}^{-1-C\sqrt{t}} \frac{\bE (\hat{\psi}(x, t)-\psi(x,t))^2}{p(x,t)}\,\rd x \lesssim \frac{1}{nt}.\]
After adding them up, we conclude
\[I_1 = \int_\bR \frac{\bE (\hat{\psi}(x, t)-\psi(x,t))^2}{p(x,t)}\,\rd x \lesssim \frac{1}{nt^{3/2}}+\frac{1}{nt} \lesssim \frac{1}{nt^{3/2}}. \]
\end{proof}

\begin{proof}[Proof of Lemma \ref{lemma:up-7}]
From (\ref{eqn:6-2}), we have
\[\bE (\hat{p}(x, t)-p(x,t))^2 = \mathrm{Var}\left[\hat{p}(x, t)\right] \lesssim \frac{1}{n t^{1/2}}\cdot  \bE_{z\sim \mathcal N(0,1)} f(x+\sqrt{t/2}z).\]
For $|x| < 1+C\sqrt{t}$, we have $p(x,t) > c$ by Lemma \ref{lemma:up-1} and $s(x,t)^2\lesssim 1/t$ by Lemma \ref{lemma:score-up}, then
\begin{equation*}
\begin{aligned}
&~~~\int_{|x|<1+C\sqrt{t}} s(x,t)^2 \cdot \frac{\bE (\hat{p}(x, t)-p(x,t))^2}{p(x,t)} \,\rd x \lesssim \int_{|x|<1+C\sqrt{t}} \frac{1}{t}\cdot \bE (\hat{p}(x, t)-p(x,t))^2 \,\rd x \\
&< \frac{1}{nt^{3/2}} \int_\bR \bE_{z\sim \mathcal N(0,1)} f(x+\sqrt{t/2}z) \,\rd x = \frac{1}{nt^{3/2}}  \bE_{z\sim \mathcal N(0,1)}\int_\bR f(x+\sqrt{t/2}z) \,\rd x = \frac{1}{nt^{3/2}}.
\end{aligned}
\end{equation*}
For $x\ge 1+C\sqrt{t}$, we apply Lemma \ref{lemma:up-3} and obtain that
\begin{equation*}
\begin{aligned}
\bE (\hat{p}(x, t)-p(x,t))^2 &\le \frac{1}{nt^{1/2}} \bP_{z\sim \mathcal N(0,1)}\left[\left|x+\sqrt{t/2}z\right|\le 1\right] \lesssim \frac{1}{nt^{1/2}}\cdot \frac{\sqrt{t/2}}{x-1}\exp\left(-\frac{(x-1)^2}{t}\right)\\
p(x,t) & \asymp \frac{\sqrt{t}}{x-1}\exp\left(-\frac{(x-1)^2}{2t}\right)
\end{aligned}
\end{equation*}
Therefore, we conclude that
\begin{equation*}
\begin{aligned}
&\int_{1+C\sqrt{t}}^{\infty} s(x,t)^2 \cdot \frac{\bE (\hat{p}(x, t)-p(x,t))^2}{p(x,t)} \,\rd x \lesssim  \int_{1+C\sqrt{t}}^{\infty} s(x,t)^2\cdot \frac{1}{nt^{1/2}}\exp\left(-\frac{(x-1)^2}{2t}\right)\,\rd x \\
&~~~\le \frac{1}{nt^{3/2}}\int_{1+C\sqrt{t}}^{\infty} \log \frac{1}{\sqrt{2\pi t}\cdot p(x,t)} \cdot \exp\left(-\frac{(x-1)^2}{2t}\right)\,\rd x\\
&~~~\asymp  \frac{1}{nt^{3/2}}\int_{1+C\sqrt{t}}^{\infty} \left(\log \frac{1}{\sqrt{2\pi t}} + \frac{(x-1)^2}{2t} + \log \frac{x-1}{\sqrt{t}}\right)\cdot \exp\left(-\frac{(x-1)^2}{2t}\right)\,\rd x\\
&~~~= \frac{1}{nt^{3/2}}\int_{C}^{\infty} \left(\log \frac{1}{\sqrt{2\pi t}} + \frac{y^2}{2} + \log y\right)\cdot \exp\left(-\frac{y^2}{2}\right)\cdot \sqrt{t}\,\rd y\lesssim \frac{\log(1/t)}{nt}. 
\end{aligned}
\end{equation*}
Similarly, 
\[\int_{-\infty}^{-1-C\sqrt{t}} s(x,t)^2 \cdot \frac{\bE (\hat{p}(x, t)-p(x,t))^2}{p(x,t)} \,\rd x \lesssim \frac{\log(1/t)}{nt}.\]
After summing them up, we finally conclude that
\[I_2 = \int_\bR s(x,t)^2 \cdot \frac{\bE (\hat{p}(x, t)-p(x,t))^2}{p(x,t)} \,\rd x \lesssim \frac{1}{nt^{3/2}} + \frac{\log(1/t)}{nt} \lesssim \frac{1}{nt^{3/2}}. \]
\end{proof}
Finally, we combine Lemma \ref{lemma:up-6} and Lemma \ref{lemma:up-7}, and conclude Theorem \ref{thm:up-high-1}.

%% file: upper/low-noise.tex
In this section, we work towards giving a proof of Theorem \ref{thm:up-1} by way of proving Propositions \ref{prop:up-boundary} and \ref{prop:up-external}. Recall Proposition \ref{prop:up-internal} was already proved in Section \ref{section:proof_sketch_upper_bound}. Recall we are working in the low noise regime ($t< n^{-\frac{2}{2\alpha+1}}$) with $\alpha \ge 1$, and further recall from Section \ref{section:methodology_low_noise} that we use a kernel-based estimator \(\hat{f}_n\) of the original distribution $f(x)$ as well as \(\hat{f}_n^{(k)}\) to estimate the \(k\)th derivative $f^{(k)}(x)$ for \(1 \leq k \leq \lfloor \alpha \rfloor\). Recall Lemma \ref{lemma:upper-1} states the errors of these estimators. Further recall we construct a separate score estimator on each of three regions: the internal part $D_1 := \{x:~ |x| < 1-\sqrt{Ct\log(1/t)}\}$, the boundary part $D_2 := \{x:~1-\sqrt{Ct\log(1/t)}\le |x| \le 1+C\sqrt{t}\}$, and the external part $D_3 := \{x:~|x|>1+C\sqrt{t}\}$, where $C>0$ is a constant.

Let us first work towards proving Proposition \ref{prop:up-boundary}, which addresses \(D_2\). Recall from Section \ref{section:methodology_low_noise} we use
\[\hat{s}_{2, n}(x) = \frac{\hat{\psi}_n(x,t)}{\phi_t* \hat{f}_n (x)} \mathbbm{1}_{\Omega} + \frac{\phi_t' * u(x)}{\phi_t * u(x)} \mathbbm{1}_{\Omega^c}\]
where the event \(\Omega\) is defined as \(\Omega := \left\{\hat{f}_n(x) \ge \frac{c_d}{2} \text{ for all } x \in [-1, 1]\right\}\), \(u(x) = \frac12 \mathbbm{1}_{\left\{|x|\le 1\right\}},\) and 
\[\hat{\psi}_n(x,t) = \phi_t(x+1)\hat{f}_n(-1)-\phi_t(x-1)\hat{f}_n(1)+\int_{-1}^{1} \phi_t(x-\mu)\hat{f}'_n(\mu) \,\rd\mu.\]

Here, we notice that $\left\{\|\hat{f}_n-f\|_\infty < \frac{c_d}{2}\right\} \subset \Omega$ since \(f \geq c_d\). Since \(\|\hat{f}_n-f\|_\infty < \frac{c_d}{2}\) is a high-probability event, it follows \(\Omega^c\) is a low-probability event. Consequently, it turns out most of our attention will be focused on the event \(\Omega\), in which case $\hat{f}_n$ is also $\alpha$-H\"{o}lder smooth on $[-1,1]$. In the following lemma, we study the probability $\mathbb P (\Omega^c)$. 
\begin{lemma}
\label{lemma:infty-probability}
We have
\[\bP(\Omega^c) \lesssim \exp\left(-C_1\cdot n^{\frac{2\alpha-1}{2\alpha+1}}\right).\]
Here $C_1$ is a universal constant. 
\end{lemma}
\begin{proof}
Since \(\bP(\Omega^c) \le \mathbb P\left\{\|\hat{f}_n-f\|_\infty > \frac{c_d}{2}\right\}\), it suffices to obtain a tail bound for the \(\ell_\infty\) loss. According to Theorem 6.3.7 of \cite{gine2021mathematical}, we know the kernel-based estimator $\hat{f}_n$ proposed in Lemma \ref{lemma:upper-1} has its $\ell_\infty$ mean estimation error bounded by:
\[\bE \|\hat{f}_n - f\|_\infty \lesssim \left(\frac{\log n}{n}\right)^{\frac{\alpha}{2\alpha+1}} < \frac{c_d}{4}. \]
By Remark 5.1.14 of \cite{gine2021mathematical}, the tail of concentration error $\|\hat{f}_n - f\|_\infty - \bE \|\hat{f}_n - f\|_\infty $ can be bounded as:
\[\mathbb P\left\{\|\hat{f}_n-f\|_\infty \ge \bE \|\hat{f}_n - f\|_\infty + \sqrt{Cx\cdot n^{-\frac{2\alpha-1}{2\alpha+1}}}\right\} \le e^{-x}\]
where $C$ is a universal constant. Let $x = \frac{c_d^2}{16C} n^{\frac{2\alpha-1}{2\alpha+1}}$, we finally obtain
\begin{align*}
\mathbb P \left\{\|\hat{f}_n - f\|_\infty > \frac{c_d}{2}\right\} \le \mathbb P\left\{\|\hat{f}_n-f\|_\infty \ge \bE \|\hat{f}_n - f\|_\infty + \sqrt{Cx\cdot n^{-\frac{2\alpha-1}{2\alpha+1}}}\right\} \le \exp\left(-C_1\cdot n^{\frac{2\alpha-1}{2\alpha+1}}\right)
\end{align*}
where $C_1$ is a universal constant. The proof is complete. 
\end{proof}

\begin{proof}[Proof of Proposition \ref{prop:up-boundary}]
It is clear we can split the score-matching error across \(\Omega\) and \(\Omega^c\), that is 
\begin{align*}
&\bE \left(\int_{D_2} \left|\hat{s}_{2,n}(x,t)-s(x,t)\right|^2 p(x,t) \,\rd x\right) \\
&= \bE \left(\mathbbm{1}_\Omega \int_{D_2} \left|\hat{s}_{2,n}(x,t)-s(x,t)\right|^2 p(x,t) \,\rd x\right) + \bE \left(\mathbbm{1}_{\Omega^c}\int_{D_2} \left|\hat{s}_{2,n}(x,t)-s(x,t)\right|^2 p(x,t) \,\rd x\right).
\end{align*}
By Remark \ref{remark:alpha}, we have \(\bE \left(\mathbbm{1}_{\Omega^c}\int_{D_2} \left|\hat{s}_{2,n}(x,t)-s(x,t)\right|^2 p( \,\rd x\right) \lesssim \bP\left(\Omega^c\right) \lesssim e^{-C_1 n^{\frac{2(\alpha-1)}{2\alpha+1}}}\) for some universal constant \(C_1 > 0\). Here, we have used Lemma \ref{lemma:infty-probability}. Hence, it remains to examine the score-matching error on the event \(\Omega\). In what follows, we work on the event \(\Omega\).

When $|x| < 1+C\sqrt{t}$, we already know $p(x,t)\gtrsim 1$ and $\phi_t* \hat{f}_n (x) \gtrsim 1$ by Lemma \ref{lemma:up-1}. Then, we only need to bound the numerator term since
\begin{equation}
\label{eqn:J-ineq}
\left(\hat{s}_{2,n}(x,t)-s(x,t)\right)^2\cdot p(x,t) = \left(\frac{J(x,t)}{p(x,t)\cdot (\phi_t * \hat{f}_n)(x)}\right)^2\cdot p(x,t) \asymp J(x,t)^2.
\end{equation}
Here, $J(x,t) = \hat{\psi}_n(x,t) p(x,t) - \psi(x,t) (\phi_t * \hat{f}_n)(x)$. Since 
\begin{align*}
\psi(x,t) &= (\phi_t' * f)(x)=\varphi_t(x+1)f(-1) - \varphi_t(x-1)f(1) + \int_{-1}^{1} \varphi_t(x-\mu) f'(\mu) \,\rd\mu, \\
p(x,t) &= (\varphi_t * f)(x) = \int_{-1}^{1} \varphi_t(x-\mu) f(\mu) \,\rd\mu, \\
\hat{\psi}_n(x,t)&= \varphi_t(x+1)\hat{f}_n(-1) - \varphi_t(x-1)\hat{f}_n(1) + \int_{-1}^{1} \varphi_t(x-\mu) \hat{f}_n'(\mu) \,\rd\mu, \\
(\phi_t * \hat{f}_n)(x) &= \int_{-1}^{1} \varphi_t(x-\mu) \hat{f}_n(\mu) \,\rd\mu, 
\end{align*}
it follows that $|J(x,t)|^2 \lesssim J_1(x,t) + J_2(x,t) + J_3(x,t)$ where
\begin{align}
\label{eqn:J-123}
\begin{split}
J_1(x,t) &:= \left|\varphi_t(x-1)\right|^2 \left|f(1)\int_{-1}^{1}\varphi_t(x-\mu) \hat{f}_n(\mu) \, d\mu - \hat{f}_n(1)\int_{-1}^{1}\varphi_t(x-\mu)f(\mu)\, d\mu \right|^2, \\
J_2(x,t) &:= \left|\varphi_t(x+1)\right|^2 \left|f(-1)\int_{-1}^{1}\varphi_t(x-\mu) \hat{f}_n(\mu) \,\rd\mu - \hat{f}_n(-1)\int_{-1}^{1}\varphi_t(x-\mu)f(\mu)\,\rd\mu \right|^2, \\
J_3(x,t) &:= \left[\int_{[-1,1]^2} \varphi_t(x-\mu)\varphi_t(x-\nu) \left(f'(\mu) \hat{f}_n(\nu) - \hat{f}_n'(\mu) f(\nu)\right)\,\rd\mu \,\rd\nu\right]^2.
\end{split}
\end{align}
The term $\bE(J_2(x,t)\mathbbm{1}_{\Omega})$ is exponentially small as established in Lemma \ref{lemma:J2}. The expectations \(\bE(J_1(x, t)\mathbbm{1}_{\Omega})\) and \(\bE(J_3(x, t)\mathbbm{1}_{\Omega})\) can be bounded using Lemmas \ref{lemma:J1} and \ref{lemma:J3} respectively. After summing the three terms up, we conclude that
\begin{align*}
\bE\left(\mathbbm{1}_{\Omega} \int_{D_2^+} \left(\hat{s}_{2,n}(x,t)-s(x,t)\right)^2 p(x,t) \,\rd x\right) &\lesssim \sqrt{t\log(1/t)}\cdot \left(\log^{\alpha}(1/t) n^{-\frac{2(\alpha-1)}{2\alpha+1}} + n^{-\frac{2(\alpha-1)}{2\alpha+1}}\right) \\
&\lesssim \sqrt{t} \log^{\alpha+1/2}(1/t)\cdot n^{-\frac{2(\alpha-1)}{2\alpha+1}}.
\end{align*}
The proof is complete. 
\end{proof}

\begin{proof}[Proof of Lemma \ref{lemma:J1}]
Without loss of generality, assume \(x \in D_2^{+}\) where \(D_2^{+} := D_2 \cap \R^{+}\). By Taylor expansion around $1$ and Jensen's inequality, we know that $x\in D_2^+$ implies
\begin{align}
\label{eqn:Taylor}
\left|\int_{-1}^{1} \varphi_t(x-\mu) f(\mu)\,\rd\mu - \sum_{k=0}^{\lfloor \alpha \rfloor} \frac{f^{(k)}(1)}{k!} \int_{-1}^{1} (\mu - 1)^k \varphi_t(x-\mu) d\mu\right| &\lesssim \int_{-1}^{1} \varphi_t(x-\mu) |\mu-1|^{\alpha} \,\rd\mu \nonumber \\
&\lesssim (t\log(1/t))^{\alpha/2}.
\end{align}
Similarly, 
\begin{align*}
\left|\int_{-1}^{1} \varphi_t(x-\mu) \hat{f}_n(\mu)\,\rd\mu - \sum_{k=0}^{\lfloor \alpha \rfloor} \frac{\hat{f}_n^{(k)}(1)}{k!} \int_{-1}^{1} (\mu - 1)^k \varphi_t(x-\mu) d\mu\right| \lesssim (t\log(1/t))^{\alpha/2}.
\end{align*}
Therefore, we have for $x\in D_2^+$, \begin{align*}
&\bE \left|f(1)\int_{-1}^{1}\varphi_t(x-\mu) \hat{f}_n(\mu) \,\rd\mu - \hat{f}_n(1)\int_{-1}^{1}\varphi_t(x-\mu)f(\mu)\,\rd\mu \right|^2 \, dx \\
&\lesssim \left(t\log(1/t)\right)^\alpha + \sum_{k=0}^{\lfloor \alpha \rfloor} \frac{1}{k!} \left(\int_{-1}^{1} |\mu-1|^{2k} \varphi_t(x-\mu) \,\rd\mu\right) \bE \left|f(1)\hat{f}_n^{(k)}(1) - \hat{f}_n(1)f^{(k)}(1)\right|^2.
\end{align*}
The key observation is that the summand corresponding to \(k = 0\) is actually zero, since \(|f(1)\hat{f}_n(1) - \hat{f}_n(1)f(1)| = 0\). This cancellation is precisely what makes score estimation at the boundary faster than estimating the derivative of the density at the boundary. Continuing the calculation, we have 
\begin{align*}
&\left(t\log(1/t)\right)^\alpha + \sum_{k=0}^{\lfloor \alpha \rfloor} \frac{1}{k!} \left(\int_{-1}^{1} |\mu-1|^{2k} \varphi_t(x-\mu) \,\rd\mu\right) \bE \left|f(1)\hat{f}_n^{(k)}(1) - \hat{f}_n(1)f^{(k)}(1)\right|^2 \\
&\lesssim \left(t\log(1/t)\right)^\alpha + \sum_{k=1}^{\lfloor \alpha \rfloor} (t\log(1/t))^k \cdot n^{-\frac{2(\alpha-k)}{2\alpha+1}} \\
&= (t\log(1/t))^\alpha + (\log(1/t))^\alpha n^{-\frac{2\alpha}{2\alpha+1}}  \sum_{k=1}^{\lfloor \alpha \rfloor} \left(\frac{t}{n^{-\frac{2}{2\alpha+1}}}\right)^k \\
&\asymp t (\log(1/t))^\alpha \left(t^{\alpha - 1} + n^{-\frac{2(\alpha - 1)}{2\alpha+1}}\right) \lesssim t (\log(1/t))^\alpha\cdot n^{-\frac{2(\alpha - 1)}{2\alpha+1}}
\end{align*}
where we have used \(t \lesssim n^{-\frac{2\alpha}{2\alpha+1}}\) to conclude \(\sum_{k=1}^{\lfloor \alpha \rfloor} \left(\frac{t}{n^{-\frac{2}{2\alpha+1}}}\right)^k \asymp \frac{t}{n^{-\frac{2}{2\alpha+1}}}\). We also applied that for all $x\in [-1,1]$,
\begin{equation}
\label{eqn:up-11}
\begin{aligned}
\bE \left|f(x)\hat{f}_n^{(k)}(x) - \hat{f}_n(x)f^{(k)}(x)\right|^2 &\lesssim f(x)^2 \bE|\hat{f}_n^{(k)}(x)-f^{(k)}(x)|^2 +  |f^{(k)}(x)|^2 \bE |\hat{f}_n(x)-f(x)|^2 \\
&\lesssim n^{-\frac{2(\alpha - k)}{2\alpha+1}} + n^{-\frac{2\alpha}{2\alpha+1}} \lesssim n^{-\frac{2(\alpha - k)}{2\alpha+1}}. 
\end{aligned}
\end{equation}
Therefore, we can bound
\[\bE\left(J_1(x,t)\mathbbm{1}_{\Omega}\right) \lesssim |\phi_t(x-1)|^2 \cdot t(\log(1/t))^\alpha  n^{-\frac{2(\alpha - 1)}{2\alpha+1}} \lesssim (\log(1/t))^\alpha  n^{-\frac{2(\alpha - 1)}{2\alpha+1}}\]
since $\phi_t(x-1) \lesssim \frac{1}{\sqrt{t}}$.
\end{proof}

Now we examine the external part \(D_3\) and give a proof of Proposition \ref{prop:up-external}. Recall from Section \ref{section:methodology_low_noise} use the same estimator as in the boundary part.
\begin{proof}[Proof of Proposition \ref{prop:up-external}]
Without loss of generality, we bound the error on the the positive side $D_3^+ := D_3 \cap \bR^+$. As argued in the proof of Proposition \ref{prop:up-boundary}, it suffices to bound the score-matching error on the event \(\Omega\). When $x\ge 1+C\sqrt{t}$, we know that $p(x,t), \phi_t* \hat{f}_n(x) \asymp \bP\{|\mathcal N(x,t)|\le 1\} \lesssim \sqrt{t}\cdot \phi_t(x-1)$ by using Lemma \ref{lemma:up-3}. Then, we still split the numerator into three terms as follows,
\begin{align*}
&\bE\left(\mathbbm{1}_{\Omega} \int_{D_3^+} \left(\hat{s}_{3, n}(x,t)-s(x,t)\right)^2\cdot p(x,t)\,\rd x\right) \lesssim \bE\left(\mathbbm{1}_{\Omega} \int_{D_3^+} \left(\frac{J(x,t)}{\bP\{|\mathcal N(x,t)|\le 1\}^2}\right)^2 p(x,t) \,\rd x \right) \\
&~~~~~\lesssim \bE\left(\mathbbm{1}_{\Omega} \int_{D_3^+} \frac{J_1(x,t)+J_2(x,t)+J_3(x,t)}{\bP\{|\mathcal N(x,t)|\le 1\}^4}\cdot p(x,t) \,\rd x\right). 
\end{align*}
Here, \(J(x, t) = \hat{\psi}(x, t)p(x, t) - \psi(x, t)(\varphi_t * \tilde{f}_n(x))\) and the three terms $J_i(x,t)~(i=1,2,3)$ are defined in (\ref{eqn:J-123}). For the term $J_3(x,t)$, we have
\begin{align*}
\bE\left(J_3(x,t)\mathbbm{1}_{\Omega}\right) &\le \bE\left[\bE_{\mu,\nu\sim \mathcal N(x,t)} \left(f'(\mu) \hat{f}_n(\nu) - \hat{f}_n'(\mu) f(\nu)\right)\cdot \mathbbm{1}_{\{|u|,|v|\le 1\}}\right]^2 \\
&= \bP\{|\mathcal N(x,t)|\le 1\}^4 \cdot \bE \left(\bE_{\mu,\nu\sim q}\left(f'(\mu) \hat{f}_n(\nu) - \hat{f}_n'(\mu) f(\nu)\right)\right)^2 \\
&\le \bP\{|\mathcal N(x,t)|\le 1\}^4 \cdot \bE_{\mu,\nu\sim q} \bE \left(f'(\mu) \hat{f}_n(\nu) - \hat{f}_n'(\mu) f(\nu)\right)^2 \\
&\lesssim \bP\{|\mathcal N(x,t)|\le 1\}^4 \cdot n^{-\frac{2(\alpha-1)}{2\alpha+1}}. 
\end{align*}
Here, $q$ is the conditional density of the random variable \(Y\) conditional on the event \(\{|Y| \leq 1\}\), where \(Y \sim N(x, t)\). Therefore, we conclude that
\begin{align}
\label{eqn:external-J3}
&\int_{D_3^+} \frac{\bE\left(J_3(x,t)\mathbbm{1}_{\Omega}\right)}{\bP\{|\mathcal N(x,t)|\le 1\}^4} \cdot p(x,t) \,\rd x \lesssim n^{-\frac{2(\alpha-1)}{2\alpha+1}}\cdot \int_{D_3^+} \sqrt{t}\phi_t(x-1)\,\rd x  \notag\\
&~~~~~~~~~~~~~~~~~\lesssim \sqrt{t} n^{-\frac{2(\alpha-1)}{2\alpha+1}}\cdot \int_C^{\infty} \phi(y) \,\rd y < \sqrt{t} n^{-\frac{2(\alpha-1)}{2\alpha+1}}. 
\end{align}
Then for $J_2(x,t)$, we can simply bound it as
\begin{align*}
\bE\left(J_2(x,t)\mathbbm{1}_{\Omega}\right) \lesssim |\phi_t(x+1)|^2 \cdot \left|\int_{-1}^1 \phi_t(x-\mu) \,\rd \mu\right|^2 = \bP\{|\mathcal N(x,t)|\le 1\}^2 \cdot |\phi_t(x+1)|^2.
\end{align*}
Therefore, by using Lemma \ref{lemma:up-3},
\begin{align}
\label{eqn:external-J2}
&\int_{D_3^+} \frac{\bE\left(J_2(x,t)\mathbbm{1}_{\Omega}\right)}{\bP\{|\mathcal N(x,t)|\le 1\}^4} \cdot p(x,t) \,\rd x \lesssim \int_{D_3^+} \frac{|\phi_t(x+1)|^2}{\frac{t}{(x-1)^2}\cdot |\phi_t(x-1)|^2} \cdot \frac{\sqrt{t}}{x-1}\phi_t(x-1)\,\rd x \notag\\
&~~~< \int_{D_3^+} \phi_t(x+1)\cdot \frac{x-1}{\sqrt{t}} \,\rd x = \int_{C}^{\infty} \phi(2/\sqrt{t}+y)\cdot y \,\rd y \lesssim \phi(2/\sqrt{t}) \lesssim \exp(-2/t), 
\end{align}
which is exponentially small. Finally, we come to the term $J_1(x,t)$. Notice that
\begin{align*}
&\bE \left|f(1)\int_{-1}^{1}\varphi_t(x-\mu) \hat{f}_n(\mu) \, d\mu - \hat{f}_n(1)\int_{-1}^{1}\varphi_t(x-\mu)f(\mu)\, d\mu \right|^2 \\
&~~\lesssim \left(\int_{-1}^{1} \varphi_t(x-\mu) |\mu-1|^{\alpha} \,\rd\mu \right)^2 + \sum_{k=1}^{\lfloor \alpha \rfloor} \left(\int_{-1}^1 |\mu-1|^k \phi_t(x-\mu)\right)^2 \cdot \bE \left|f(1)\hat{f}_n^{(k)}(1) - \hat{f}_n(1)f^{(k)}(1)\right|^2. 
\end{align*}
For any positive constant $l > 0$, we have
\begin{align*}
&~~~\int_{-1}^1 \varphi_t(x-\mu) |\mu-1|^{l} \,\rd\mu = \bP\{|\mathcal N(x,t)|\le 1\}\cdot \bE_{\mu\sim q} \left(|\mu-1|^l\right) \le \bP\{|\mathcal N(x,t)|\le 1\}\cdot\\
&~~~~\left(\bP_{\mu\sim q}\left\{\mu \ge 1-\sqrt{2C't\log(1/t)}\right\} \cdot (2C't\log(1/t))^{l/2} + \bP_{\mu\sim q}\left\{\mu < 1-\sqrt{2C't\log(1/t)}\right\} \cdot 2^l \right) \\
&\lesssim \bP\{|\mathcal N(x,t)|\le 1\}\cdot \left[t^{C'} + (t\log(1/t))^{l/2}\right].   
\end{align*}
Here, we apply Lemma \ref{lemma:prob-1}. After making $C' > \alpha/2$, we conclude for \(l \leq \alpha\), 
\[\int_{-1}^1 \varphi_t(x-\mu) |\mu-1|^{l} \,\rd\mu \lesssim \bP\{|\mathcal N(x,t)|\le 1\}\cdot (t\log(1/t))^{l/2}. \]
Therefore, 
\begin{align*}
&~~~~\bE \left|f(1)\int_{-1}^{1}\varphi_t(x-\mu) \hat{f}_n(\mu) \, d\mu - \hat{f}_n(1)\int_{-1}^{1}\varphi_t(x-\mu)f(\mu)\, d\mu \right|^2 \\
&\lesssim \bP\{|\mathcal N(x,t)|\le 1\}^2\cdot \left[(t\log(1/t))^{\alpha} + \sum_{k=1}^{\lfloor \alpha\rfloor} (t\log(1/t))^{k}\cdot n^{-\frac{2(\alpha-k)}{2\alpha+1}} \right] \\
&\lesssim \bP\{|\mathcal N(x,t)|\le 1\}^2\cdot \left[t(\log(1/t))^{\alpha}\cdot n^{-\frac{2(\alpha-1)}{2\alpha+1}}\right],  
\end{align*}
where the same techniques are used as in the analysis of the boundary part. For $J_1(x,t)$, it can be bounded as
\begin{align}
\label{eqn:external-J1}
&~~~~\int_{D_3^+} \frac{\bE\left(J_1(x,t)\mathbbm{1}_{\Omega}\right)}{\bP\{|\mathcal N(x,t)|\le 1\}^4} \cdot p(x,t) \,\rd x \notag \\
&\lesssim  t(\log(1/t))^{\alpha} n^{-\frac{2(\alpha-1)}{2\alpha+1}} \cdot \int_{D_3^+} \frac{|\phi_t(x-1)|^2}{\frac{t}{(x-1)^2}\cdot |\phi_t(x-1)|^2} \cdot \frac{\sqrt{t}}{x-1}\phi_t(x-1)\,\rd x \notag\\
&= t(\log(1/t))^{\alpha} n^{-\frac{2(\alpha-1)}{2\alpha+1}} \cdot \int_{D_3^+} \frac{x-1}{\sqrt{t}} \phi_t(x-1) \,\rd x \lesssim  t(\log(1/t))^{\alpha} n^{-\frac{2(\alpha-1)}{2\alpha+1}}. 
\end{align}
After the summation of (\ref{eqn:external-J3}), (\ref{eqn:external-J2}) and (\ref{eqn:external-J1}), we obtain that 
\[\bE\left( \mathbbm{1}_{\Omega}\int_{D_3^+} \left(\hat{s}_{3, n}(x,t)-s(x,t)\right)^2\cdot p(x,t)\,\rd x\right)\lesssim \left(\sqrt{t} + t(\log(1/t))^\alpha \right)\cdot n^{-\frac{2(\alpha-1)}{2\alpha+1}}\lesssim  \sqrt{t}\cdot n^{-\frac{2(\alpha-1)}{2\alpha+1}}.\]
The same analysis goes through for the negative side $D_3^- = D_3 \cap \bR^-$, and so we arrive at our conclusion.
\end{proof}
Finally, we combine Propositions \ref{prop:up-internal}, \ref{prop:up-boundary} and \ref{prop:up-external} to conclude Theorem \ref{thm:up-1} which bounds the error of the score estimator \(\hat{s}_n(x, t) = \hat{s}_{1, n}(x, t) \mathbbm{1}_{\{x \in D_1\}} + \hat{s}_{2, n}(x, t)\mathbbm{1}_{\{x \in D_2\}} + \hat{s}_{3, n}(x, t) \mathbbm{1}_{\{x \in D_3\}}\). 

%% file: upper/data-free.tex
In this section, we work to give a proof of Theorem \ref{thm:up-2}. In the low noise regime ($t< n^{-\frac{2}{2\alpha+1}}$) with low smoothness $\alpha < 1$, the density of the data $f \in \mathcal{F}_\alpha$ may not be differentiable. We frequently used integration by parts when analyzing the \(\alpha \ge 1\) case, which is no longer generically valid when \(\alpha < 1\). Recall the data-free estimator proposed in Section \ref{section:methodology_low_noise}. Denote $u(x)=\frac12 \mathbbm{1}_{\{|x|\le 1\}}$ as the uniform distribution on the interval $U=[-1,1]$. Our score estimator is
\[\hat{s}(x, t) = \frac{\phi_t'* u(x)}{\phi_t* u(x)},\]
which is exactly the score of $\varphi_t * u$. Similar to the previous section, we have
\begin{equation*}
\begin{aligned}
s(x,t)-\hat{s}(x,t) &= \frac{\phi_t'* f(x)}{\phi_t* f(x)}-\frac{\phi_t'* u(x)}{\phi_t* u(x)} = \frac{\left(\phi_t'* f\right)(x)\left(\phi_t* u\right)(x)-\left(\phi_t'* u\right)(x)\left(\phi_t* f\right)(x)}{\left(\phi_t* f\right)(x)\left(\phi_t* u\right)(x)} \\
& := \frac{J(x,t)}{p(x,t) u(x,t)}. 
\end{aligned}
\end{equation*}
Here, $u(\cdot,t) := \phi_t* u$ and the numerator 
\begin{equation}
\label{eqn:I-1}
\begin{aligned}
J(x,t) & := \left(\phi_t'* f\right)(x)\left(\phi_t* u\right)(x)-\left(\phi_t'* u\right)(x)\left(\phi_t* f\right)(x)\\
&= \int_{\bR^2} \left[\phi_t'(x-\mu)f(\mu)\cdot \phi_t(x-\nu)u(\nu) - \phi_t'(x-\mu)u(\mu)\cdot \phi_t(x-\nu)f(\nu)\right] \rd \mu \rd \nu \\
&= \int_{\bR^2} \frac{\mu-x}{t}\cdot \phi_t(x-\mu)\phi_t(x-\nu)\cdot \left(f(\mu)u(\nu)-f(\nu)u(\mu)\right) \rd \mu \rd \nu \\
&= \frac12 \bE_{\mu,\nu\sim\mathcal N(x,t)} \frac{\mu-x}{t}(f(\mu)-f(\nu)) \cdot \mathbbm{1}_{\{|u|, |v| \le 1\}}. 
\end{aligned}
\end{equation}
We have $|f(\mu)-f(\nu)| \le L |\mu-\nu|^\alpha$ when $\mu,\nu\in[-1,1]$ since $f$ is $\alpha$-H\"{o}lder smooth with $\alpha < 1$. Therefore,
\begin{equation}
\label{eqn:I-2}
|J(x,t)|\lesssim \bE_{\mu,\nu\sim \mathcal N(x,t)} \frac{|x-\mu|}{t}\cdot |\mu-\nu|^{\alpha}\cdot \mathbbm{1}_{\{|\mu|, |\nu| \le 1\}}.
\end{equation}
In order to bound the score estimation error, we split $\bR$ into two parts,
\[I = \int_{D} (s(x,t)-\hat{s}(x,t))^2\cdot p(x,t)\,\rd x + \int_{D^c} (s(x,t)-\hat{s}(x,t))^2\cdot p(x,t)\,\rd x =: I_1 + I_2, \]
where the interval $D=[-(1+C\sqrt{t}), 1+C\sqrt{t}]$ denotes the internal and boundary part combined, and $D^c$ denotes the external part. 
\begin{lemma}
\label{lemma:up-8}
If \(\alpha < 1\) and \(t < n^{-\frac{2}{2\alpha+1}}\), then \(I_1 \lesssim t^{\alpha-1}\).
\end{lemma}
\begin{proof}[Proof of Lemma \ref{lemma:up-8}]
According to Lemma \ref{lemma:up-1}, we know that $p(x,t), u(x,t)\ge c$ for some constant $c>0$ when $|x|\le 1+C\sqrt{t}$. We can further bound $|J(x,t)|$ as follows,
\begin{equation*}
\begin{aligned}
|J(x,t)| &\le \bE_{\mu,\nu\sim \mathcal N(x,t)} \frac{|x-\mu|}{t}\cdot |\mu-\nu|^{\alpha} = \bE_{\mu \sim \mathcal N(x,t)}\left[ \frac{|x-\mu|}{t}\cdot \bE_{z\sim \mathcal N(0,1)} |\mu-(x+\sqrt{t} z)|^\alpha \right] \\ 
&\le \bE_{\mu \sim \mathcal N(x,t)}\left[ \frac{|x-\mu|}{t}\cdot \bE_{z\sim \mathcal N(0,1)} \left(|\mu-x|^\alpha + |\sqrt{t} z|^\alpha\right)\right] \\
& = \frac1t \bE_{\mu \sim \mathcal N(x,t)} |x-\mu|^{\alpha+1} + t^{\alpha/2 - 1} \bE_{\mu \sim \mathcal N(x,t)} |x-\mu| \lesssim \frac1t\cdot (\sqrt{t})^{\alpha+1} + t^{\alpha/2 - 1} \cdot \sqrt{t} = 2t^{(\alpha-1)/2}. 
\end{aligned}
\end{equation*}
To sum up, we have $|J(x,t)| \lesssim t^{(\alpha-1)/2}$ and therefore, for $|x| \le 1+C\sqrt{t}$, it holds that
\[(s(x,t)-\hat{s}(x,t))^2 = \frac{|J(x,t)|^2}{|p(x,t)|^2\cdot |u(x,t)|^2} \lesssim \frac{t^{\alpha-1}}{c^4} \lesssim t^{\alpha-1}. \]
Finally, we have
\[I_1 = \int_{D} (s(x,t)-\hat{s}(x,t))^2\cdot p(x,t)\,\rd x \lesssim t^{\alpha-1}\cdot \int_{D} p(x,t) \,\rd x < t^{\alpha-1}, \]
which comes to our conclusion. 
\end{proof}
\begin{lemma}
\label{lemma:up-9}
If \(\alpha < 1\) and \(t < n^{-\frac{2}{2\alpha+1}}\), then \(I_2 \lesssim t^{\alpha-1/2} \log^{\alpha+1}(1/t)\).
\end{lemma}
\begin{proof}[Proof of Lemma \ref{lemma:up-9}]
For any given $x > 1+C\sqrt{t}$, we denote $q$ as the conditional density of the random variable \(Y\) conditional on the event \(\{|Y| \leq 1\}\), where \(Y \sim N(x, t)\). By using Lemma \ref{lemma:up-3}, we know that $p(x,t), u(x,t)\asymp \bP\{|\mathcal N(x,t)|\le 1\}$. We further upper bound $J(x,t)$ in (\ref{eqn:I-2}) as follows,     
\begin{equation*}
\begin{aligned}
&~~~\frac{|J(x,t)|}{p(x,t) u(x,t)} \asymp \frac{\bE_{\mu, \nu\sim\mathcal N(x,t)} \left(\frac{|x-\mu|}{t}\cdot |\mu-\nu|^{\alpha} \cdot \mathbbm{1}_{\{|\mu|, |\nu| \le 1\}}\right)}{\bP\{|\mathcal N(x,t)|\le 1\}^2}\\
&\le \bE_{\mu,\nu\sim\mathcal N(x,t)}\left[\frac{|\mu-x|}{t}\cdot |\mu-\nu|^\alpha \Big| |\mu|, |\nu| \le 1\right] = \bE_{\mu,\nu\sim q} \left[\frac{|\mu-x|}{t}\cdot |\mu-\nu|^\alpha\right]. 
\end{aligned}
\end{equation*}
Denote $\Omega = \{(\mu,\nu) \in [-1, 1]^2 \mid 1-\sqrt{2C't\log(1/t)} \le \mu,\nu \le 1\}$. According to Lemma \ref{lemma:prob-1}, we know that $\bP_{u,v\sim q}(\Omega^c) < 2t^{C'}$. Therefore, we consider the two different regions $\Omega$ and $\Omega^c$, and obtain
\begin{equation*}
\begin{aligned}
\frac{|J(x,t)|}{p(x,t) u(x,t)} &\lesssim \bP(\Omega) \cdot \bE_{\mu,\nu\sim q} \left[\frac{|\mu-x|}{t}\cdot |\mu-\nu|^\alpha \Big| \Omega\right] + \bP(\Omega^c) \cdot \bE_{u,v\sim q} \left[\frac{|\mu-x|}{t}\cdot |\mu-\nu|^\alpha \Big| \Omega^c\right] \\
&< \frac{x-1+\sqrt{2C't\log(1/t)}}{t}\cdot (2C't\log(1/t))^{\alpha/2} + 2t^{C'}\cdot \frac{2^{\alpha}(x+1)}{t}.
\end{aligned}
\end{equation*}
Now we take integral over $x>1+C\sqrt{t}$ to obtain
\begin{equation*}
\begin{aligned}
&~~~\int_{x>1+C\sqrt{t}} (s(x,t)-\hat{s}(x,t))^2\cdot p(x,t)\,\rd x = \int_{x>1+C\sqrt{t}} \left(\frac{|J(x,t)|}{p(x,t) u(x,t)}\right)^2 \cdot p(x,t)\,\rd x \\
&\lesssim \int_{x>1+C\sqrt{t}} \left[\left(\frac{x-1+\sqrt{2C't\log(1/t)}}{t}\right)^2 \cdot (t\log(1/t))^\alpha + t^{2(C'-1)}(x+1)^2\right]\cdot \sqrt{t}\phi_t(x-1) \,\rd x\\
&\lesssim \int_C^{\infty} \left[\left(z+\sqrt{2C'\log(1/t)}\right)^2\cdot t^{\alpha-1}\log^\alpha(1/t) + t^{2(C'-1)}(2+\sqrt{t}z)^2\right] \cdot \sqrt{t} \phi(z) \rd z \\
&\lesssim t^{\alpha-1/2}\log^{\alpha+1}(1/t) + t^{2C'-3/2}. 
\end{aligned}
\end{equation*}
For sufficiently large $C' > \frac{\alpha+1}{2}$, we finally conclude
\[\int_{1+C\sqrt{t}}^{\infty} (s(x,t)-\hat{s}(x,t))^2\cdot p(x,t)\,\rd x \lesssim t^{\alpha-1/2}\log^{\alpha+1}(1/t). \]
Similarly, it also holds for $x<-(1+C\sqrt{t})$,
\[\int_{-\infty}^{-1-C\sqrt{t}} (s(x,t)-\hat{s}(x,t))^2\cdot p(x,t)\,\rd x \lesssim t^{\alpha-1/2}\log^{\alpha+1}(1/t), \]
which comes to our conclusion.
\end{proof}
Finally, we combine Lemma \ref{lemma:up-8} and Lemma \ref{lemma:up-9}, and conclude Theorem \ref{thm:up-2}.

\begin{remark}
\label{remark:alpha}
Theorem \ref{thm:up-2} also holds for $\alpha=1$. Therefore, when the smoothness $\alpha \ge 1$, we have $f\in \mathcal F_{\alpha} \subseteq \mathcal F_1$, and it leads to
\[\sup_{f \in \mathcal{F}_\alpha} \bE \int_\bR (\hat{s}(x,t)-s(x,t))^2 \cdot p(x,t) \,\rd x \le \sup_{f \in \mathcal{F}_1} \bE \int_\bR (\hat{s}(x,t)-s(x,t))^2 \cdot p(x,t) \,\rd x\lesssim 1.\]
Here, $\hat{s}(x,t)$ is the score function of $\Uniform[-1,1] * \mathcal N(0,t)$. 
\end{remark}

%% file: lower/lower_bound_proof.tex
    \subsection{Proof of Theorem \ref{thm:score_lowerbound}}
    \begin{proof}[Proof of Theorem \ref{thm:score_lowerbound}]
        Theorem \ref{thm:score_lowerbound} is proved via Fano's method, which is now a standard technique for establishing minimax lower bounds (e.g. see \cite{tsybakov_introduction_2009}). To use Fano's method, we will first construct a collection of densities lying in \(\mathcal{F}_\alpha\). \newline 

        \noindent \textbf{Construction of the collection:} Define the density \(f_0(\mu) = \frac{1}{2}\mathbbm{1}_{\{|\mu| \le 1\}}\). Fix \(0 < \epsilon \le \rho\) where \(\epsilon, \rho \in (0, 1)\) are to be chosen later. Fix a function \(w : \R \to \R\) such that 
        \begin{itemize}
            \item \(w \in C^\infty(\R)\),
            \item \(w\) is supported on \([-1, 1]\),
            \item \(\int_{-\infty}^{\infty} w(x) \, \rd x = 0\),
            \item \(\max_{0 \le k \le \lfloor \alpha \rfloor} ||w^{(k)}||_\infty^2\vee ||w||_{L^2(\R)}^2 \vee ||w''||_{L^2(\R)}^2 \le C'\) and \(||w'||_{L^2(\R)}^2 \ge c'\) for some universal constants \(C', c' > 0\).
        \end{itemize}
        Define the interval 
        \begin{equation}\label{def:interior_interval}
            I := \left[-1 + \sqrt{C t \log\left(\frac{1}{t}\right)}, 1 - \sqrt{C t \log\left(\frac{1}{t}\right)} \right]
        \end{equation}
        where \(C = C(\alpha, L) > 0\) is a sufficiently large constant depending only on \(\alpha\) and \(L\). Let \(C_D > 0\) be a sufficiently large universal constant. Let \(m\) be an integer such that \(m \asymp \frac{1}{\rho}\) and so that \(m\)-many grid points \(\{x_i\}_{i=1}^{m}\) with spacing \(2\rho\) apart exist in \([-1+\sqrt{Ct \log(1/t)} + C_D\rho, 1 - \sqrt{Ct\log(1/t)} - C_D\rho]\). Here, we will insist on taking \(\rho\) smaller than a sufficiently small constant and taking \(c_2\) sufficiently small so that this interval is not empty. For \(b \in \{0, 1\}^m\), define 
        \begin{equation*}
            f_b(\mu) = f_0(\mu) + \epsilon^\alpha \sum_{i=1}^{m} b_i w\left(\frac{\mu-x_i}{\rho}\right). 
        \end{equation*}
        To ensure the condition \(c_d \le f_b(x) \le C_d\) for \(|x| \le 1\) in the definition (\ref{def:param}) of \(\mathcal{F}_\alpha\), we require the condition \(\epsilon^\alpha \le c^*\) for some sufficiently small \(c^* > 0\). Furthermore, by taking \(c^*\) sufficiently small depending on \(\alpha, L\) and noting that \(\epsilon \le \rho\), it is clear \(f_b|_{[-1, 1]} \in \mathcal{H}_\alpha(L)\). Hence, \(\{f_b\}_{b \in \{0, 1\}^m} \subset \mathcal{F}_\alpha\). \newline
        
        \noindent \textbf{Reduction to submodel:} Having constructed the collection of densities \(\{f_b\}_{b \in \{0, 1\}^m}\), we now show that the score estimation problem in the model \(\mathcal{F}_\alpha\) over \(\R\) is no easier than estimating the score function in the submodel \(\{f_b\}_{b \in \{0, 1\}^m}\) over the interval \(I\). For \(b \in \{0, 1\}^m\), define the convolution \(p_b(x, t) := (\varphi_t * f_b)(x)\) and denote the score function \(s_b(x, t) := \frac{\partial}{\partial x} \log p_b(x, t)\). Since \(c_2\) is sufficiently small and \(t \le c_2\), we have \(p_b(\cdot, t) \gtrsim 1\) on \([-1, 1]\) for all \(b \in \{0, 1\}^m\). With this in hand, consider 
        \begin{align*}
            \inf_{\hat{s}} \sup_{f \in \mathcal{F}_\alpha} \bE \left(\int_{-\infty}^{\infty} \left|\hat{s}(x,t) - s(x, t) \right|^2 \, p(x, t)\,\rd x\right) &\ge \inf_{\hat{s}} \sup_{b \in \{0, 1\}^m} \bE \left(\int_{-\infty}^{\infty} \left|\hat{s}(x, t) - s_b(x, t) \right|^2 \, p_b(x, t)\,\rd x \right) \\
            &\ge \inf_{\hat{s}} \sup_{b \in \{0, 1\}^m} \bE \left(\int_{I} \left|\hat{s}(x,t) - s_b(x, t) \right|^2 \, p_b(x, t)\,\rd x \right) \\
            &\gtrsim \inf_{\hat{s}} \sup_{b \in \{0, 1\}^m} \bE \left(\int_{I} \left|\hat{s}(x,t) - s_b(x, t) \right|^2 \, \rd x\right).
        \end{align*}
        The problem has been reduced to a score estimation problem in the \(L^2(I)\)-norm. \newline 
        
        \noindent \textbf{Score separation:} To apply Fano's method, the pairwise \(L^2(I)\) separation between the score functions \(\left\{s_b(\cdot, t)\right\}_{b \in \{0, 1\}^m}\) needs to be established. Denote \(\psi_b(x, t) := \frac{\partial}{\partial x} p_b(x, t)\), and note chain rule yields \(s_b(x, t) = \frac{\psi_b(x, t)}{p_b(x, t)}\). For \(b, b' \in \{0, 1\}^m\), observe 
        \begin{align*}
            \int_{I} |s_{b}(x, t) - s_{b'}(x, t)|^2 \, \rd x &= \int_{I} \left|\frac{\psi_{b}(x, t)p_{b'}(x, t) - \psi_{b'}(x, t)p_b(x, t)}{p_{b}(x, t)p_{b'}(x, t)}\right|^2 \, \rd x \\
            &\gtrsim \int_{I} \left|\psi_{b}(x, t)p_{b'}(x, t) - \psi_{b'}(x, t)p_b(x, t)\right|^2 \, \rd x
        \end{align*}
        where again we have used \(p_{b}(\cdot, t), p_{b'}(\cdot, t) \gtrsim 1\) on \([-1, 1]\). Consider 
        \begin{align*}
            \left|\psi_{b}p_{b'} - \psi_{b'}p_b\right|^2 &= |\psi_{b}p_{b'} - \psi_{b'}p_{b'} + \psi_{b'}p_{b'} - \psi_{b'}p_{b}|^2 \\
            &= |p_{b'}|^2 |\psi_{b} - \psi_{b'}|^2 + |\psi_{b'}|^2 |p_{b'} - p_{b}|^2 - 2\cdot |p_{b'}||\psi_{b} - \psi_{b'}| \cdot |\psi_{b'}| |p_{b'} - p_b| \\
            &\geq |p_{b'}|^2 |\psi_{b} - \psi_{b'}|^2 - 2\cdot |p_{b'}||\psi_{b} - \psi_{b'}| \cdot |\psi_{b'}| |p_{b'} - p_b| \\ 
            &\geq \left(1 - \frac{1}{\tilde{C}}\right) |p_{b'}|^2 |\psi_{b} - \psi_{b'}|^2 - \tilde{C} |\psi_{b'}|^2 |p_{b'} - p_{b}|^2
        \end{align*}
        for any \(\tilde{C} > 0\). The last line is obtained by the inequality \(2uv \le \frac{u^2}{\tilde{C}} + \tilde{C}v^2\) for any \(\tilde{C} > 0\). Choosing a large, universal constant \(\tilde{C} > 0\) and noting \(p_{b}(\cdot, t) \geq c\) on \([-1, 1]\) for some small universal \(c > 0\), we have 
        \begin{align}
            &\int_{I} \left|\psi_{b}(x, t)p_{b'}(x, t) - \psi_{b'}(x, t)p_b(x, t)\right|^2 \, \rd x \nonumber \\
            &\ge \int_{I} \left(1 - \frac{1}{\tilde{C}}\right) |p_{b'}(x, t)|^2 |\psi_{b}(x, t) - \psi_{b'}(x, t)|^2 - \tilde{C} |\psi_{b'}(x, t)|^2 |p_{b'}(x, t) - p_{b}(x, t)|^2 \, \rd x \nonumber \\
            &\ge c^2\left(1 - \frac{1}{\tilde{C}}\right)\int_I \left|\psi_b(x, t) - \psi_{b'}(x, t)\right|^2 \, \rd x - \tilde{C} \int_{I} \left|\psi_{b'}(x, t)\right|^2 \cdot \left|p_{b'}(x, t) - p_{b}(x, t)\right|^2 \, \rd x. \label{eqn:pointwise_score_separation}
        \end{align}
        To further lower bound (\ref{eqn:pointwise_score_separation}), we now furnish an upper bound for the second term. \newline

        \noindent \textbf{Case 1:} Suppose \(\alpha \ge 1\). Note we can write \(f_b = f_0 + \epsilon^{\alpha} \sum_{i=1}^{m} b_i g_i(\rho^{-1}\cdot)\) where \(g_i(\mu) = w\left(\mu-\rho^{-1} x_i\right)\). Note \(g_i \in C^\infty(\R)\). Consider \(||f_b||_\infty \lesssim 1\), and so 
        \begin{align*}
            \left|\psi_{b'}(x, t)\right| &= \left| (\varphi_t' * f_{b'})(x) \right| \le \left|(\varphi_t' * f_0)(x)\right| + \epsilon^{\alpha} \sum_{i=1}^{m} \left|(\varphi_t' * g_i(\rho^{-1} \cdot )) \right| \\
            &\le \left|(\varphi_t' * f_0)(x)\right| + \epsilon^{\alpha}\rho^{-1} \sum_{i=1}^{m} \left|(\varphi_t * g_i'(\rho^{-1} \cdot )) \right| \le \left|(\varphi_t' * f_0)(x)\right| + C' \epsilon^{\alpha} \rho^{-1} m.
        \end{align*}
        Here, we have used \(||g'||_\infty = ||w'||_\infty \le C'\) to obtain the final inequality. It remains to bound \(|(\varphi_t' * f_0)(x)|\). Note by the fundamental theorem of calculus, we have for \(x \in I\), 
        \begin{align*}
            \left|(\varphi_t' * f_0)(x)\right| = \left|\frac{1}{2} \int_{-1}^{1} \varphi_t'(x-\mu) \, \rd \mu\right| = \frac{\left|\varphi_t(x+1) - \varphi_t(x-1)\right|}{2} \lesssim t^{(C-1)/2}. 
        \end{align*}
        Therefore,
        \begin{align}
            &\int_{I} \left|\psi_{b'}(x, t)\right|^2  \cdot \left|p_{b'}(x, t) - p_{b}(x, t)\right|^2 \lesssim \left(t^{C-1} + \epsilon^{2\alpha} \rho^{-2} m^2\right) \int_{I} |p_{b'}(x, t) - p_{b}(x, t)|^2 \, \rd x \nonumber \\ 
            &= \left(t^{C-1} + \epsilon^{2\alpha} \rho^{-2} m^2\right) \int_{-\infty}^{\infty} |(\varphi_t * (f_{b'} - f_{b}))(x)|^2 \, \rd x \le \left(t^{C-1} + \epsilon^{2\alpha} \rho^{-2} m^2\right) ||f_{b'} - f_b||^2 \nonumber \\
            &\lesssim \left(t^{C-1} + \epsilon^{2\alpha} \rho^{-2} m^2\right) \epsilon^{2\alpha} \rho \cdot d_{\text{Ham}}(b, b') \label{eqn:pointwise_score_separation_lowerorder_bigalpha}.
        \end{align}

        \noindent \textbf{Case 2:} Suppose \(\alpha < 1\). Note we have \(|f_{b'}(x) - f_{b'}(y)| \lesssim |x-y|^\alpha\) for \(x, y \in [-1, 1]\). Therefore, for \(x \in I\) it follows 
        \begin{align*}
            |\psi_{b'}(x, t)| &= |(\varphi_t' * f_{b'})(x)| = \left|\int_{-\infty}^{\infty} \varphi_t'(\mu) f_{b'}(x-\mu) \,\rd \mu\right| = \frac{1}{t} \left|\int_{-\infty}^{\infty} \mu \varphi_t(\mu) f_{b'}(x-\mu) \, \rd \mu\right| \\
            &= \frac{1}{t} \left|\int_{-\infty}^{\infty} \mu\varphi_t(\mu)(f_{b'}(x-\mu) - f_{b'}(x))\,\rd \mu \right| \\
            &\le \frac{1}{t} \int_{|x-\mu| \le 1} |\mu|^{1+\alpha} \varphi_t(\mu) \,\rd \mu + \frac{||f_{b}||_\infty}{t} \int_{|x-\mu| > 1} |\mu| \varphi_t(\mu) \,\rd \mu \\
            &\lesssim t^{\frac{\alpha-1}{2}} + \frac{1}{\sqrt{t}} \sqrt{\bP \left\{|\mathcal{N}(x, t)| > 1\right\}} \lesssim t^{\frac{\alpha-1}{2}}
        \end{align*}
        where we have used that \(x \in I\) implies \( \sqrt{\bP \left\{|\mathcal{N}(x, t)| > 1\right\}} \lesssim t^{\alpha/2}\) since \(C\) can be taken sufficiently large (potentially depending on \(\alpha\)). Therefore, we have 
        \begin{equation}
            \int_{I} \left|\psi_{b'}(x, t)\right|^2  \cdot \left|p_{b'}(x, t) - p_{b}(x, t)\right|^2 \lesssim t^{\alpha-1} \cdot \epsilon^{2\alpha}\rho \cdot d_{\text{Ham}}(b, b') \label{eqn:pointwise_score_separation_lowerorder_smallalpha}
        \end{equation}
        when \(\alpha < 1\). This concludes the analysis for this case. \newline 

        \noindent \textbf{Bounding the score separation:}
        From (\ref{eqn:pointwise_score_separation}), (\ref{eqn:pointwise_score_separation_lowerorder_bigalpha}), and (\ref{eqn:pointwise_score_separation_lowerorder_smallalpha}), it follows 
        \begin{align*}
            \int_{I}|s_b(x, t) - s_{b'}(x, t)|^2 \, \rd x \ge &\tilde{c}_1 \int_{I} |\psi_b(x, t) - \psi_{b'}(x, t)|^2 \, \rd x \\
            &- \tilde{C}_1
            \begin{cases}
                 \left(t^{C-1} + \epsilon^{2\alpha} \rho^{-2} m^2\right) \epsilon^{2\alpha} \rho \cdot d_{\text{Ham}}(b, b') &\textit{if } \alpha \ge 1, \\
                 t^{\alpha-1} \cdot \epsilon^{2\alpha}\rho \cdot d_{\text{Ham}}(b, b') &\textit{if } \alpha < 1,
            \end{cases}
        \end{align*}
        for some universal constants \(\tilde{C}_1, \tilde{c}_1 > 0\). It remains to lower bound \(\int_{I} |\psi_b(x, t) - \psi_{b'}(x, t)|^2 \, \rd x\). By Proposition \ref{prop:derivative_separation}, there exists a sufficiently large universal constant \(C_3 > 0\) so that the choice
        \begin{equation*}
            \rho = C_3\sqrt{t} \vee \epsilon
        \end{equation*}
        yields 
        \begin{equation*}
            \int_I |s_b(x, t) - s_{b'}(x, t)|^2 \, \rd x \ge \epsilon^{2\alpha} d_{\text{Ham}}(b, b') \left(\tilde{c}_2 \rho^{-1} - \tilde{C}_1 \begin{cases}
                \left(t^{C-1} + \epsilon^{2\alpha} \rho^{-2} m^2\right) \rho &\textit{if } \alpha \ge 1, \\
                t^{\alpha-1} \rho &\textit{if } \alpha < 1. 
           \end{cases}\right)
        \end{equation*}
        where \(\tilde{c}_2 > 0\) is some universal constant. Note that \(m \asymp \frac{1}{\rho}\) and \(\rho \asymp \sqrt{t} \vee \epsilon\). Consider the case \(\alpha \ge 1\). Since \(t \le c_2 < 1\), it is clear by taking \(C\) sufficiently large (possibly depending on \(\alpha\)), we have \(\tilde{C}_1 \rho t^{C-1} \le (\tilde{c}_2/4)\rho^{-1}\). It is also clear, since \(\alpha \ge 1\) and \(\epsilon \le c^*\) for \(c^*\) sufficiently small, that \(\tilde{C}_1 \epsilon^{2\alpha} \rho^{-2} m^2 \rho \le (\tilde{c}_2/4) \rho^{-1}\). Therefore, we have shown \(\tilde{c}_2 \rho^{-1} - \tilde{C}_1 \left(t^{C-1} + \epsilon^{2\alpha} \rho^{-2} m^2\right) \rho \ge \frac{\tilde{c}_2}{2}\rho^{-1}\). Now consider the case \(\alpha < 1\). Under the condition 
        \begin{equation}\label{eqn:eps_condition_smallalpha}
            \epsilon \le c^{**}\sqrt{t}
        \end{equation}
        where \(c^{**} > 0\) is a sufficiently small universal constant, we have \(\rho = C_3 \sqrt{t}\). Consequently, \(\tilde{C}_1 t^{\alpha-1}\rho = \tilde{C}_1 C_3^2 t^\alpha \rho^{-1} \le \frac{\tilde{c}_2}{2} \rho^{-1}\) since \(t \le c_2\) and \(c_2\) is sufficiently small. Therefore, it follows \(\tilde{c}_2\rho^{-1} - \tilde{C}_1t^{\alpha-1}\rho \ge \frac{\tilde{c}_2}{2}\rho^{-1}\). To summarize, we have shown 
        \begin{equation}\label{eqn:score_separation}
            \int_I |s_b(x, t) - s_{b'}(x, t)|^2 \, \rd x\ge \frac{\tilde{c}_2}{2} \epsilon^{2\alpha} \rho^{-1} d_{\text{Ham}}(b, b')
        \end{equation}
        for all \(\alpha > 0\), under the condition (\ref{eqn:eps_condition_smallalpha}) if \(\alpha < 1\). The score separation bound (\ref{eqn:score_separation}) is suitable for use in Fano's method, and so we move on to the next ingredient of the lower bound argument.
        \newline

        \noindent \textbf{Information theory:} 
        The next ingredient in Fano's method is a bound on Kullback-Leibler divergences between the data generating distributions \(\left\{f_b^{\otimes n}\right\}_{b \in \{0, 1\}^m}\). In fact, it suffices to bound \(\dKL(f_b^{\otimes n} \,||\, f_0^{\otimes n})\) for each \(b \in \{0, 1\}^m\) in the version of Fano's method given by Corollary 2.6 in \cite{tsybakov_introduction_2009}, rather than bound all pairwise Kullback-Leibler divergences. Consider 
        \begin{align}
            \dKL(f_b^{\otimes n}\,||\,f_0^{\otimes n}) &= n \dKL(f_b\,||\,f_0)\le n\chi^2(f_b\,||\,f_0) = 2n\int_{-1}^{1} (f_b(x) - f_0(x))^2 \, \rd x \nonumber \\
            &= 2n\epsilon^{2\alpha} \sum_{i=1}^{m} \int_{-\infty}^{\infty} w^2\left(\frac{x_i - \mu}{\rho}\right) \,\rd\mu \le 2C'n\epsilon^{2\alpha}m\rho. \label{eqn:KL_radius}
        \end{align}
        With this bound on the Kullback-Leibler divergence in hand, the final remaining ingredient is to construct a packing set. \newline
        
        \noindent \textbf{Packing:} 
        By the Gilbert-Varshamov bound (see Lemma 2.9 in \cite{tsybakov_introduction_2009}), there exists a subset \(\mathcal{B} \subset \{0, 1\}^m\) such that \(\log |\mathcal{B}| \ge c_B m\) and \(\min_{b \neq b' \in \mathcal{B}} d_{\text{Ham}}(b, b') \ge c_B m\) for some universal constant \(c_B > 0\). Thus, from (\ref{eqn:score_separation}), the score separation on this subcollection is 
        \begin{equation}\label{eqn:separation_final_I}
            \int_{I} |s_b(x, t) - s_{b'}(x, t)|^2 \, \rd x \ge \frac{\tilde{c}_2c_B}{2} \epsilon^{2\alpha}\rho^{-1}m 
        \end{equation}
        for \(b \neq b' \in \mathcal{B}\). \newline 
        
        \noindent \textbf{Applying Fano's method:} From (\ref{eqn:KL_radius}), (\ref{eqn:separation_final_I}), and \(\log|\mathcal{B}| \ge c_B m\), Fano's method (see Corollary 2.6 in \cite{tsybakov_introduction_2009}) yields 
        \begin{align*}
            \inf_{\hat{s}} \sup_{f \in \mathcal{F}_\alpha} \bE \left(\int_{-\infty}^{\infty} \left|\hat{s}(x,t) - s(x, t) \right|^2 \, p(x, t)\,\rd x\right) &\ge \frac{\tilde{c}_2c_B}{2} \epsilon^{2\alpha}\rho^{-1}m \left(1 - \frac{2C' n\epsilon^{2\alpha}m\rho + \log 2}{m}\right). 
        \end{align*}
        We are now in position to choose \(\epsilon\) subject to the constraint \(\epsilon \le c^*\), and the additional constraint (\ref{eqn:eps_condition_smallalpha}) if \(\alpha < 1\). Note that the constraint \(\epsilon \le \rho\) is already satisfied by our choice of \(\rho\). Let us select 
        \begin{equation*}
            \epsilon = c_e \left( (n\sqrt{t})^{-\frac{1}{2\alpha}} \wedge 
            \begin{cases}
                n^{-\frac{1}{2\alpha+1}} &\textit{if } \alpha \ge 1, \\
                \sqrt{t} &\textit{if } \alpha < 1,     
            \end{cases}
            \right)
        \end{equation*}
        where \(c_e > 0\) is a sufficiently small constant (potentially depending on \(\alpha, L\)). Note since \(c_e\) is sufficiently small that condition (\ref{eqn:eps_condition_smallalpha}) is satisfied when \(\alpha < 1\). For any \(\alpha > 0\), we now claim \(2C'n\epsilon^{2\alpha} \rho \le \frac{1}{3}\). To show this, consider we must show \(2C'n\epsilon^{2\alpha+1} \le \frac{1}{3}\) and \(2C'n\epsilon^{2\alpha} \cdot C_3 \sqrt{t} \le \frac{1}{3}\) both hold since \(\rho = C_3\sqrt{t} \vee \epsilon\). Observe 
        \begin{equation*}
            2C'n\epsilon^{2\alpha+1} \le 2C' c_e^{2\alpha+1} \left( (nt^{\alpha+1/2})^{-\frac{1}{2\alpha}} \wedge 
            \begin{cases} 
                1 &\textit{if } \alpha \ge 1, \\
                nt^{\alpha+1/2} &\textit{if } \alpha < 1,
            \end{cases}
            \right).
        \end{equation*}
        If \(\alpha \ge 1\), it is clear \(2C'n\epsilon^{2\alpha+1} \le \frac{1}{3}\) since \(c_e\) is sufficiently small. If \(\alpha < 1\), then we have \(2C'n\epsilon^{2\alpha+1} \le 2C' c_e^{2\alpha + 1} \left( (nt^{\alpha+1/2})^{-\frac{1}{2\alpha}} \wedge nt^{\alpha+1/2}\right) \le 2C' c_e^{2\alpha + 1} \le \frac{1}{3}\) since \(c_e\) is sufficiently small. Here, we have used that \(a^{-\frac{1}{b}} \wedge a \leq 1\) for all \(a, b \geq 0\). Hence, we have shown \(2C'n\epsilon^{2\alpha+1}\le \frac{1}{3}\) for all \(\alpha > 0\). It is straightforward to see that \(2C'n\epsilon^{2\alpha} \cdot C_3 \sqrt{t} \le \frac{1}{3}\) holds since \(2C'n\epsilon^{2\alpha} \cdot C_3 \sqrt{t} \le \frac{1}{3} \le 2C'C_3 c_e^{2\alpha} n\sqrt{t} ((n\sqrt{t})^{-\frac{1}{2\alpha}})^{2\alpha} = 2C'C_3c_e^{2\alpha+1} \le \frac{1}{3}\) because \(c_e\) is sufficiently small. Therefore, we have proved \(2C'n\epsilon^{2\alpha}\rho \le \frac{1}{3}\) for all \(\alpha > 0\). By taking \(c_2\) sufficiently small, we have \(\frac{\log 2}{m} \le \frac{1}{3}\). All of these choices yield the bound 
        \begin{equation*}
            \inf_{\hat{s}} \sup_{f \in \mathcal{F}_{\alpha}} \bE \left(\int_{-\infty}^{\infty} \left|\hat{s}(x,t) - s(x, t) \right|^2 \, p(x, t)\,\rd x\right) \ge \frac{\tilde{c}_2c_B}{3} \epsilon^{2\alpha}\rho^{-1}m = c_1(\alpha, L)\left(\frac{1}{nt^{3/2}} \wedge \left(n^{-\frac{2(\alpha-1)}{2\alpha+1}} + t^{\alpha-1}\right)\right)
        \end{equation*}
        as desired. The proof is complete.
    \end{proof}

    \begin{proof}[Proof of Proposition \ref{prop:derivative_separation}]
        For \(b, b' \in \{0, 1\}^m\), define
        \begin{equation}\label{def:Gamma_I}
            \Gamma_{b,b'}(t) = \int_{-\infty}^{\infty} |\psi_{b}(x, t) - \psi_{b'}(x, t)|^2 \, \rd x. 
        \end{equation}
        Note 
        \begin{equation}\label{eqn:separation_prep_I}
            \int_{I} |\psi_b(x, t) - \psi_{b'}(x, t)|^2 \, \rd x = \Gamma_{b, b'}(t) - \int_{I^c} |\psi_b(x, t) - \psi_{b'}(x, t)|^2 \, \rd x. 
        \end{equation}
        To lower bound the separation, it suffices to lower bound \(\Gamma_{b,b'}(t)\) and upper bound the second integral above. First, let's examine \(\Gamma_{b, b'}\). It is clear \(\Gamma_{b,b'}\) is differentiable in \(t\), and so by Taylor's theorem we have
        \begin{equation}\label{eqn:Gamma_taylor_I}
            \Gamma_{b,b'}(t) = \Gamma_{b,b'}(0) + \Gamma_{b,b'}'(\xi)t
        \end{equation}
        for some \(\xi \in (0, t)\). It is well-known that for any function \(h : \R \to \R\), the convolution \((\varphi_t * h)(x)\) satisfies the heat equation \(\frac{\partial}{\partial t} (\varphi_t * h)(x) = \frac{1}{2}(\varphi_t'' * h)(x)\). The following well-known identity is immediate
        \begin{equation*}
            \frac{\rd}{\rd t} \int_{-\infty}^{\infty} |(\varphi_t * h)(x)|^2 \,\rd x = \left.((\varphi_t*h) \cdot (\varphi_t'*h))(x)\right|_{-\infty}^{\infty} - \int_{-\infty}^{\infty} \left|(\varphi_t' * h)(x)\right|^2 \, \rd x.
        \end{equation*}
        If \(h\) is compactly supported and differentiable everywhere, it immediately follows 
        \begin{equation}\label{eqn:heat_equation_norm}
            \frac{\rd}{\rd t} \int_{-\infty}^{\infty} |(\varphi_t * h)(x)|^2 \,\rd x = - \int_{-\infty}^{\infty} \left|(\varphi_t' * h)(x)\right|^2 \, \rd x = -\int_{-\infty}^{\infty} \left|(\varphi_t * h')(x)\right|^2 \, \rd x.
        \end{equation}
        Furthermore, observe we can deduce \(||\varphi_t*h||_{L^2(\R)} \le ||h||_{L^2(\R)}\) since the time derivative (\ref{eqn:heat_equation_norm}) nonpositive. Let us now apply (\ref{eqn:heat_equation_norm}) to \(\Gamma_{b,b'}\). First, note we can write \(f_b = f_0 + \epsilon^{\alpha} \sum_{i=1}^{m} b_i g_i(\rho^{-1}\cdot)\) where \(g_i(\mu) = w\left(\mu-\rho^{-1} x_i\right)\). Note \(g_i \in C^\infty(\R)\). Then we have 
        \begin{equation*}
            |\psi_{b}(x, t) - \psi_{b'}(x, t)| = \epsilon^{\alpha}\rho^{-1}\left|\left(\varphi_t * \sum_{i=1}^{m} (b_i - b_i') g_i'(\rho^{-1}\cdot)\right)(x)\right|. 
        \end{equation*}
        Therefore, 
        \begin{equation*}
            \Gamma_{b,b'}'(t) = -\epsilon^{2\alpha}\rho^{-4}\int_{-\infty}^{\infty} \left|\left(\varphi_t * \sum_{i=1}^{m} (b_i - b_i') g_i''(\rho^{-1}\cdot)\right)(x)\right|^2 \, \rd x.
        \end{equation*}
        Since the collection \(\{g_i\}_{i=1}^{m}\) have disjoint support, it is immediate for any \(\xi > 0\), 
        \begin{align*}
            \Gamma_{b,b'}(0) &= \epsilon^{2\alpha}\rho^{-1} d_{\text{Ham}}(b, b') \int_{-\infty}^{\infty} \left|w'(y)\right|^2 \, \rd y, \\
            \Gamma_{b,b'}'(\xi) &\ge -\epsilon^{2\alpha} \rho^{-3} d_{\text{Ham}}(b,b')\int_{-\infty}^{\infty}\left|w''\left(y\right)\right|^2 \, \rd y. 
        \end{align*}
        The last inequality follows from the fact \(||\varphi_\xi * h||_{L^2(\R)} \le ||h||_{L^2(\R)}\). Therefore, from (\ref{eqn:Gamma_taylor_I}) we have 
        \begin{align}
            \Gamma_{b,b'}(t) &\ge \epsilon^{2\alpha} \rho^{-1} d_{\text{Ham}}(b, b')\left(\int_{-\infty}^{\infty} |w'(y)|^2 \, \rd y - t\rho^{-2} \int_{-\infty}^{\infty} |w''(y)|^2 \, \rd y \right) \nonumber \\
            &\ge \epsilon^{2\alpha}\rho^{-1} d_{\text{Ham}}(b, b') \left(c' - C' t\rho^{-2}\right). \label{eqn:Gamma_signal_I}
        \end{align}

        Now that we have handled \(\Gamma_{b,b'}(t)\), let us turn to the second term in (\ref{eqn:separation_prep_I}). Consider
        \begin{align*}
            \int_{I^c} |\psi_b(x, t) - \psi_{b'}(x, t)|^2 \, \rd x &= \epsilon^{2\alpha} \rho^{-2} \int_{I^c} \left|\left(\varphi_t * \sum_{i=1}^{m}(b_i - b_i')g_i'(\rho^{-1}\cdot)\right)(x)\right|^2 \, \rd x. 
        \end{align*}        
        To evaluate the integral, first fix a point \(x \in I^c \cap [0, \infty)\). Consider since \(w\) is supported on \([-1, 1]\) and since \(x_i \in [-1+\sqrt{Ct \log(1/t)} + C_D\rho, 1-\sqrt{Ct\log(1/t)} - C_D\rho]\), we have by Jensen's inequality
        \begin{align*}
            &\left|\left(\varphi_t *\sum_{i=1}^{m} (b_i - b_i')g_i'(\rho^{-1} \cdot) \right)(x)\right|^2 = \left|\int_{-\infty}^{\infty} \varphi_t(x-\mu) \left(\sum_{i=1}^{m} (b_i - b_i') w'\left(\frac{\mu-x_i}{\rho}\right)\right) \, \rd \mu \right|^2 \\
            &\le \int_{-\infty}^{\infty} \varphi_t(x-\mu) \left|\sum_{i=1}^{m} (b_i - b_i') w'\left(\frac{\mu-x_i}{\rho}\right)\right|^2 \,\rd \mu = \sum_{i=1}^{m} (b_i - b_i')^2 \int_{-\infty}^{\infty} \varphi_t(x-\mu) \left|w'\left(\frac{\mu-x_i}{\rho}\right)\right|^2\,\rd \mu \\
            &\le C' \sum_{i=1}^{m} (b_i - b_i')^2\cdot \bP \left\{x_i-\rho \le \mathcal{N}(x, t) \le x_i+\rho\right\}\\ &\le C' d_{\text{Ham}}(b, b') \bP \left\{\mathcal{N}(0, 1) \ge \frac{x-1+\sqrt{Ct\log(1/t)}+(C_D-1)\rho}{\sqrt{t}}\right\}. 
        \end{align*}
        By a similar argument, for \(x \in I^c \cap (-\infty, 0]\) we have 
        \begin{equation*}
            \left|\left(\varphi_t *\sum_{i=1}^{m} (b_i - b_i')g_i'(\rho^{-1} \cdot) \right)(x)\right|^2 \le C' d_{\text{Ham}}(b, b')  \bP \left\{\mathcal{N}(0, 1) \ge \frac{-x-1+\sqrt{Ct\log(1/t)}+(C_D-1)\rho}{\sqrt{t}}\right\}.
        \end{equation*}
        With this bound in hand, observe 
        \begin{align*}
            &\int_{I^c}\left|\left(\varphi_t *\sum_{i=1}^{m} (b_i - b_i')g_i'(\rho^{-1} \cdot) \right)(x)\right|^2\, \rd x \\
            &\le 2C' d_{\text{Ham}}(b, b')\int_{1-\sqrt{Ct\log(1/t)}}^{\infty}  \bP \left\{\mathcal{N}(0, 1) \ge \frac{x-1+\sqrt{Ct\log(1/t)}+(C_D-1)\rho}{\sqrt{t}}\right\}\,\rd x. 
        \end{align*} 
        To summarize, we have thus shown 
        \begin{align}
            &\int_{I^c} |\psi_b(x, t) - \psi_{b'}(x, t)|^2 \, \rd x \nonumber \\
            &\le 2C' \epsilon^{2\alpha}\rho^{-2} d_{\text{Ham}}(b, b')\int_{1-\sqrt{Ct\log(1/t)}}^{\infty}  \bP \left\{\mathcal{N}(0, 1) \ge \frac{x-1+\sqrt{Ct\log(1/t)}+(C_D-1)\rho}{\sqrt{t}}\right\}\,\rd x. \label{eqn:psi_err_I} 
        \end{align}
        Plugging (\ref{eqn:Gamma_signal_I}) and (\ref{eqn:psi_err_I}) into (\ref{eqn:separation_prep_I}) yields the following lower bound on the separation 
        \begin{align}
            &\int_{I} |\psi_b(x, t) - \psi_{b'}(x, t)|^2 \, \rd x \nonumber \\
            &\ge d_{\text{Ham}}(b, b') \epsilon^{2\alpha}\rho^{-1} \nonumber \\
            &\;\; \cdot \left(c' - C' t\rho^{-2} - 2C'\rho^{-1}\int_{1-\sqrt{Ct\log(1/t)}}^{\infty}  \bP \left\{\mathcal{N}(0, 1) \ge \frac{x-1+\sqrt{Ct\log(1/t)}+(C_D-1)\rho}{\sqrt{t}}\right\}\,\rd x\right).\label{eqn:separation_prep2_I}
        \end{align}
        Consider that 
        \begin{align*}
            &2C'\rho^{-1} \int_{1-\sqrt{Ct\log(1/t)}}^{\infty} \bP \left\{\mathcal{N}(0, 1) \ge \frac{x-1+\sqrt{Ct\log(1/t)}+(C_D-1)\rho}{\sqrt{t}}\right\}\,\rd x \\
            &= 2C'\rho^{-1} \int_{\frac{(C_D - 1)\rho}{\sqrt{t}}}^{\infty} \frac{\sqrt{t}}{u} \cdot \frac{1}{\sqrt{2\pi}}\exp\left(-\frac{u^2}{2}\right) \,\rd u \le \frac{2C'}{C_D-1}t\rho^{-2}. 
        \end{align*}
        Recall that the following choice is made
        \begin{equation}\label{eqn:rho_choice_I}
            \rho = C_3\sqrt{t} \vee \epsilon.
        \end{equation}
        Let us take \(C_3 > 0\) sufficiently large so that \(C't\rho^{-2} \le \frac{c'}{3}\) and \(\frac{2C'}{C_D-1}t\rho^{-2} \le \frac{c'}{3}\) both hold. With this choice, it follows from (\ref{eqn:separation_prep2_I}) that the separation is 
        \begin{equation*}
            \int_{I} |\psi_b(x, t) - \psi_{b'}(x, t)|^2 \, \rd x \ge \frac{c'}{3} d_{\text{Ham}}(b, b') \epsilon^{2\alpha}\rho^{-1}
        \end{equation*}
        as desired. 
         
    \end{proof}

%% file: lower/lower_bound_large_t.tex
\subsection{Proof of Theorem \ref{thm:score_lowerbound_large_t}}\label{section:large_t_lowerbound_proof}
\begin{proof}[Proof of Theorem \ref{thm:score_lowerbound_large_t}]
    A two-point construction will be used. Define the density \(f_0(\mu) = \frac{1}{2}\mathbbm{1}_{\{|\mu| \le 1\}}\). Fix \(0 < \epsilon \le \rho\) where \(\epsilon, \rho\) are to be chosen later. Define \(w(x) = x\), and define 
    \begin{equation*}
        f_1(\mu) = f_0(\mu) + \epsilon^\alpha w\left(\frac{\mu}{\rho}\right) \mathbbm{1}_{\{|\mu| \le 1\}}. 
    \end{equation*}
    To ensure \(f_1 \in \mathcal{F}_\alpha\), we require the condition \(\frac{\epsilon}{\rho} \le c^*\) for some sufficiently small \(c^* > 0\) depending on \(\alpha, L\). Denote the interval \(I := \left[-1-\sqrt{t}, 1+\sqrt{t}\right]\). \newline 
    
    \noindent \textbf{Two-point reduction:} For \(b \in \{0, 1\}\), define the convolution \(p_b(x, t) := (\varphi_t * f_b)(x)\) and denote the score function \(s_b(x, t) := \frac{\partial}{\partial x} \log p_b(x, t)\). Observe we have \(p_b(\cdot, t) \gtrsim \frac{1}{\sqrt{t}} \wedge 1\) on \(I\) for \(b \in \{0, 1\}\). With this in hand, consider 
    \begin{align}
        &\inf_{\hat{s}} \sup_{f \in \mathcal{F}_\alpha} \bE \left(\int_{-\infty}^{\infty} \left|\hat{s}(x,t) - s(x, t) \right|^2 \, p(x, t)\,\rd x\right) \nonumber \\
        &\ge \inf_{\hat{s}} \sup_{b \in \{0, 1\}} \bE \left(\int_{-\infty}^{\infty} \left|\hat{s}(x, t) - s_b(x, t) \right|^2 \, p_b(x, t)\,\rd x \right) \nonumber \\
        &\ge \inf_{\hat{s}} \sup_{b \in \{0, 1\}} \bE \left(\int_{I} \left|\hat{s}(x,t) - s_b(x, t) \right|^2 \, p_b(x, t)\,\rd x \right) \nonumber \\
        &\gtrsim \inf_{\hat{s}} \sup_{b \in \{0, 1\}} \bE \left(\left(\frac{1}{\sqrt{t}} \wedge 1\right) \int_{I} \left|\hat{s}(x,t) - s_b(x, t) \right|^2 \, \rd x\right) \nonumber \\
        &\gtrsim \left(\frac{1}{\sqrt{t}} \wedge 1\right) \left(\int_{I}|s_0(x, t) - s_1(x, t)|^2 \, \rd x\right) e^{-n \dKL(f_0\,||\, f_1)}, \label{eqn:twopoint_reduction}
    \end{align}
    where the final line follows from a standard two-point lower bound argument (e.g. see \cite{tsybakov_introduction_2009}). The problem has been reduced to a score estimation problem in the \(L^2(I)\)-norm. \newline 
    
    \noindent \textbf{Score separation:} 
    We now compute the \(L^2(I)\) separation between the scores \(s_0(\cdot, t)\) and \(s_1(\cdot, t)\). Denote \(\psi_b(x, t) := \frac{\partial}{\partial x} p_b(x, t)\), and note chain rule yields \(s_b(x, t) = \frac{\psi_b(x, t)}{p_b(x, t)}\). Observe
    \begin{align}
        \left(\frac{1}{\sqrt{t}} \wedge 1\right)\int_{I} |s_{1}(x, t) - s_{0}(x, t)|^2 \, \rd x &= \left(\frac{1}{\sqrt{t}} \wedge 1\right) \int_{I} \left|\frac{\psi_{1}(x, t)p_{0}(x, t) - \psi_{0}(x, t)p_1(x, t)}{p_{0}(x, t)p_{1}(x, t)}\right|^2 \, \rd x \nonumber \\
        &\gtrsim (t^{3/2} \vee 1) \int_{I} \left|\psi_{1}(x, t)p_{0}(x, t) - \psi_{0}(x, t)p_1(x, t)\right|^2 \, \rd x, \label{eqn:score_sep_I}
    \end{align}
    where again we have used \(p_{0}(\cdot, t), p_{1}(\cdot, t) \gtrsim \frac{1}{\sqrt{t}} \wedge 1\) on \(I\). For ease of notation, let \(g(\cdot) = \epsilon^\alpha w\left(\frac{\cdot}{\rho}\right)\mathbbm{1}_{\{|\cdot| \le 1\}}\). Now, consider \(\psi_1(x, t) = (\varphi_t' * f_1)(x) = \psi_0(x, t) + (\varphi_t' * g)(x)\) and \(p_1(x, t) = (\varphi_t * f_1)(x) = p_0(x, t) + (\varphi_t * g)(x)\). Therefore,
    \begin{equation*}
        \left|\psi_{1}(x, t)p_{0}(x, t) - \psi_{0}(x, t)p_1(x, t)\right|^2 = \left| (\varphi_t' * g)(x)p_0(x, t) - \psi_0(x, t)(\varphi_t * g)(x)\right|^2.
    \end{equation*}
    Now, consider \((\varphi_t' * g)(x) = \int_{-1}^{1} -\frac{x-\mu}{t} \varphi_t(x-\mu) g(\mu) \, \rd \mu = -\frac{x}{t} (\varphi_t * g)(x) + \frac{1}{t}\int_{-1}^{1}\mu \varphi_t(x-\mu) g(\mu) \, \rd \mu\). Likewise, \(\psi_0(x, t) = -\frac{x}{t} p_0(x, t) + \frac{1}{t} \int_{-1}^{1} \mu \varphi_t(x-\mu) f_0(\mu) \, \rd \mu\). Therefore, it follows 
    \begin{align}
        &\left|\psi_{1}(x, t)p_{0}(x, t) - \psi_{0}(x, t)p_1(x, t)\right|^2 \nonumber \\
        &= \left| (\varphi_t' * g)(x)p_0(x, t) - \psi_0(x, t)(\varphi_t * g)(x)\right|^2 \nonumber \\
        &= \left| \left(\frac{1}{t} \int_{-1}^{1} \mu \varphi_t(x-\mu) g(\mu) \, \rd \mu \right)p_0(x, t) - \left(\frac{1}{t}\int_{-1}^{1} \mu\varphi_t(x-\mu) f_0(\mu) \, \rd \mu \right)(\varphi_t * g)(x)\right|^2 \nonumber \\
        &= \frac{\epsilon^{2\alpha}}{4t^2\rho^2} \left| \left(\int_{-1}^{1} \mu^2 \varphi_t(x-\mu)\, \rd \mu \right)\left(\int_{-1}^{1} \varphi_t(x-\mu) \, \rd \mu \right) - \left(\int_{-1}^{1} \mu \varphi_t(x-\mu) \, \rd \mu\right)^2 \right|^2 \label{eqn:score_sep_II}
    \end{align}
    where we have used \(f_0(\mu) = \frac{1}{2}\) for \(|\mu| \le 1\) and \(w(x) = x\) to obtain the last line. To further bound (\ref{eqn:score_sep_II}) from below, we use Lemma \ref{lemma:g_sep}. To summarize, it follows from (\ref{eqn:score_sep_I}), (\ref{eqn:score_sep_II}), and (\ref{eqn:g_sep}) that the score separation satisfies
    \begin{equation}\label{eqn:score_sep_twopoint}
        \left(\frac{1}{\sqrt{t}} \wedge 1\right) \int_{I} |s_1(x, t) - s_0(x, t)|^2 \gtrsim (t^{3/2} \vee 1) \int_I \frac{\epsilon^{2\alpha}}{t^4\rho^2} e^{-\frac{2(2 + \sqrt{t})^2}{t}}\, dx \geq \frac{\epsilon^{2\alpha}}{(t^2 \vee 1)\rho^2}e^{-\frac{2(2 + \sqrt{t})^2}{t}}.  
    \end{equation}
    
    \noindent \textbf{Information theory:} 
    The next ingredient is to bound the Kullback-Leibler divergence \(\dKL(f_0 \,||\, f_1)\). Consider 
    \begin{equation}
        n\dKL(f_0\,||\,f_1) \le n\chi^2(f_0\,||\,f_1)
        \lesssim n\int_{-1}^{1} (f_1(x) - f_0(x))^2 \, \rd x 
        = n \epsilon^{2\alpha} \int_{-1}^{1} w^2\left(\frac{\mu}{\rho} \right)\, \rd \mu  \asymp \frac{n\epsilon^{2\alpha}}{\rho^2}. \label{eqn:twopoint_KL}
    \end{equation}
    Here, we we have used \(w(x) = x\). With this bound on the Kullback-Leibler divergence in hand, we can turn to (\ref{eqn:twopoint_reduction}). \newline
    
    \noindent \textbf{Obtaining a minimax lower bound:} From (\ref{eqn:twopoint_reduction}), (\ref{eqn:score_sep_twopoint}), and (\ref{eqn:twopoint_KL}), it follows 
    \begin{align*}
        \inf_{\hat{s}} \sup_{f \in \mathcal{F}_\alpha} \bE \left(\int_{-\infty}^{\infty} \left|\hat{s}(x,t) - s(x, t) \right|^2 \, p(x, t)\,\rd x\right) &\gtrsim \frac{\epsilon^{2\alpha}}{(t^2 \vee 1)\rho^2} e^{-\frac{2(2+\sqrt{t})^2}{t}}\exp\left(-\frac{Cn\epsilon^{2\alpha}}{\rho^2} \right)
    \end{align*}
    where \(C > 0\) is some universal constant. We are now in position to choose \(\epsilon, \rho\) subject to the constraint \(\frac{\epsilon}{\rho} \le c^*\). Select \(\rho = 1\) and \(\epsilon = c^* n^{-\frac{1}{2\alpha}}\), which yields the lower bound 
    \begin{equation*}
    \inf_{\hat{s}} \sup_{f \in \mathcal{F}_\alpha} \bE \left(\int_{-\infty}^{\infty} \left|\hat{s}(x,t) - s(x, t) \right|^2 \, p(x, t)\,\rd x\right) \geq \frac{c_1}{nt^2}
    \end{equation*}
    where \(c_1\) is chosen appropriately depending on \(c^*, C,\) and \(c_2\). Here, we have used \(t \geq c_2\) and the fact \(t \mapsto e^{-\frac{2(2+\sqrt{t})^2}{t}}\) is an increasing function. We have also used \((t^2 \vee 1) \leq t^2 \cdot \frac{c_2^2 \vee 1}{c_2^2}\) since \(t \geq c_2\). The proof is complete. 
\end{proof}

\begin{proof}[Proof of Lemma \ref{lemma:g_sep}]
    We use a probabilistic argument. Let \(U, U' \overset{iid}{\sim} \Uniform([-1, 1])\). Note 
    \begin{align*}
        \int_{-1}^{1} \mu^2 \varphi_t(x-\mu) \, \rd \mu &= 2 \bE\left(U^2 \varphi_t(x-U) \right)\\
        \int_{-1}^{1} \varphi_t(x-\mu) \, \rd \mu &= 2 \bE\left(\varphi_t(x-U)\right), \\
        \int_{-1}^{1} \mu \varphi_t(x-\mu) \, \rd \mu &= 2\bE\left(U\varphi_t(x-U)\right). 
    \end{align*}
    Consider that if \((X, Y)\) and \((X', Y')\) are i.i.d. copies of a pair of random variables, then \(\bE(X^2)\bE(Y^2) - (\bE(XY))^2 = \frac{1}{2}\bE((XY' - X'Y)^2)\) by direct calculation. With this in mind, let us take \(X = U \sqrt{\varphi_t(x-U)}, Y = \sqrt{\varphi_t(x-U)}, X' = U' \sqrt{\varphi_t(x-U')}\), and \(Y' = \sqrt{\varphi_t(x-U')}\). Then it follows that with probability one, we have
    \begin{align*}
        &\left(\int_{-1}^{1}\mu^2 \varphi_t(x-\mu) \, \rd \mu\right)\left(\int_{-1}^{1} \varphi_t(x-\mu) \, \rd \mu \right) - \left(\int_{-1}^{1} \mu \varphi_t(x-\mu) \, \rd \mu\right)^2 \\
        &= 4\bE(X^2)\bE(Y^2) - 4\left(\bE\left(XY\right)\right)^2 = 2\bE((XY' - X'Y)^2) = 2\bE(\varphi_t(x-U)\varphi_t(x-U')(U - U')^2).
    \end{align*}
    Observe that \(\varphi_t(x-U) \geq \frac{1}{\sqrt{2\pi t}} e^{-\frac{(|x| + 1)^2}{2t}} \gtrsim \frac{1}{\sqrt{t}} e^{-\frac{(2+\sqrt{t})^2}{2t}}\) since \(x \in I\). Clearly, the same bound holds for \(U'\) as well. Furthermore, consider \(\bE((U - U')^2) = 2 \cdot \Var(U) = \frac{2}{3}\). Therefore, for every \(x \in I\) we have 
    \begin{equation*}
        \left|\left(\int_{1}^{1}\mu^2 \varphi_t(x-\mu) \, \rd \mu\right)\left(\int_{-1}^{1} \varphi_t(x-\mu) \, \rd \mu \right) - \left(\int_{-1}^{1} \mu \varphi_t(x-\mu) \, \rd \mu\right)^2\right|^2 \gtrsim \frac{1}{t^2} e^{-\frac{2(2 + \sqrt{t})^2}{t}} 
    \end{equation*}
    as desired.

\end{proof}

%% file: density_estimation/density_estimation_proof.tex
\section{Lemmas for density estimation}
\label{app:DE}
\begin{lemma}
\label{lem:KLappro}
Let \( f \) be a probability density function supported on \([-1,1]\), and let \( \phi_T \) be the probability density function of the Gaussian distribution \(\mathcal{N}(0,T) \). Then, \(\dKL(f * \phi_T \,\|\, \phi_T) \le \frac{1}{2T}\).
\end{lemma}

\begin{proof}
The proof relies on two key facts. First, for Gaussian distributions with the same variance, the Kullback-Leibler divergence is given by \(\dKL(N(\mu_1,\sigma^2) \,\|\, N(\mu_2,\sigma^2)) = \frac{(\mu_1 - \mu_2)^2}{2\sigma^2}\). This result follows from a straightforward calculation. Second, the Kullback-Leibler divergence \( \dKL(p \,\|\, q) \) is convex in the pair \((p, q)\) \cite{cover2006elements}. Given these facts, consider the convolution \( f * \phi_T \), which can be expressed as \(f * \phi_T(y) = \mathbb{E}_{X \sim f} \left[\phi_T(y - X)\right]\), where \( y \) is a point in \(\mathbb{R}\). Using the convexity of the Kullback-Leibler divergence, we obtain
\[
\dKL(f * \phi_T \,\|\, \phi_T) \le \mathbb{E}_{X \sim f} \left[\dKL(\phi_T(\cdot - X) \,\|\, \phi_T(\cdot))\right].
\]
By applying the first fact, this becomes \(\dKL(f * \phi_T \,\|\, \phi_T) \le \mathbb{E}_{X \sim f} \left[\frac{X^2}{2T}\right]\). Finally, since \( f \) is supported on \([-1,1]\), it follows that \(\mathbb{E}_{X \sim f} \left[\frac{X^2}{2T}\right] \le \frac{1}{2T}\). This concludes the proof.
\end{proof}

Girsanov’s Theorem is commonly used to measure the difference between two stochastic processes, where we can control the difference of the drift terms \cite{oko2023diffusion,zhang_minimax_2024,song2021maximum,chen_sampling_2023}. 

\begin{lemma}[Girsanov’s theorem \cite{karatzas2014brownian,chen_sampling_2023,zhang_minimax_2024,oko2023diffusion}]
\label{lemma:Girsanov}
Let $\vv X = \{X_t\}_{t \in [0,T]}$ and $\vv Y = \{Y_t\}_{t \in [0,T]}$ be the solutions to the SDEs,
\bb
&\diff X_t = a(X_t,t)\diff t + \sqrt{2} \diff W_t,~~~~ X_0 \sim p(\cdot,0),\\
&\diff Y_t = b(Y_t,t)\diff t + \sqrt{2} \diff W_t,~~~~ Y_0 \sim p(\cdot,0).
\ee
Let $p(\cdot,t)$ and $q(\cdot, t)$ denote the probability density functions of $X_t$ and $Y_t$, respectively, and let $\mathbb{P}$ and $\mathbb{Q}$ denote the path measures of $\mathbf{X}$ and $\mathbf{Y}$, respectively. It holds that,
\bbb
\label{eq:Gir}
\dKL(\mathbb{P}||\mathbb{Q}) \le \frac{1}{2}\int_{0}^T \int_{\R^d}  p(x,t) \|a(x,t)-b(x,t)\|^2  \, \diff x \, \diff t.
\eee
Moreover, if Novikov's condition,
\bb
\E_{\mathbb{P}} \left[\exp\left(\int_{0}^T \|a(X_t,t)-b(X_t,t)\|^2 \, \diff t\right)\right] < \infty,
\ee
is satisfied, then it follows that
\bb
\dKL(\mathbb{P}||\mathbb{Q}) = \frac{1}{2}\int_{0}^T \int_{\R^d}  p(x,t) \|a(x,t)-b(x,t)\|^2  \,\diff x \, \diff t.
\ee
\end{lemma}

\begin{lemma}
\label{lem:cost}
For two atomless random variables $X$ and $Y$, say with densities $p$ and $q$ respectively, it holds that
\begin{equation}
\label{eq:de1}
\mathrm{W_1}(X,Y) \le \int_\R |x-a| |p(x)-q(x)| \,\diff x,
\end{equation}
for any $a \in \R$.
\end{lemma}
\begin{proof}
Inequality (\ref{eq:de1}) can be understood by the following intuition. For $x \in \R$, if \(p(x) > q(x)\), then we can transport mass $p(x)-q(x)$ from $x$ to $a$. If $p(x) < q(x)$, we transport $q(x)-p(x)$ from $a$ to $x$. The transport distance is $|x-a|$. This transport successfully transports $X$ to $Y$, with a cost no more than $\int_\R |x-a| |p(x)-q(x)| \diff x$.

To formalize this intuition, consider the classical coupling of $X$ and $Y$, which guarantees that $\mathbb{P}(X \neq Y) = \dTV(X,Y)$. Define $\epsilon = \dTV(X,Y)$, and let $B_p = \{x : p(x) > q(x)\}$ and $B_q = \{x : q(x) > p(x)\}$ \cite{villani2009optimal}.
With probability $1-\epsilon$, let $X=Y$ with density $(p(x) \wedge q(x))/(1-\epsilon)$. With probability $\epsilon$, let the joint density of $(X,Y)$ be $f(x,y) = \epsilon^{-1}(p(x)-q(x))(q(y)-p(y))\mathbbm{1}_{\{x \in B_p\}}\mathbbm{1}_{\{y \in B_q\}}$. The transport cost of this coupling is
\begin{align*}
\int_{\R^2} |x-y| f(x,y) \diff x \diff y &= \frac{1}{\epsilon} \int_{x \in B_p, y \in B_q} |x-y| (p(x)-q(x))(q(y)-p(y)) \diff x \diff y \\
&\le \frac{1}{\epsilon} \int_{x \in B_p, y \in B_q} (|x-a| + |y-a|) (p(x)-q(x))(q(y)-p(y)) \diff x \diff y \\
&= \int_{B_p} |x-a| (p(x)-q(x)) \diff x + \int_{B_q} |y-a| (q(y)-p(y)) \diff y \\
&= \int_\R |x-a| |p(x)-q(x)| \diff x.
\end{align*}
\end{proof}

\begin{lemma}
\label{lem:KL_chisq}
For two random vectors $X$ and $Y$ with densities $p$ and $q$ respectively, it holds that
\begin{equation}
\dKL(p \| q) \gtrsim \chi^2 \left(p \left|\left| \frac{p+q}{2} \right.\right.\right).
\end{equation}
\end{lemma}
This lemma follows from a more general form presented in Corollary 2 of \cite{nishiyama2020relations},
\begin{equation}
\dKL(p \| q) \ge \frac{1-s}{s^2} \log\left(\frac{1}{1-s}\right) \chi^2 \left(p \| (1-s)p + sq \right),
\end{equation}
for $s \in (0,1)$, where we obtain the lemma by plugging in $s=1/2$.

\begin{lemma}
\label{lem:Trans_DIST}
For two random variables $X$ and $Y$, if $\mathrm{spt}(Y) \subset [-C,C]$ where $C$ is a constant, it holds that
\begin{equation}
\E[Y^2] \lesssim \E[X^2] + \dKL(X \| Y).
\end{equation}
\end{lemma}

\begin{proof}
It is easy to consider random variables which have densities. Therefore, we perturb $X$ and $Y$ a bit; let $X_\epsilon = X + \sqrt{\epsilon} Z$ and $Y_\epsilon = Y + \sqrt{\epsilon} Z'$, where $(Z,Z') \sim N(0,I_2)$ is independent of $(X,Y)$. Let $p$ denote the density of $X_\epsilon$ and $q$ denote the density of $Y_\epsilon$. By the  DPI, we know $\dKL(X_\epsilon \| Y_\epsilon) \le \dKL(X \| Y)$. Therefore, it holds that
\[\dKL(X \| Y) \ge \dKL(X_\epsilon \| Y_\epsilon)\gtrsim \chi^2 \left(p \left|\left| \frac{p+q}{2} \right.\right. \right) \gtrsim \int_\R \frac{(p(x)-q(x))^2}{p(x)+q(x)} \diff x,\]
where we use Lemma~\ref{lem:KL_chisq}. We know that $(x-1)^2/(x+1) \ge x/2$ when $x \ge 5$, which implies
\begin{align*}
\int_{\R} \frac{(p(x)-q(x))^2}{p(x)+q(x)} \,\diff x &= \int_\R p(x) \frac{(1-q(x)/p(x))^2}{1+q(x)/p(x)} \,\diff x \ge \int_{q/p \ge 5} p(x) \frac{(1-q(x)/p(x))^2}{1+q(x)/p(x)} \,\diff x \\
&\ge \int_{q/p \ge 5} p(x) \frac{q(x)}{2p(x)} \,\diff x
\gtrsim \int_{q/p \ge 5} q(x) \,\diff x.
\end{align*}
Hence, we have shown 
\begin{equation}\label{eqn:large_qp}
    \int_{q/p \ge 5} q(x) \diff x \lesssim \dKL(X \| Y).
\end{equation}
This relation will be useful later on. Let us now turn to examining the second moment. 
\begin{align*}
\E[Y_\epsilon^2] &= \int_\R y^2 q(y) \diff y \\
&= \underbrace{\int_{|y| \ge 1+C} y^2 q(y) \diff y}_{\displaystyle \mathfrak{A}} + \underbrace{\int_{|y| < 1+C, q/p < 5} y^2 q(y) \diff y}_{\displaystyle \mathfrak{B}} + \underbrace{\int_{|y| < 1+C, q/p \ge 5} y^2 q(y) \diff y}_{\displaystyle \mathfrak{C}}.
\end{align*}

First, consider $\mathfrak{A} = \E \left [Y_\epsilon \mathbbm{1}_{\{|Y_\epsilon| \ge 1+C\}} \right]$, which goes to zero as \(\epsilon \to 0\) by the dominated convergence theorem ($\lim_{\epsilon \to 0}Y_\epsilon \mathbbm{1}_{\{|Y_\epsilon| \ge 1+C\}} = 0$ and $|Y_\epsilon \mathbbm{1}_{\{|Y_\epsilon| \ge 1+C\}}| \le C + \sqrt{\epsilon} |Z'|$). Next, consider $\mathfrak{B} \lesssim \int_{\R} y^2 p(y) \diff y = \E[X_\epsilon^2]$, since in the region of $\mathfrak{B}$ we have $q/p < 5$. Finally, consider $\mathfrak{C} \le (C+1)^2 \int_{q/p \ge 5} q(x) \diff x \lesssim \dKL(X \| Y)$ where we have used (\ref{eqn:large_qp}) and that $C$ is a constant. Taking $\epsilon \to 0$ finishes the proof.
\end{proof}

\begin{lemma}
\label{lem:W1_bound}
For two random vectors $X$ and $Y$ supported on $\mathbb{R}^2$, if the marginal distributions of $X_1$ and $Y_1$ are the same, it holds that
\begin{equation}
\mathrm{W_1}(X_2, Y_2)^2 \lesssim \left( \E\left[(X_1 - X_2)^2\right] \vee \E\left[(Y_1 - Y_2)^2\right] \right) \dKL(X \| Y).
\end{equation}
\end{lemma}

\begin{proof}
To simplify the argument, we define $X_\epsilon = X + \sqrt{\epsilon} Z$ and $Y_\epsilon = Y + \sqrt{\epsilon} Z'$, where $Z \sim N(0,_2)$, $Z' \sim N(0,_2)$, and ($(X,Y)$, $Z$,  $Z'$) are mutually independent. Denote the probability density function of $X_\epsilon$ by $p$, with marginal densities $p_1$ and $p_2$. Likewise, denote the probability density function of $Y_\epsilon$ by $q$, with marginal densities $q_1$ and $q_2$, where $p_1 = q_1$. It holds that
\begin{align*}
\mathrm{W_1}((X_\epsilon)_2,(Y_\epsilon)_2) &= \mathrm{W_1}\left(\int p(x_1,x_2) \diff x_1, \int q(x_1,x_2) \diff x_1\right) \\
&= \mathrm{W_1}\left(\int p(\cdot|x_1) p_1(x_1) \diff x_1, \int q(\cdot|x_1) p_1(x_1) \diff x_1\right) \\
&\le \int \mathrm{W_1}(p(\cdot|x_1), q(\cdot|x_1)) p_1(x_1) \diff x_1,
\end{align*}
where we use the convexity of $\mathrm{W_1}$. By Lemma~\ref{lem:cost}, we have $\mathrm{W_1}(p(\cdot|x_1), q(\cdot|x_1)) \le \int_\mathbb{R} |x_1 - x_2| |p(x_2|x_1) - q(x_2|x_1)| \diff x_2$, which, together with the above inequality, implies
\begin{align*}
\mathrm{W_1}((X_\epsilon)_2, (Y_\epsilon)_2) &\le \int_{\mathbb{R}^2} |x_1 - x_2| |p(x_1, x_2) - q(x_1, x_2)| \diff x_1 \diff x_2 \\
&= \int_{\mathbb{R}^2} |x_1 - x_2| \frac{|p(x_1, x_2) - q(x_1, x_2)|}{p(x_1, x_2) + q(x_1, x_2)} (p(x_1, x_2) + q(x_1, x_2)) \diff x_1 \diff x_2 \\
&\lesssim \sqrt{\int_{\mathbb{R}^2} |x_1 - x_2|^2 (p(x_1, x_2) + q(x_1, x_2)) \diff x_1 \diff x_2} \cdot \sqrt{\chi^2 \left(p \left|\left| \frac{p+q}{2}\right.\right.\right)},
\end{align*}
which is implied by the Cauchy-Schwarz inequality. We know that the first term is bounded by $\E\left((X_\epsilon)_1 - (X_\epsilon)_2)^2\right) \vee \E\left((Y_\epsilon)_1 - (Y_\epsilon)_2)^2\right)$, and the second term is bounded by $\dKL(p \| q)$ up to a constant by Lemma~\ref{lem:KL_chisq}. Applying the DPI, we know that $\dKL(p \| q) = \dKL(X_\epsilon \| Y_\epsilon) \le \dKL(X \| Y)$. Taking $\epsilon \to 0$ finishes the proof.
\end{proof}

\begin{lemma}
    Given the estimator \(\widehat{Y}^t_T = Y^t_T \mathbbm{1}_{\left\{| Y^t_T| \le 1\right\}}\), where \(\vv{Y}^t\) is defined in Section~\ref{sec:PFWASS}, for \(t_1 < t_2\) we have
    \begin{equation}
    \label{eq:T}
    \mathrm{W_1}(\widehat{Y}^{t_1}_T, \widehat{Y}^{t_2}_T) \lesssim \sqrt{t_2 + \dKL(\vv{Y}||\vv{Y}^{t_2})} \cdot \sqrt{\dKL(\vv{Y}^{t_1}||\vv{Y}^{t_2})}.
    \end{equation}
\end{lemma}
\begin{proof}
We want to bound $\mathrm{W_1}(\widehat{Y}^{t_1}_T, \widehat{Y}^{t_2}_T)$ for $t_1<t_2$. Firstly, we know that $\mathrm{Law}(Y^{t_1}_{T-t_2})=\mathrm{Law}(Y^{t_2}_{T-t_2})=\mathrm{Law}(X_{t_2})$ by construction, where $\vv{X}$ solves (\ref{eq:forward}). We shall find that the transport distance between $\widehat{Y}^{t_1}_T$ ($\widehat{Y}^{t_2}_T$) and $Y^{t_1}_{T-t_2}$ ($Y^{t_2}_{T-t_2}$) is small. For technical simplicity, let us also truncate $Y^{t_1}_{T-t_2}$ and $Y^{t_2}_{T-t_2}$. Let $G_1 = Y^{t_1}_{T-t_2}\mathbbm{1}_{\left\{|Y^{t_1}_{T-t_2}| \le 2\right\}}$ and $G_2 = Y^{t_2}_{T-t_2}\mathbbm{1}_{\left\{|Y^{t_2}_{T-t_2}| \le 2\right\}}$.  Firstly, let's show that $G_2$ and $\widehat{Y}^{t_2}_T$ are close. Let $\vv{Y}$ denote the true reverse process solving (\ref{eq:backward}). We have $Y_{T-t}=X_t$ and $Y_{T-t} - Y_T \sim \mathcal{N}(0,t)$. By Lemma~\ref{lemma:Girsanov}, we have 
\[
\dKL(\vv{Y}||\vv{Y}^{t_2}) \lesssim \int_{0}^{t_2} \int_{\mathbb{R}} \E(|s(x, t) - \widehat{s}(x, t)|^2) \, p(x, t) \, \diff x \, \diff t.
\]
And by the DPI, we know
\[
\dKL(Y_{T-t_2}\mathbbm{1}_{\left\{|Y_{T-t_2}| \le 2\right\}} - Y_T||Y^{t_2}_{T-t_2}\mathbbm{1}_{\left\{|Y^{t_2}_{T-t_2}| \le 2\right\}} - \widehat{Y}^{t_2}_T) \le \dKL(\vv{Y}||\vv{Y}^{t_2}),
\]
noting that $Y_T$ is supported on $[-1,1]$, so its truncation is itself. Let $U = Y_{T-t_2}\mathbbm{1}_{\left\{|Y_{T-t_2}| \le 2\right\}} - Y_T$ and $V = Y^{t_2}_{T-t_2}\mathbbm{1}_{\left\{|Y^{t_2}_{T-t_2}| \le 2\right\}} - \widehat{Y}^{t_2}_T = G_2 - \widehat{Y}^{t_2}_T$. We know that $U$ is supported on $[-3,3]$. Therefore, by Lemma~\ref{lem:Trans_DIST}, we have
\begin{equation}
\label{eq:TD}
\E \left( V^2 \right) \lesssim \E \left( U^2 \right) + \dKL(\vv{Y}||\vv{Y}^{t_2}).
\end{equation}
We also know that $|Y_T| \le 1$, implying
\[
|U| = |Y_{T-t_2}\mathbbm{1}_{\left\{|Y_{T-t_2}| \le 2\right\}} - Y_T| \le |Y_{T-t_2} - Y_T|.
\]
Then, we know that $\E\left(U^2\right) \le \E\left( (Y_{T-t_2} - Y_T)^2 \right) = t_2$. Combining it with (\ref{eq:TD}), we get
\[
\E \left( V^2 \right) \lesssim t_2 + \dKL(\vv{Y}||\vv{Y}^{t_2}).
\]
Similarly, we can show
\[
\E \left( (G_1 - \widehat{Y}^{t_1}_T)^2 \right) \lesssim t_2 + \dKL(\vv{Y}||\vv{Y}^{t_2}).
\]
Now we have two pairs of random variables, $(G_1, \widehat{Y}^{t_1}_T)$ and $(G_2, \widehat{Y}^{t_2}_T)$. We know that $\mathrm{Law}(G_1) = \mathrm{Law}(G_2)$, and we also know that $G_1$ ($G_2$) is close to $Y^{t_1}_T$ ($Y^{t_2}_T$). Therefore, we can apply Lemma~\ref{lem:W1_bound} to get the result.
\end{proof}

\begin{lemma}
\label{lem:W1}
    For \(\alpha \ge 1\), we have
    \begin{align*}
    \E(\mathrm{W_1}(Y_T^{t_i}, Y_T^{t_{i+1}})) &\lesssim \sqrt{t_{i+1}} \cdot \sqrt{\int_{t_i}^{t_{i+1}} \int_{\mathbb{R}} \E(|s(x, t) - \widehat{s}(x, t)|^2) \, p(x,t) \, \diff x \, \diff t} \lesssim \mathrm{E}_i 
    \end{align*}
    where \(\mathrm{E}_i = n^{-\frac{\alpha-1}{2\alpha+1}} t_{i+1}\) when \(i \in [0, \lfloor \log_2(nt_*) \rfloor - 1]\) and \(\mathrm{E}_i = t_{i+1}^{\frac{1}{4}} n^{-\frac{1}{2}}\) when \(i \in [\lfloor \log_2(nt_*) \rfloor, N]\), and \(t_* = n^{-\frac{2}{2\alpha+1}}\).
\end{lemma}
\begin{proof}
When $\alpha \ge 1$, by Theorem~\ref{thm:score_upperbound}, we know that 
\[
\int_{\mathbb{R}} \mathbb{E}(|s(x, t) - \widehat{s}(x, t)|^2) \, p(x,t) \, \mathrm{d}x \lesssim \frac{1}{nt^2} \wedge \frac{1}{n t^{3/2}} \wedge n^{-\frac{2(\alpha-1)}{2\alpha+1}} \lesssim 1.
\]
We can apply the above inequality to (\ref{eq:T}), getting
\bb
\E(\mathrm{W_1}(Y_T^{t_1}, Y_T^{t_2})) &\lesssim \sqrt{t_2 + \E \left(\dKL(\vv{Y}||\vv{Y}^{t_2}) \right)} \cdot \sqrt{\E\left(\dKL(\vv{Y}^{t_1}||\vv{Y}^{t_2})\right)} \\
&\lesssim \sqrt{t_2} \cdot \sqrt{\int_{t_1}^{t_2} \int_{\mathbb{R}} \E(|s(x, t) - \widehat{s}(x, t)|^2) \, p(x,t) \, \diff x \, \diff t},
\ee
where we use $\mathbb{E} \left(\dKL(\mathbf{Y} \| \mathbf{Y}^{t_2}) \right) = \int_0^{t_1} \int_{\mathbb{R}} \mathbb{E}(|s(x, t) - \widehat{s}(x, t)|^2) \, p(x,t) \, \mathrm{d}x \, \mathrm{d}t \lesssim t_2$ and the term \(\sqrt{t_2}\) arises from the transport distance (both $Y^{t_1}_T$ and $Y^{t_2}_T$ are close to $X_{t_2}$) and the remaining term illustrates the transport mass. We would like to obtain a bound for times \(t_0\) and \(t_{N+1}\), where $0 = t_0 < t_1 < \cdots < t_N \le t_{N+1} = 1$, and $t_i = \frac{2^i}{n}$ with $i = 1,\cdots,\lfloor \log_2(n) \rfloor=N$. It holds that
\bb
&\E(\mathrm{W_1}(Y^0_T, \widehat{Y}^T_T)) - \E(\mathrm{W_1}(\widehat{Y}^1_T, \widehat{Y}^T_T)) \le \sum_{i=0}^{N} \E(\mathrm{W_1}(\widehat{Y}^{t_i}_T, \widehat{Y}^{t_{i+1}}_T)) \\
&~~~~~~~~\lesssim \sum_{i=0}^{N} \sqrt{t_{i+1}}\cdot \sqrt{\int_{t_i}^{t_{i+1}} \int_{\mathbb{R}} \E(|s(x,t) - \widehat{s}(x, t)|^2) \, p(x, t) \, \diff x \, \diff t} \\
&~~~~~~~~= \left(\sum_{i=0}^{\lfloor\log_2(nt_*)\rfloor-1} + \sum_{i=\lfloor\log_2(nt_*)\rfloor}^{\log_2(n)-1}\right)\sqrt{t_{i+1}} \cdot \sqrt{\int_{t_i}^{t_{i+1}} \int_{\mathbb{R}} \E(|s(x,t) - \widehat{s}(x, t)|^2) \, p(x, t) \, \diff x \, \diff t} \\
&~~~~~~~~=: \mathrm{A} + \mathrm{B}.
\ee
For $i \le \lfloor\log_2(nt_*)\rfloor -1$ and \(t \in [t_i, t_{i+1}]\), we have $\int_{\mathbb{R}} \E(|s(x, t) - \widehat{s}(x, t)|^2) \, p(x, t) \, \diff x \le n^{-2(\alpha-1)/(2\alpha+1)}$, which implies
\bb
\mathrm{A} \lesssim \sum_{i=0}^{\lfloor\log_2(n t_{*})\rfloor-1} \sqrt{t_{i+1}} \cdot \sqrt{(t_{i+1} - t_i) n^{-\frac{2(\alpha-1)}{2\alpha+1}}} 
&\le n^{-\frac{\alpha-1}{2\alpha+1}}\sum_{i=0}^{\lfloor\log_2(nt_{*})\rfloor-1} t_{i+1} \lesssim n^{-\frac{\alpha-1}{2\alpha+1}} t_{\lfloor\log_2(nt_{*})\rfloor} \\
&\le n^{-\frac{\alpha-1}{2\alpha+1}} t_* = n^{-\frac{\alpha+1}{2\alpha+1}} \le n^{-\frac{1}{2}}.
\ee
For $i \ge \lfloor\log_2(nt_*)\rfloor$ and \(t \in [t_{i}, t_{i+1}]\), we have $\int_{\mathbb{R}} \E(|s(x, t) - \widehat{s}(x, t)|^2) \, p(x, t) \, \diff x \lesssim n^{-1}t^{-3/2}$, which implies 
\[
\mathrm{B} \lesssim \sum_{i=\lfloor\log_2(nt_*)\rfloor}^{N} \sqrt{t_{i+1}}\cdot \sqrt{(t_{i+1}-t_{i})n^{-1}t_{i+1}^{-3/2}} \le \sum_{i=\lfloor\log_2(nt_*)\rfloor}^{N} \frac{t_{i+1}^{\frac{1}{4}}}{\sqrt{n}} \lesssim \frac{1}{\sqrt{n}},
\]
yielding $\mathrm{A} + \mathrm{B} \lesssim n^{-1/2}$. Finally, we know that
\[
\E(\mathrm{W_1}(Y^0_T, \widehat{Y}^T_T))  \lesssim \frac{1}{\sqrt{n}}.
\]
To summarize, we have shown
\[
\E(\mathrm{W_1}(Y^0_T, \widehat{Y}^T_T))  \lesssim \frac{1}{\sqrt{n}},
\]
where $Y^0_T = X_0$ and $\widehat{Y}^T_T = \widehat{X}_0$. We have finished the proof of Theorem~\ref{thm:Wasserstein}.

To summarize, we have shown
\bb
\E(\mathrm{W_1}(X_0, \widehat{X}_0)) &\le \E(\mathrm{W_1}(X_0, \widetilde{Y}_T \mathbbm{1}_{|\widetilde{Y}_T|\le 1})) + \E(\mathrm{W_1}(\widehat{Y}_T \mathbbm{1}_{|\widehat{Y}_T|\le 1}, \widetilde{Y}_T \mathbbm{1}_{|\widetilde{Y}_T|\le 1})) \\
&\lesssim \E(\dTV(\widetilde{Y}_T, \widehat{Y}_T)) + \E(\mathrm{W_1}(Y^0_T, Y^T_T)) \lesssim n^{-\frac{1}{2}}.
\ee
Hence, we can achieve the faster rate \(n^{-1/2}\) when considering estimation in Wasserstein distance.
\end{proof}

%% file: multivariate.tex
\section{Extension to the multivariate case}\label{appendix:multivariate}
    In this section, we give a rough sketch on how to generalize to the multivariate case with dimension \(d \geq 1\). The H\"{o}lder class we consider is 
    \begin{align*}
        \begin{split}
        \mathcal{H}_\alpha(L) = &\left\{ f : [-1, 1]^d \to [0, \infty) : \int_{[-1, 1]^d} f(x) \, \rd x = 1, f \text{ is continuous}, \right. \\
        &\;\;\;\left. f \text{ admits all } \lfloor \alpha \rfloor \text{ partial derivatives on } (-1, 1)^d\text{,} \right. \\
        &\;\;\; \left. \text{and } \max_{S \in [d]^{\lfloor \alpha \rfloor}} \left|\partial_{S} f(x) - \partial_S f(y)\right| \le L ||x-y||^{\alpha - \lfloor \alpha \rfloor} \text{ for all } x, y \in (-1, 1)^d \right\}.
        \end{split}
    \end{align*}
    Here, the notation \(\partial_S f = \partial_{s_1s_2...s_{\lfloor \alpha \rfloor}}f\) denotes the partial derivatives with respect to order of the indices \(S = (s_1,...,s_{\lfloor \alpha \rfloor})\). Again, we focus on the case \(L\) is a fixed universal constant and thus drop it from notation. The parameter space is 
    \begin{equation*}
        \mathcal{F}_\alpha := \left\{ f : \R^d \to [0, \infty) : \supp(f) \subset [-1, 1]^d, f|_{[-1, 1]}\in \mathcal{H}_\alpha, \text{ and } c_d \le f(x) \le C_d \text{ for } ||x||_\infty \le 1 \right\}    
    \end{equation*}
    where \(c_d, C_d > 0\) are some universal constants. Throughout, we will consider \(d\) to be some fixed value, and thus will freely absorb into universal constants and the \(\lesssim\) notation. Now, \(\varphi_t\) denotes the probability density function of the multivariate Gaussian \(\mathcal{N}(0, tI_d)\).

    \subsection{Lower bound}
    Theorem \ref{thm:score_lowerbound_d} states a minimax lower bound in the multivariate setting. Conceptually, the argument is exactly the same as the \(d = 1\) case, and the only differences are the standard, technical ones of the flavor found in the typical minimax lower bounds for multivariate density estimation. Thus, we only provide a proof sketch and point out notable aspects. 

    \begin{theorem}\label{thm:score_lowerbound_d}
        If \(\alpha > 0\), then there exists positive constants \(c_1 = c_1(\alpha, L, d)\) and \(c_2 = c_2(\alpha, L, d)\) depending only on \(\alpha, L, d\) such that 
        \begin{equation*}
            \inf_{\hat{s}} \sup_{f \in \mathcal{F}_\alpha} \bE \left(\int_{\bR^d} \left|\left|\hat{s}(x, t) - s(x, t)\right|\right|^2 \, p(x, t) \, \rd x \right) \ge c_1\left(\frac{1}{nt^{d/2 + 1}} \wedge \left(n^{-\frac{2(\alpha-1)}{2\alpha+d}} + t^{\alpha-1}\right)\right)
        \end{equation*}
        for \(t \le c_2\). 
    \end{theorem}
    \begin{proof}[Proof sketch]
        Theorem \ref{thm:score_lowerbound_d} can be proved via Fano's method as used in the proof of Theorem \ref{thm:score_lowerbound}. The construction of the collection of densities lying in \(\mathcal{F}_\alpha\) is done in a similar way.  \newline 

        \noindent \textbf{Construction of the collection:} Define the density \(f_0(\mu) = \frac{1}{2^d}\mathbbm{1}_{\{||\mu||_\infty \le 1\}}\). Fix \(0 < \epsilon \le \rho\) where \(\epsilon, \rho \in (0, 1)\) are to be chosen later. Fix any standard kernel function \(w : \R^d \to \R\) such that \(w \in C^\infty(\R^d)\), \(w\) is supported on \([-1, 1]^d\), \(\int_{\R^d} w(x)\rd x = 0\), has uniformly bounded \(\lfloor \alpha \rfloor\) derivatives, and \(\int_{\R^d} ||\nabla w(y)||^2 \, dy \gtrsim 1\). These conditions are the multivariate versions of the conditions stipulated in the proof of Theorem \ref{thm:score_lowerbound}. Define the subcube 
        \begin{equation}\label{def:interior_interval_d}
            I := \left[-1 + \sqrt{C t \log\left(1/t\right)}, 1 - \sqrt{C t \log\left(1/t\right)} \right]^d
        \end{equation}
        where \(C = C(\alpha, L) > 0\) is a sufficiently large constant depending only on \(\alpha\) and \(L\). Let \(C_D > 0\) be a sufficiently large universal constant. Let \(m\) be an integer such that \(m \asymp \frac{1}{\rho^d}\) and so that \(m\)-many regular lattice points \(\{x_i\}_{i=1}^{m}\) with spacing \(2\rho\) apart exist in \([-1+\sqrt{Ct \log(1/t)} + C_D\rho, 1 - \sqrt{Ct\log(1/t)} - C_D\rho]^d\). Here, we will insist on taking \(\rho\) smaller than a sufficiently small constant and taking \(c_2\) sufficiently small so that this interval is not empty. For \(b \in \{0, 1\}^m\), define 
        \begin{equation*}
            f_b(\mu) = f_0(\mu) + \epsilon^\alpha \sum_{i=1}^{m} b_i w\left(\frac{\mu-x_i}{\rho}\right). 
        \end{equation*}
        It can be checked that under the condition \(\epsilon\) is smaller than a sufficiently small universal constant, we have \(\{f_b\}_{b \in \{0, 1\}^m} \subset \mathcal{F}_\alpha\). \newline
        
        \noindent \textbf{Reduction to submodel:} Having constructed the collection of densities \(\{f_b\}_{b \in \{0, 1\}^m}\), we show the score matching problem can be reduced to estimating the score over the region \(I\) in the previously specified submodel. For \(b \in \{0, 1\}^m\), define the convolution \(p_b(x, t) := (\varphi_t * f_b)(x)\) and denote the score function \(s_b(x, t) := \nabla_x \log p_b(x, t)\). Since \(c_2\) is sufficiently small and \(t \le c_2\), we have \(p_b(\cdot, t) \gtrsim 1\) on \([-1, 1]^d\) for all \(b \in \{0, 1\}^m\). It can be shown 
        \begin{align*}
            \inf_{\hat{s}} \sup_{f \in \mathcal{F}_\alpha} \bE \left(\int_{\R^d} \left|\left|\hat{s}(x,t) - s(x, t) \right|\right|^2 \, p(x, t)\,\rd x\right) \gtrsim \inf_{\hat{s}} \sup_{b \in \{0, 1\}^m} \bE \left(\int_{I} \left|\left|\hat{s}(x,t) - s_b(x, t) \right|\right|^2 \, \rd x\right)
        \end{align*}
        via an argument similar to that in the proof of Theorem \ref{thm:score_lowerbound}. The problem has been reduced to a score estimation problem in the \(L^2(I)\)-norm. \newline 
        
        \noindent \textbf{Score separation:} To apply Fano's method, the pairwise \(L^2(I)\) separation between the score functions \(\left\{s_b(\cdot, t)\right\}_{b \in \{0, 1\}^m}\) needs to be established. Denote \(\psi_b(x, t) := \nabla_x p_b(x, t)\), and note chain rule yields \(s_b(x, t) = \frac{\psi_b(x, t)}{p_b(x, t)}\). For \(b, b' \in \{0, 1\}^m\), observe 
        \begin{align*}
            \int_{I} ||s_{b}(x, t) - s_{b'}(x, t)||^2 \, \rd x &\gtrsim \int_{I} \left|\left|\psi_{b}(x, t)p_{b'}(x, t) - \psi_{b'}(x, t)p_b(x, t)\right|\right|^2 \, \rd x
        \end{align*}
        where again we have used \(p_{b}(\cdot, t), p_{b'}(\cdot, t) \gtrsim 1\) on \([-1, 1]^d\). Arguing similarly as in the proof of Theorem \ref{thm:score_lowerbound}, it can be obtained 
        \begin{align}
            &\int_{I} \left|\left|\psi_{b}(x, t)p_{b'}(x, t) - \psi_{b'}(x, t)p_b(x, t)\right|\right|^2 \, \rd x \nonumber \\
            &\ge c^2\left(1 - \frac{1}{\tilde{C}}\right)\int_I \left|\left|\psi_b(x, t) - \psi_{b'}(x, t)\right|\right|^2 \, \rd x - \tilde{C} \int_{I} \left|\left|\psi_{b'}(x, t)\right|\right|^2 \cdot \left|p_{b'}(x, t) - p_{b}(x, t)\right|^2 \, \rd x \label{eqn:pointwise_score_separation_d}
        \end{align}
        for a small universal constant \(c\) and a large universal constant \(\tilde{C}\). To further lower bound (\ref{eqn:pointwise_score_separation_d}), the slack from the second term can be bounded first. \newline

        \noindent \textbf{Case 1:} Suppose \(\alpha \ge 1\). Note we can write \(f_b = f_0 + \epsilon^{\alpha} \sum_{i=1}^{m} b_i g_i(\rho^{-1}\cdot)\) where \(g_i(\mu) = w\left(\mu-\rho^{-1} x_i\right)\). Note \(g_i \in C^\infty(\R^d)\). Consider \(||f_b||_\infty \lesssim 1\), and so 
        \begin{align*}
            \left|\left|\psi_{b'}(x, t)\right|\right| &\le \left|\left|\nabla ( \varphi_t * f_0)(x)\right|\right| + \epsilon^{\alpha} \sum_{i=1}^{m} \left|\left|\nabla (\varphi_t * g_i(\rho^{-1} \cdot )) \right| \right| \\
            &= \left|\left|\nabla ( \varphi_t * f_0)(x)\right|\right| + \epsilon^{\alpha}\rho^{-1} \sum_{i=1}^{m} \sqrt{\sum_{j=1}^{d} \left|\left(\varphi_t * \partial_j g_i\left(\rho^{-1} \cdot \right)\right)\right|^{2}}\\
            &\le \left|\left|\nabla(\varphi_t * f_0)(x)\right|\right| + C' \epsilon^{\alpha} \rho^{-1} m
        \end{align*}
        where \(C' > 0\) is a universal constant. To bound \(||\nabla(\varphi_t * f_0)(x)||\), note by the fundamental theorem of calculus we have for \(x \in I\), 
        \begin{align*}
            ||\nabla (\varphi_t * f_0)(x)|| &= \sqrt{\sum_{j=1}^{d} \left|\left(\partial_j \varphi_t * f_0\right)(x)\right|^2} \\
            &= \sqrt{\sum_{j=1}^{d}\left|\frac{1}{2^d} \int_{[-1,1]^d} (\partial_j\varphi_t)(x-\mu) \, \rd \mu\right|^2} \\
            &= \sqrt{\sum_{j=1}^{d}\left|\left(\frac{1}{2} \int_{-1}^{1} \gamma_t'(x_j-\mu_j) \, \rd \mu_j\right) \prod_{k \neq j} \left(\frac{1}{2} \int_{-1}^{1} \gamma_t(x_k - \mu_k) \, \rd \mu_k\right)\right|^2}\\
            &\lesssim t^{-(d-1)/2}\sqrt{\sum_{j=1}^{d}\left|\left(\frac{1}{2} \int_{-1}^{1} \gamma_t'(x_j-\mu_j) \, \rd \mu_j\right)\right|^2} \\
            &= t^{-(d-1)/2}\sqrt{\sum_{j=1}^{d} \frac{\left|\gamma_t(x_j+1) - \gamma_t(x_j-1)\right|^2}{4}} \\
            &\lesssim t^{-(d-1)/2} \cdot t^{(C-1)/2} 
        \end{align*}
        where \(\gamma_t\) is the probability density of the univariate Gaussian \(\mathcal{N}(0, t)\). The final line follows from the fact \(x \in I\). Therefore,
        \begin{align}
            \int_{I} \left|\left|\psi_{b'}(x, t)\right|\right|^2  \cdot \left|p_{b'}(x, t) - p_{b}(x, t)\right|^2 &\lesssim \left(t^{C-d} + \epsilon^{2\alpha} \rho^{-2} m^2\right) \int_{I} |p_{b'}(x, t) - p_{b}(x, t)|^2 \, \rd x \nonumber \\ 
            &\le \left(t^{C-d} + \epsilon^{2\alpha} \rho^{-2} m^2\right) \int_{\R^d} |(\varphi_t * (f_{b'} - f_{b}))(x)|^2 \, \rd x \nonumber \\
            &\le \left(t^{C-d} + \epsilon^{2\alpha} \rho^{-2} m^2\right) ||f_{b'} - f_b||^2 \nonumber \\
            &\lesssim \left(t^{C-d} + \epsilon^{2\alpha} \rho^{-2} m^2\right) \epsilon^{2\alpha} \rho^d \cdot d_{\text{Ham}}(b, b') \label{eqn:pointwise_score_separation_lowerorder_bigalpha_d}.
        \end{align}

        \noindent \textbf{Case 2:} Suppose \(\alpha < 1\). Note we have \(|f_{b'}(x) - f_{b'}(y)| \lesssim ||x-y||^\alpha\) for \(x, y \in [-1, 1]^d\). Therefore, for \(x \in I\) it can be shown from an argument analogous to that in the proof of Theorem \ref{thm:score_lowerbound} that 
        \begin{align*}
            ||\psi_{b'}(x, t)|| &\lesssim t^{\frac{\alpha-1}{2}} + \frac{1}{\sqrt{t}} \sqrt{\bP \left\{||\mathcal{N}(x, tI_d)||_\infty > 1\right\}} \lesssim t^{\frac{\alpha-1}{2}}.
        \end{align*}
        Here, we have used that \(x \in I\) implies \( \sqrt{\bP \left\{||\mathcal{N}(x, tI_d)||_\infty > 1\right\}} \lesssim t^{\alpha/2}\) since \(C\) can be taken sufficiently large. Therefore, we have 
        \begin{equation}
            \int_{I} \left|\left|\psi_{b'}(x, t)\right|\right|^2  \cdot \left|p_{b'}(x, t) - p_{b}(x, t)\right|^2 \lesssim t^{\alpha-1} \cdot \epsilon^{2\alpha}\rho^d \cdot d_{\text{Ham}}(b, b') \label{eqn:pointwise_score_separation_lowerorder_smallalpha_d}
        \end{equation}
        when \(\alpha < 1\). This concludes the analysis for this case. \newline 

        \noindent \textbf{Bounding the score separation:}
        From (\ref{eqn:pointwise_score_separation_d}), (\ref{eqn:pointwise_score_separation_lowerorder_bigalpha_d}), and (\ref{eqn:pointwise_score_separation_lowerorder_smallalpha_d}), it follows 
        \begin{align*}
            \int_{I}||s_b(x, t) - s_{b'}(x, t)||^2 \, \rd x \ge &\tilde{c}_1 \int_{I} ||\psi_b(x, t) - \psi_{b'}(x, t)||^2 \, \rd x \\
            &- \tilde{C}_1
            \begin{cases}
                 \left(t^{C-d} + \epsilon^{2\alpha} \rho^{-2} m^2\right) \epsilon^{2\alpha} \rho^d \cdot d_{\text{Ham}}(b, b') &\textit{if } \alpha \ge 1, \\
                 t^{\alpha-1} \cdot \epsilon^{2\alpha}\rho^d \cdot d_{\text{Ham}}(b, b') &\textit{if } \alpha < 1,
            \end{cases}
        \end{align*}
        for some universal constants \(\tilde{C}_1, \tilde{c}_1 > 0\). It remains to lower bound \(\int_{I} ||\psi_b(x, t) - \psi_{b'}(x, t)||^2 \, \rd x\). By Proposition \ref{prop:derivative_separation_d}, there exists a sufficiently large universal constant \(C_3 > 0\) so that the choice
        \begin{equation*}
            \rho = C_3\sqrt{t} \vee \epsilon
        \end{equation*}
        yields 
        \begin{equation*}
            \int_I ||s_b(x, t) - s_{b'}(x, t)||^2 \, \rd x \ge \epsilon^{2\alpha} d_{\text{Ham}}(b, b')\rho^{d}\left(\tilde{c}_2 \rho^{-2} - \tilde{C}_1 \begin{cases}
                \left(t^{C-d} + \epsilon^{2\alpha} \rho^{-2} m^2\right) &\textit{if } \alpha \ge 1, \\
                t^{\alpha-1} &\textit{if } \alpha < 1. 
           \end{cases}\right)
        \end{equation*}
        where \(\tilde{c}_2 > 0\) is some universal constant. It can be shown 
        \begin{equation}\label{eqn:score_separation_d}
            \int_I |s_b(x, t) - s_{b'}(x, t)|^2 \, \rd x \gtrsim \epsilon^{2\alpha} \rho^{d-2} d_{\text{Ham}}(b, b')
        \end{equation}
        for all \(\alpha > 0\), under the condition 
        \begin{equation}\label{eqn:eps_condition_smallalpha_d}
            \epsilon \lesssim \sqrt{t}
        \end{equation}
        if \(\alpha < 1\). The score separation bound (\ref{eqn:score_separation_d}) is suitable for use in Fano's method. \newline

        \noindent \textbf{Information theory:} 
        The next ingredient in Fano's method is a bound on Kullback-Leibler divergences between the data generating distributions \(\left\{f_b^{\otimes n}\right\}_{b \in \{0, 1\}^m}\). In fact, it suffices to bound \(\dKL(f_b^{\otimes n} \,||\, f_0^{\otimes n})\) for each \(b \in \{0, 1\}^m\) in the version of Fano's method given by Corollary 2.6 in \cite{tsybakov_introduction_2009}, rather than bound all pairwise Kullback-Leibler divergences. A direct calculation shows 
        \begin{equation}
            \dKL(f_b^{\otimes n}\,||\,f_0^{\otimes n}) \le n\chi^2(f_b\,||\,f_0) \lesssim n\epsilon^{2\alpha}m\rho^d. \label{eqn:KL_radius_d}
        \end{equation}
        With this bound on the Kullback-Leibler divergence in hand, the final remaining ingredient is to construct a packing set. \newline
        
        \noindent \textbf{Packing:} 
        By the Gilbert-Varshamov bound (see Lemma 2.9 in \cite{tsybakov_introduction_2009}), there exists a subset \(\mathcal{B} \subset \{0, 1\}^m\) such that \(\log |\mathcal{B}| \gtrsim m\) and \(\min_{b \neq b' \in \mathcal{B}} d_{\text{Ham}}(b, b') \gtrsim m\). Thus, from (\ref{eqn:score_separation_d}), the score separation on this subcollection is 
        \begin{equation}\label{eqn:separation_final_I_d}
            \int_{I} ||s_b(x, t) - s_{b'}(x, t)||^2 \, \rd x \gtrsim \epsilon^{2\alpha}\rho^{d-2}m 
        \end{equation}
        for \(b \neq b' \in \mathcal{B}\). \newline 
        
        \noindent \textbf{Applying Fano's method:} From (\ref{eqn:KL_radius_d}), (\ref{eqn:separation_final_I_d}), and \(\log|\mathcal{B}| \gtrsim m\), Fano's method (see Corollary 2.6 in \cite{tsybakov_introduction_2009}) yields 
        \begin{align*}
            \inf_{\hat{s}} \sup_{f \in \mathcal{F}_\alpha} \bE \left(\int_{\R^d} \left|\left|\hat{s}(x,t) - s(x, t) \right| \right|^2 \, p(x, t)\,\rd x\right) &\gtrsim \epsilon^{2\alpha}\rho^{d-2}m \left(1 - \frac{C'' n\epsilon^{2\alpha}m\rho^d + \log 2}{m}\right)
        \end{align*}
        for some universal constant \(C'' > 0\). We are now in position to choose \(\epsilon\) subject to the constraint \(\epsilon \lesssim 1\), and the additional constraint (\ref{eqn:eps_condition_smallalpha_d}) if \(\alpha < 1\). Note that the constraint \(\epsilon \le \rho\) is already satisfied by our choice of \(\rho\). Let us select 
        \begin{equation*}
            \epsilon \asymp \left( (nt^{d/2})^{-\frac{1}{2\alpha}} \wedge 
            \begin{cases}
                n^{-\frac{1}{2\alpha+d}} &\textit{if } \alpha \ge 1, \\
                \sqrt{t} &\textit{if } \alpha < 1,     
            \end{cases}
            \right)
        \end{equation*}
        Noting \(m \asymp \rho^{-d}\), this choice can be shown to yield the bound 
        \begin{equation*}
            \inf_{\hat{s}} \sup_{f \in \mathcal{F}_{\alpha}} \bE \left(\int_{\R^d} \left|\left|\hat{s}(x,t) - s(x, t) \right|\right|^2 \, p(x, t)\,\rd x\right) \gtrsim \epsilon^{2\alpha}\rho^{d-2}m \asymp \frac{1}{nt^{d/2+1}} \wedge \left(n^{-\frac{2(\alpha-1)}{2\alpha+d}} + t^{\alpha-1}\right)
        \end{equation*}
        as desired.
    \end{proof}

    \begin{proposition}\label{prop:derivative_separation_d}
        In the context of the proof of Theorem \ref{thm:score_lowerbound_d}, there exists a universal constant \(C_3 > 0\) such that if \(\rho = C_3 \sqrt{t} \vee \epsilon\), then 
        \begin{equation*}
            \int_{I} ||\psi_b(x, t) - \psi_{b'}(x, t)||^2 \, \rd x \gtrsim d_{\text{Ham}}(b, b') \epsilon^{2\alpha} \rho^{d-2}
        \end{equation*}
        for \(b, b' \in \{0, 1\}^m\). 
    \end{proposition}
    \begin{proof}[Proof sketch]
        The proof is quite similar to that of the proof of Proposition \ref{prop:derivative_separation} so we only give a sketch. For \(b, b' \in \{0, 1\}^m\), define
        \begin{equation}\label{def:Gamma_I_d}
            \Gamma_{b,b'}(t) = \int_{\R^d} ||\psi_{b}(x, t) - \psi_{b'}(x, t)||^2 \, \rd x. 
        \end{equation}
        Note 
        \begin{equation}\label{eqn:separation_prep_I_d}
            \int_{I} ||\psi_b(x, t) - \psi_{b'}(x, t)||^2 \, \rd x = \Gamma_{b, b'}(t) - \int_{I^c} ||\psi_b(x, t) - \psi_{b'}(x, t)||^2 \, \rd x. 
        \end{equation}
        Following the same idea as in the proof of Proposition \ref{prop:derivative_separation}, consider by Taylor's theorem we have
        \begin{equation}\label{eqn:Gamma_taylor_I_d}
            \Gamma_{b,b'}(t) = \Gamma_{b,b'}(0) + \Gamma_{b,b'}'(\xi)t
        \end{equation}
        for some \(\xi \in (0, t)\). 
        
        A notable difference from Proposition \ref{prop:derivative_separation} is that the heat equation we will now examine is in the context of dimension \(d\) which may be larger than one. Specifically, consider for a function \(h : \R \to \R\) that the convolution \((\varphi_t * h)(x)\) satisfies the heat equation \(\frac{\rd}{\rd t} (\varphi_t * h)(x) = \frac{1}{2}\Delta (\varphi_t * h)(x) = \frac{1}{2}\sum_{j=1}^{d} (\partial_{j}^2 \varphi_t * h)(x)\) where \(\Delta\) denotes the Laplacian operator in \(d\) dimensions. If \(h\) is compactly supported and differentiable everywhere, it immediately follows 
        \begin{equation}\label{eqn:heat_equation_norm_d}
            \frac{\rd}{\rd t} \int_{\R^d} |(\varphi_t * h)(x)|^2 \,\rd x = - \int_{\R^d} \left|\left|\nabla (\varphi_t * h)(x)\right|\right|^2 \, \rd x = -\sum_{j=1}^{d} \int_{\R^d} \left|(\varphi_t * \partial_j h)(x)\right|^2 \, \rd x.
        \end{equation}
        Furthermore, observe we can deduce \(||\varphi_t*h||_{L^2(\R^d)} \le ||h||_{L^2(\R^d)}\) since the time derivative (\ref{eqn:heat_equation_norm_d}) nonpositive. Since \(g_i \in C^\infty(\R^d)\) is compactly supported, consider 
        \begin{equation*}
            \left|\left|\psi_{b}(x, t) - \psi_{b'}(x, t)\right|\right|^2 = \epsilon^{2\alpha}\rho^{-2} \sum_{j=1}^{d} \left|\left(\varphi_t * \sum_{i=1}^{m} (b_i - b_i')(\partial_j g_i(\rho^{-1}\cdot))\right)(x)\right|^2.
        \end{equation*}
        Therefore by (\ref{eqn:heat_equation_norm_d}), 
        \begin{align*}
            \Gamma_{b,b'}'(t) &= \epsilon^{2\alpha} \rho^{-2} \sum_{j=1}^{d} \frac{\rd}{\rd t} \int_{\R^d} \left|\left(\varphi_t * \sum_{i=1}^{m} (b_i - b_i')\left(\partial_j g_i(\rho^{-1}\cdot) \right)\right)(x)\right|^2 \, \rd x \\
            &= -\epsilon^{2\alpha}\rho^{-4} \sum_{j=1}^{d} \sum_{j' = 1}^{d} \int_{\R^d} \left|\left(\varphi_t * \sum_{i=1}^{m} (b_i - b_i')\left(\partial_{jj'} g_i(\rho^{-1}\cdot) \right)\right)(x)\right|^2\, \rd x. 
        \end{align*}
        Since the collection \(\{g_i\}_{i=1}^{m}\) have disjoint support, it is immediate for any \(\xi > 0\), 
        \begin{align*}
            \Gamma_{b,b'}(0) &= \epsilon^{2\alpha}\rho^{d-2} d_{\text{Ham}}(b, b') \int_{\R^d} \left|\left|\nabla w(y)\right|\right|^2 \, \rd y, \\
            \Gamma_{b,b'}'(\xi) &\ge -\epsilon^{2\alpha} \rho^{d-4} d_{\text{Ham}}(b,b')\sum_{1 \leq j, j' \leq d} \int_{\R^d} \left|\partial_{jj'} w \left(y\right)\right|^2 \, \rd y. 
        \end{align*}
        The last inequality follows from the fact \(||\varphi_\xi * h||_{L^2(\R)} \le ||h||_{L^2(\R)}\). Therefore, from (\ref{eqn:Gamma_taylor_I_d}) we have 
        \begin{align}
            \Gamma_{b,b'}(t) &\ge \epsilon^{2\alpha} \rho^{d-2} d_{\text{Ham}}(b, b')\left(\int_{\R^d} ||\nabla w(y)||^2 \, \rd y - t\rho^{-2} \sum_{1 \leq j, j' \leq d} \int_{\R^d} |\partial_{jj'}w(y)|^2 \, \rd y \right) \nonumber \\
            &\ge \epsilon^{2\alpha}\rho^{d-2} d_{\text{Ham}}(b, b') \left(c' - C' t\rho^{-2}\right). \label{eqn:Gamma_signal_I_d}
        \end{align}
        Now that we have handled \(\Gamma_{b,b'}(t)\), let us turn to the second term in (\ref{eqn:separation_prep_I_d}). Consider
        \begin{align*}
            \int_{I^c} ||\psi_b(x, t) - \psi_{b'}(x, t)||^2 \, \rd x &= \epsilon^{2\alpha} \rho^{-2} \sum_{j=1}^{d} \int_{I^c} \left|\left(\varphi_t * \sum_{i=1}^{m}(b_i - b_i')\partial_jg_i(\rho^{-1} \cdot )\right)(x)\right|^2 \, \rd x. 
        \end{align*}        
        To evaluate the integral, first fix a point \(x \in I^c\). Consider since \(w\) is supported on \([-1, 1]^d\) and since \(x_i \in [-1+\sqrt{Ct \log(1/t)} + C_D\rho, 1-\sqrt{Ct\log(1/t)} - C_D\rho]^d\), it can be shown 
        \begin{equation*}
            \left|\left(\varphi_t *\sum_{i=1}^{m} (b_i - b_i')\partial_j g_i(\rho^{-1} \cdot) \right)(x)\right|^2 \lesssim \sum_{i=1}^{m} (b_i - b_i')^2 \bP \left\{||\mathcal{N}(x, tI_d) - x_i||_\infty \le \rho\right\}.
        \end{equation*}
        Since \(x \in I^c\) implies \(||x||_\infty > 1 - \sqrt{C t \log(1/t)}\), it follows \(||x - x_i||_\infty \geq ||x||_\infty - 1 + 1 - ||x_i||_\infty \geq ||x||_\infty - 1 + \sqrt{C t \log(1/t)} + C_D\rho\). Therefore, by triangle inequality and union bound
        \begin{align*}
            \bP \left\{||\mathcal{N}(x, tI_d) - x_i||_\infty \le \rho\right\} &\leq \bP\left\{||\mathcal{N}(0, I_d)||_\infty \geq \frac{||x-x_i||_\infty - \rho}{\sqrt{t}}\right\} \\
            &\leq d \bP\left\{ |\mathcal{N}(0, 1)| \geq \frac{||x-x_i||_\infty - \rho}{\sqrt{t}}\right\} \\
            &\lesssim \bP\left\{ |\mathcal{N}(0, 1)| \geq \frac{||x||_\infty - 1 + \sqrt{Ct\log(1/t)} + (C_D-1)\rho}{\sqrt{t}}\right\}. 
        \end{align*}
        With this bound in hand, observe 
        \begin{align*}
            &\int_{I^c}\left|\left(\varphi_t *\sum_{i=1}^{m} (b_i - b_i')\partial_j g_i(\rho^{-1} \cdot) \right)(x)\right|^2\, \rd x \\
            &\lesssim d_{\text{Ham}}(b, b')\int_{||x||_\infty \geq 1-\sqrt{Ct\log(1/t)}} \bP \left\{|\mathcal{N}(0, 1)| \ge \frac{||x||_\infty-1+\sqrt{Ct\log(1/t)}+(C_D-1)\rho}{\sqrt{t}}\right\}\,\rd x. 
        \end{align*} 
        To summarize, we have thus shown 
        \begin{align}
            &\int_{I^c} ||\psi_b(x, t) - \psi_{b'}(x, t)||^2 \, \rd x \nonumber \\
            &\le C'' \epsilon^{2\alpha}\rho^{-2} d_{\text{Ham}}(b, b')\int_{||x||_\infty \geq 1-\sqrt{Ct\log(1/t)}} \bP \left\{\mathcal{N}(0, 1) \ge \frac{||x||_\infty-1+\sqrt{Ct\log(1/t)}+(C_D-1)\rho}{\sqrt{t}}\right\}\,\rd x \label{eqn:psi_err_I_d} 
        \end{align}
        for some universal constant \(C'' > 0\). Plugging (\ref{eqn:Gamma_signal_I_d}) and (\ref{eqn:psi_err_I_d}) into (\ref{eqn:separation_prep_I_d}) yields the following lower bound on the separation 
        \begin{align}
            &\int_{I} ||\psi_b(x, t) - \psi_{b'}(x, t)||^2 \, \rd x \nonumber \\
            \ge&~ d_{\text{Ham}}(b, b') \epsilon^{2\alpha}\rho^{-1}\cdot \nonumber \\
            &\left(c' - C' t\rho^{-2} - C''\rho^{-1}\int_{||x||_\infty \geq 1-\sqrt{Ct\log(1/t)}}  \bP \left\{\mathcal{N}(0, 1) \ge \frac{||x||_\infty-1+\sqrt{Ct\log(1/t)}+(C_D-1)\rho}{\sqrt{t}}\right\}\,\rd x\right).\label{eqn:separation_prep2_I_d}
        \end{align}
        It can be reasoned that
        \begin{align*}
            &\int_{||x||_\infty \geq 1-\sqrt{Ct\log(1/t)}}  \bP \left\{\mathcal{N}(0, 1) \ge \frac{||x||_\infty-1+\sqrt{Ct\log(1/t)}+(C_D-1)\rho}{\sqrt{t}}\right\}\,\rd x \\
            &\lesssim \int_{1 - \sqrt{Ct\log(1/t)}}^{\infty} y^{d-1}\bP\left\{ \mathcal{N}(0, 1) \geq \frac{y - 1 + \sqrt{Ct \log(1/t)} + (C_D - 1)\rho}{\sqrt{t}}\right\} \, \rd y
        \end{align*}
        since the region \(R(y) = \left\{ x \in \R^d : ||x||_\infty = y \right\}\) has \(\R^{d-1}\)-Lebesgue surface measure \((2y)^{d-1}\). It can then be argued 
        \begin{align*}
            & \rho^{-1}  \int_{1 - \sqrt{Ct\log(1/t)}}^{\infty} y^{d-1} \bP\left\{ \mathcal{N}(0, 1) \geq \frac{y - 1 + \sqrt{Ct \log(1/t)} + (C_D - 1)\rho}{\sqrt{t}}\right\} \, \rd y \\
            &= \rho^{-1} t^{d/2} \int_{\frac{(C_D-1)\rho}{\sqrt{t}}}^{\infty} \left(u + \frac{1 - \sqrt{Ct \log(1/t)} - (C_D - 1)\rho}{\sqrt{t}} \right)^{d-1} \bP\left\{ \mathcal{N}(0, 1) \geq u\right\} \, \rd u \\
            &\lesssim \rho^{-2}t + \rho^{-1} t^{d/2}. 
        \end{align*}
        Hence, we have arrived from (\ref{eqn:psi_err_I_d}) to the following bound, 
        \begin{align*}
            \int_{I} ||\psi_b(x, t) - \psi_{b'}(x, t)||^2 \, \rd x \ge d_{\text{Ham}}(b, b') \epsilon^{2\alpha}\rho^{d-2} \left(c' - (C'+C''') t\rho^{-2} - C'''\rho^{-1}t^{d/2}\right)
        \end{align*}
        where \(C''' > 0\) is some universal constant. Recall that the following choice is made
        \begin{equation}\label{eqn:rho_choice_I_d}
            \rho = C_3\sqrt{t} \vee \epsilon.
        \end{equation}
        Noting \(t \lesssim 1\) implies \(t^{d/2} \lesssim t^{1/2}\), let‘s take \(C_3 > 0\) sufficiently large so that we obtain the separation 
        \begin{equation*}
            \int_{I} |\psi_b(x, t) - \psi_{b'}(x, t)|^2 \, \rd x \gtrsim d_{\text{Ham}}(b, b') \epsilon^{2\alpha}\rho^{d-2}
        \end{equation*}
        as desired. 
         
    \end{proof}

\subsection{Upper Bound}
In this section, we only address the regime \(t \lesssim 1\) from the upper bound perspective, as the analysis in the regime \(t \gtrsim 1\) for \(d > 1\) follows in a very straightforward manner from the analysis in the case \(d = 1\). Broadly speaking, our score estimator as well as the main framework of proofs stays the same. We use unbiased estimators for the diffused density and its derivative in the high noise regime $t\gtrsim n^{-\frac{2}{2\alpha+d}}$. In the low noise regime \(t \lesssim n^{-\frac{2}{2\alpha+d}}\), we use a kernel-based estimator when \(\alpha \ge 1\) and a data-free estimator when \(\alpha < 1\). In our proof of upper bounds, we split \(\R^d\) into three regions: the internal part $D_1 := \left\{x:~\|x\|_\infty \le 1-\sqrt{Ct\log(1/t)}\right\}$, the boundary part $D_2:=\left\{x:~1-\sqrt{Ct\log(1/t)} < \|x\|_{\infty} \le 1+C\sqrt{t}\right\}$, and the external part $D_3:=\left\{x:~\|x\|_{\infty} > 1+C\sqrt{t}\right\}$. Since the only differences with the $d=1$ case are standard, technical ones, we only provide a proof sketch below, and we focus on two typical cases: the external part under high noise regime and the boundary part under low noise regime. Once we prove these two cases, all the other parts can be extended by using similar techniques. 

\begin{theorem}
\label{thm:score_upperbound_extend}
For $\alpha > 0$, there exist universal constants $c_1 = c_1(\alpha, L, d)$ and $0 < c_2 = c_2(\alpha, L, d)$ only dependent on $\alpha, L, d$ such that
\begin{equation*}
\inf_{\hat{s}} \sup_{f \in \mathcal{F}_\alpha} \bE\left(\int_{\bR^d} \left|\left|\hat{s}(x, t) - s(x, t)\right|\right|^2 \, p(x, t) \rd x \right) \le c_1\left(\frac{1}{nt^{1+d/2}} \wedge \left(n^{-\frac{2(\alpha-1)}{2\alpha+d}} + t^{\alpha-1}\right)\right)
\end{equation*}
for all $0 < t < c_2$. 
\end{theorem}
\begin{proof}[Proof sketch]
As in the case \(d = 1\), the analysis splits depending on the value of \(t\). ~\\

\noindent \textbf{High-noise regime:} 
In the high-noise regime \(t \gtrsim n^{-\frac{2}{2\alpha+d}}\), we use
\[\hat{s}_n(x,t) = \frac{\hat{\psi}_n(x,t)}{\hat{p}_n(x,t)\vee \varepsilon(x)}.\]
Here, 
\[\hat{\psi}_n(x,t) := \frac1n\sum_{j=1}^n \phi_t(x-\mu_j)\cdot \frac{\mu_j-x}{t}~~~\text{and}~~~\hat{p}_n(x,t) := \frac1n\sum_{j=1}^n \phi_t(x-\mu_j)\]
are the unbiased estimators of $\psi(x,t)$ and $p(x,t)$ where $\phi_t(x):= \frac{1}{(2\pi t)^{d/2}} \exp\left(-\frac{\|x\|^2}{2t}\right)$ is the density of $\mathcal N(0, t I_d)$. The regularizer $\varepsilon(x)$ in the denominator has the following form,
\[\varepsilon(x) = c_d\cdot\int_{\bR^d} \phi_t(x-\mu) \rd \mu. \]
By similar calculation as (\ref{eqn:6-1}) and (\ref{eqn:6-2}), we easily conclude that
\begin{equation}
\label{eqn:6-1-extend}
\bE \left\|\hat{\psi}_n(x,t)-\psi(x,t)\right\|^2 \lesssim \frac{1}{nt^{1+d/2}}\cdot \bE_{z\sim \mathcal N(0,I_d)} f(x+\sqrt{t}z).
\end{equation}
\begin{equation}
\label{eqn:6-2-extend}
\bE \left(\hat{p}_n(x,t)-p(x,t)\right)^2 \lesssim \frac{1}{nt^{d/2}}\cdot \bE_{z\sim \mathcal N(0,I_d)} f(x+\sqrt{t/2} z).
\end{equation}
To bound the score estimation error, some extended properties on high-dimensional Gaussian distribution are required, which are basically derived from lemmas in Appendix \ref{sec:useful}, so their complete proofs are omitted. 
\begin{lemma}[Extension of Lemma \ref{lemma:up-1}, \ref{lemma:up-3}]
\label{lemma:up-1-extension}
For a universal constant $C > 0$, we have $p(x, t) \gtrsim 1$ holds for all $\|x\|_\infty \le 1+C\sqrt{t} $. For $\|x\|_\infty > 1+C\sqrt{t}$, it holds that
\[p(x,t) \asymp \prod_{i:~|x_i|> 1+C\sqrt{t}} \left(\frac{\sqrt{t}}{|x_i|-1} \wedge 1\right)\cdot \exp\left(-\frac{(|x_i|-1)^2}{2t}\right). \]
\end{lemma}
\begin{lemma}[Extension of Lemma \ref{lemma:score-up}]
\label{lemma:score-up-extension}
If \(\|x\|_\infty \le 1 + C\sqrt{t}\), then $\|s(x,t)\|\lesssim \frac{1}{\sqrt{t}}$. If \(\alpha \ge 1\) and $\|x\|_\infty \le 1-\sqrt{t\log(1/t)}$, we have a tighter bound $\|s(x,t)\| \lesssim 1$. Additionally, for all $x\in \bR^d$ it always holds that 
\[\|s(x,t)\|^2 \le \frac{2}{t} \log\frac{1}{(2\pi t)^{d/2} p(x,t)}. \]
\end{lemma}
According to (\ref{eqn:score-1}), we can conclude that
\begin{equation}
\begin{aligned} 
\label{eqn:score-1-extension}
&\int_{\bR^d} \bE \left\|\hat{s}_n(x,t)-s(x,t)\right\|^2 \cdot p(x,t)\rd x \\
&\quad \lesssim \int_{\bR^d} \frac{\bE \left\|\hat{\psi}_n(x,t)-\psi(x,t)\right\|^2}{p(x,t)}\,\rd x + \int_{\bR^d} \|s(x,t)\|^2 \cdot \frac{\bE \left(\hat{p}_n(x,t)-p(x,t)\right)^2}{p(x,t)} \rd x. 
\end{aligned}
\end{equation}
Therefore, when $\|x\|_{\infty} \le 1+C\sqrt{t}$, we have $p(x,t)\gtrsim 1$ and $\|s(x,t)\|^2 \lesssim \frac1t$, which directly leads to 
\begin{equation}
\label{eqn:ex-up-1}
\int_{D_1\cup D_2} \bE \left\|\hat{s}_n(x,t)-s(x,t)\right\|^2 \cdot p(x,t)\rd x \lesssim \left(\frac{1}{nt^{1+d/2}}+ \frac{1}{t}\cdot \frac{1}{nt^{d/2}}\right) \cdot (1+C\sqrt{t})^d \lesssim \frac{1}{nt^{1+d/2}}. 
\end{equation}
after plugging in \eqref{eqn:6-1-extend} and \eqref{eqn:6-2-extend}. For the external part (i.e. \(D_3\)), we need to compare the exponential tail of the numerator and denominator, which is more difficult. Extended from Equation (\ref{eqn:6-3}), we have 
\[\bE \left\|\hat{\psi}_n(x,t)-\psi(x,t)\right\|^2 \lesssim \frac{1}{nt^{1+d/2}} \cdot \prod_{i:~|x_i|>1+C\sqrt{t}} \exp\left(-\frac{2(|x_i|-1)^2}{3t}\right). \]
After combining with Lemma \ref{lemma:up-1-extension}, we conclude that
\[\frac{\bE \left\|\hat{\psi}_n(x,t)-\psi(x,t)\right\|^2}{p(x,t)} \lesssim \frac{1}{nt^{1+d/2}}\cdot \prod_{i:~|x_i|>1+C\sqrt{t}} \exp\left(-\frac{(|x_i|-1)^2}{6t}\right). \]
In the external part $D_3$, at least one component $i$ satisfies $|x_i| > 1+C\sqrt{t}$. Therefore,
\begin{align}
\label{eqn:ex-1}
&\quad\int_{D_3} \frac{\bE \left\|\hat{\psi}_n(x,t)-\psi(x,t)\right\|^2}{p(x,t)}\,\rd x \notag \\
&\lesssim \frac{1}{nt^{1+d/2}}\cdot \int_{D_3} \prod_{i:~|x_i|>1+C\sqrt{t}} \exp\left(-\frac{(|x_i|-1)^2}{6t}\right) \rd x \notag\\
&= \frac{1}{nt^{1+d/2}} \left[\prod_{i=1}^d \left(\int_{|x_i|\le 1+C\sqrt{t}} \rd x_i + \int_{|x_i|> 1+C\sqrt{t}} \exp\left(-\frac{(|x_i|-1)^2}{6t}\right) \rd x_i\right)- \prod_{i=1}^d \left(\int_{|x_i|\le 1+C\sqrt{t}} \rd x_i\right)\right]\notag\\
&= \frac{2^d}{nt^{1+d/2}} \left[(1+C\sqrt{t}+c_1\sqrt{t})^d - (1+C\sqrt{t})^d\right] \lesssim \frac{1}{nt^{(1+d)/2}}, 
\end{align}
which is a minor term. Here, $c_1$ is a constant. For the other term, since 
\begin{align*}
\bE \left(\hat{p}_n(x,t)-p(x,t)\right)^2 &\lesssim \frac{1}{nt^{d/2}} \bE_{z\sim \mathcal N(0, \bm I)} f(x+\sqrt{t/2}z) \lesssim \frac{1}{nt^{d/2}}\bP\left[\|x+\sqrt{t/2}z\|_\infty \le 1\right] \\
&\lesssim \frac{1}{nt^{d/2}} \prod_{i:~|x_i|>1+C\sqrt{t}}\frac{\sqrt{t}}{|x_i|-1}\exp\left(-\frac{(|x_i|-1)^2}{t}\right),
\end{align*}
we apply Lemma \ref{lemma:up-1-extension} and \ref{lemma:score-up-extension}, and conclude that
\begin{align*}
&\quad\int_{D_3} \|s(x,t)\|^2 \cdot \frac{\bE \left(\hat{p}_n(x,t)-p(x,t)\right)^2}{p(x,t)} \rd x \\
&\lesssim \frac{1}{nt^{d/2}}\int_{D_3} \frac{2}{t}\log\frac{1}{(2\pi t)^{d/2} p(x,t)} \cdot \prod_{i:~|x_i|>1+C\sqrt{t}} \exp\left(-\frac{(|x_i|-1)^2}{2t}\right) \rd x\\
&\lesssim \frac{1}{nt^{1+d/2}} \int_{D_3} \left[\log(1/t) + \sum_{i\in S} \frac{(|x_i|-1)^2}{2t}\right] \cdot \prod_{i\in S} \exp\left(-\frac{(|x_i|-1)^2}{2t}\right) \rd x \\
&\lesssim \frac{1}{nt^{1+d/2}}\cdot \sqrt{t}\log(1/t) + \frac{d}{nt^{1+d/2}} \int_{D_3} \prod_{i\in S} \exp\left(-\frac{(|x_i|-1)^2}{6t}\right) \rd x \\
&\lesssim \frac{1}{nt^{1+d/2}}\cdot \sqrt{t}\log(1/t) + \frac{1}{nt^{1+d/2}}\cdot \sqrt{t} \lesssim \frac{\log(1/t)}{nt^{(1+d)/2}},
\end{align*}
where $S:=\{i:~|x_i|>1+C\sqrt{t}\}$, which is also a minor term. Here we use the conclusion that $\int_{D_3} \prod_{i\in S} \exp\left(-\frac{(|x_i|-1)^2}{6t}\right) \rd x \lesssim \sqrt{t}$ obtained in \eqref{eqn:ex-1}. Therefore, we have
\[\int_{D_3} \bE \left\|\hat{s}_n(x,t)-s(x,t)\right\|^2 \cdot p(x,t)\rd x \lesssim \frac{\log(1/t)}{nt^{(1+d)/2}}\]
according to \eqref{eqn:score-1-extension}.  Finally, we combine it with \eqref{eqn:ex-up-1} and conclude the score integrated estimation error $\frac{1}{nt^{1+d/2}}$ under the high-noise regime. In the next section, we study the low-noise regime with kernel-based estimator applied. ~\\

\noindent \textbf{Low-noise regime:} In the low-noise regime \(t \lesssim n^{-\frac{2}{2\alpha+d}}\), we apply a kernel-based estimator. 
\begin{lemma}[Extension of Lemma \ref{lemma:upper-1}] Given $n$ i.i.d samples $\mu_1, \mu_2, \ldots, \mu_n\sim f$, there exists a kernel based estimator $\hat{f}_n(x)$ such that for all $k=0,1,\ldots, \lfloor\alpha \rfloor$, multi-index $|S|=k$ and $x\in [-1,1]^d$, 
\[\bE \left(\partial_S\hat{f}_n(x) - \partial_S f(x)\right)^2 \lesssim n^{-\frac{2(\alpha-k)}{2\alpha+d}}. \]
\end{lemma}
Since the internal and external parts are very similar to the $d=1$ case, from which we conclude that
\[\sup_{f\in \mathcal F_{\alpha}} \bE \int_{D_1\cup D_3} \left\|\hat{s}(x,t)-s(x,t)\right\|^2\cdot p(x,t)\rd x\lesssim n^{-\frac{2(\alpha-1)}{2\alpha+d}}. \]
Now we only focus on the boundary part (i.e. \(D_2\)) here. Let $\tilde{f}_n = \hat{f}_n \vee \frac{c_d}{2}$, and $\hat{p}_n(x,t) := \phi_t * \tilde{f}_n(x)$. For the density derivative $\psi(x,t)=\nabla_x p(x,t)$, we have via integration by parts, 
\begin{align*}
\partial_i p(x,t) &= \int_{[-1,1]^d} f(\mu) \cdot \frac{-\rd}{\rd \mu_i} \phi_t(x-\mu) \rd \mu = \int_{[-1,1]^{d-1}} \rd \mu_{-i} \cdot \int_{-1}^1 -f(\mu) \frac{\rd}{\rd \mu_i} \phi_t(x-\mu)\rd \mu_i \\
&= \int_{[-1,1]^{d-1}} f(\mu_1,\ldots, -1, \ldots, \mu_d) \phi_t(x_{-i}-\mu_{-i})\gamma_t(x_i+1) \rd \mu_{-i}\\
&\quad - \int_{[-1,1]^{d-1}} f(\mu_1,\ldots, 1, \ldots, \mu_d) \phi_t(x_{-i}-\mu_{-i})\gamma_t(x_i-1) \rd \mu_{-i}\\
&\quad + \int_{[-1,1]^d} \phi_t(x-\mu)\cdot \frac{\rd}{\rd \mu_i} f(\mu) \rd \mu. 
\end{align*}
Here, $\phi_t$ and $\gamma_t$ are the densities of  multivariate and univariate Gaussian distribution respectively. Let the estimator of density derivative
\begin{align*}
\left(\hat{\psi}_n(x,t)\right)_i &= \int_{[-1,1]^{d-1}} \hat{f}_n(\mu_1,\ldots, -1, \ldots, \mu_d) \phi_t(x_{-i}-\mu_{-i})\gamma_t(x_i+1) \rd \mu_{-i}\\
&\quad - \int_{[-1,1]^{d-1}} \hat{f}_n(\mu_1,\ldots, 1, \ldots, \mu_d) \phi_t(x_{-i}-\mu_{-i})\gamma_t(x_i-1) \rd \mu_{-i}\\
&\quad + \int_{[-1,1]^d} \phi_t(x-\mu)\cdot \frac{\rd}{\rd \mu_i} \hat{f}_n(\mu) \rd \mu 
\end{align*}
and the score estimator $\hat{s}_n(x,t) = \frac{\hat{\psi}_n(x,t)}{\hat{p}_n(x,t)}$. In the boundary part $D_2$, we extend from \eqref{eqn:J-ineq} and conclude that
\[\left\|\hat{s}_n(x,t)-s(x,t)\right\|^2\cdot p(x,t) \asymp \|J(x,t)\|^2. \]
Here, $J(x,t)=\hat{\psi}_n(x,t)p(x,t)-\psi(x,t)\hat{p}_n(x,t)$. It follows that 
\[\bE \|J(x,t)\|^2 \lesssim \sum_{i=1}^d \left[\bE J_{1,i}(x,t) + \bE J_{2,i}(x,t) + \bE J_{3,i}(x,t)\right]\]
where
\begin{align*}
J_{1,i}(x,t) &:= |\gamma_t(x_i-1)|^2 \cdot \Big | \int_{[-1,1]^{d-1}} \phi_t(x_{-i}-\mu_{-i}) f(\mu_1\dots 1\dots\mu_d) \rd \mu_{-i} \cdot \int_{[-1,1]^d} \phi_t(x-\mu)\hat{f}_n(\mu)\rd \mu\\
&\quad \quad \quad \quad -\int_{[-1,1]^{d-1}} \phi_t(x_{-i}-\mu_{-i}) \hat{f}_n(\mu_1\dots 1\dots\mu_d) \rd \mu_{-i} \cdot \int_{[-1,1]^d} \phi_t(x-\mu)f(\mu)\rd \mu \Big|^2\\
J_{2,i}(x,t) &:= |\gamma_t(x_i+1)|^2 \cdot \Big | \int_{[-1,1]^{d-1}} \phi_t(x_{-i}-\mu_{-i}) f(\mu_1\dots, -1,\dots\mu_d) \rd \mu_{-i} \cdot \int_{[-1,1]^d} \phi_t(x-\mu)\hat{f}_n(\mu)\rd \mu\\
&\quad \quad \quad \quad -\int_{[-1,1]^{d-1}} \phi_t(x_{-i}-\mu_{-i}) \hat{f}_n(\mu_1\dots ,-1,\dots\mu_d) \rd \mu_{-i} \cdot \int_{[-1,1]^d} \phi_t(x-\mu)f(\mu)\rd \mu \Big|^2\\
J_{3,i}(x,t) &:= \left[\int_{[-1,1]^d\times [-1,1]^d}\phi_t(x-\mu)\phi_t(x-\nu)\left(\frac{\rd}{\rd \mu_i}f(\mu)\hat{f}_n(\nu)-\frac{\rd}{\rd\mu_i} \hat{f}_n(\mu)f(\nu)\right)\rd \mu\rd \nu \right]^2. 
\end{align*}
For $J_{3,i}$, we use Jensen's Inequality and conclude that
\[J_{3,i} \le \bE_{\mu,\nu\sim \mathcal N(x, tI_d)} \left|\frac{\rd}{\rd \mu_i}f(\mu)\hat{f}_n(\nu)-\frac{\rd}{\rd\mu_i} \hat{f}_n(\mu)f(\nu)\right|^2 \lesssim n^{-\frac{2(\alpha-1)}{2\alpha+d}}. \]
Because of the symmetry, $J_{1,i}(x,t)$ and $J_{2,i}(x,t)$ are very similar and we only need to bound $J_{1,i}(x,t)$. If $x_i\in (1-\sqrt{Ct\log(1/t)}, 1+C\sqrt{t})$, then: $|\gamma_t(x_i-1)|^2 \lesssim \frac1t$ and extended from \eqref{eqn:Taylor}, we have
\begin{align*}
\bE J_{1,i}(x,t) &\lesssim \frac1t \cdot \Big ( \int_{[-1,1]^{d-1}\times [-1,1]^d} \rd \mu_{-i}\rd \nu\sum_{k=0}^{\lfloor\alpha\rfloor} \frac{1}{k!} \phi_t(x_{-i}-\mu_{-i})\phi_t(x-\nu) \cdot |\nu_i-1|^{2k} \cdot \\ 
& \quad \quad \bE \left|f(\mu_1\dots 1\dots\mu_d)\frac{\rd^k}{\rd \nu_i^k} \hat{f}_n(\nu_1\dots 1\dots\nu_d) - \hat{f}_n(\mu_1\dots 1\dots\mu_d)\frac{\rd^k}{\rd \nu_i^k} f(\nu_1\dots 1\dots\nu_d)\right|^2 \Big) \\
& \quad \quad + \frac1t \cdot (t\log(1/t))^\alpha\\
& \lesssim \frac1t \sum_{k=1}^{\lfloor\alpha\rfloor} n^{-\frac{2(\alpha-k)}{2\alpha+d}}\cdot (t\log(1/t))^k + \frac1t \cdot (t\log(1/t))^\alpha \lesssim \log^\alpha(1/t)\cdot n^{-\frac{2(\alpha-1)}{2\alpha+d}}. 
\end{align*}
If $x_i\in (-1-C\sqrt{t}, 1-\sqrt{Ct\log(1/t)})$, then $|\phi_t(x_i-1)|^2 \lesssim 1$. We can simply bound $J_{1,i}(x,t)$ as follows, 
\begin{align*}
J_{1,i}(x,t) &\lesssim \int_{[-1,1]^{d-1}\times [-1,1]^d} \rd\mu_{-i}\rd \nu\cdot  \phi_t(x_{-i}-\mu_{-i})\phi_t(x-\nu) \\
&\quad \quad \quad \left|f(\mu_1\dots 1 \dots \mu_d)\hat{f}_n(\nu)-\hat{f}_n(\mu_1\dots 1 \dots \mu_d)f(\nu)\right|^2 \lesssim n^{-\frac{2\alpha}{2\alpha+d}}. 
\end{align*}
To sum up, for any $x\in D_2$, we have $\bE \|\hat{s}_n(x,t)-s(x,t)\|^2\cdot p(x,t)\asymp \bE \|J(x,t)\|^2 \lesssim \log^\alpha(1/t)\cdot n^{-\frac{2(\alpha-1)}{2\alpha+d}}$. Therefore the score integrated estimation error on the boundary part can be upper bounded as
\[\int_{D_2} \bE \|\hat{s}_n(x,t)-s(x,t)\|^2\cdot p(x,t) \rd x \lesssim \log^\alpha(1/t)\cdot n^{-\frac{2(\alpha-1)}{2\alpha+d}} \cdot \mathrm{Vol}(D_2).\]
The volume of $D_2$ equals to 
\[\mathrm{Vol}(D_2)=(1+C\sqrt{t})^d - (1-\sqrt{Ct\log(1/t)})^d \lesssim \sqrt{t\log(1/t)}, \]
which finally leads to
\[\int_{D_2} \bE \|\hat{s}_n(x,t)-s(x,t)\|^2\cdot p(x,t) \rd x \lesssim \sqrt{t}\log^{\alpha+1/2}(1/t)\cdot n^{-\frac{2(\alpha-1)}{2\alpha+d}} \lesssim n^{-\frac{2(\alpha-1)}{2\alpha+d}}, \]
as desired.

\end{proof}

\subsection{Distribution estimation}

The extensions of Lemmas \ref{lem:KLappro}-\ref{lem:W1} to the high-dimensional case are straightforward, utilizing standard techniques. By incorporating a constant related to \(d\) into the \(\lesssim\) notation, the result naturally generalizes. Specifically, we substitute the corresponding score matching error for the multivariate setting, allowing us to obtain the error bounds for both the Wasserstein-1 distance and the total variation distance in our distribution estimation.

\begin{theorem}
\label{thm:TV2}
For \(\alpha > 0\), there exists a constant \(C = C(\alpha, L, d)\) depending only on \(\alpha\), \(L\), and \(d\) such that 
\begin{equation*}
    \sup_{f \in \mathcal{F}_\alpha} \mathbb{E}\left(\mathrm{TV}(X_0, \widehat{X}_0)\right) \le Cn^{-\frac{\alpha}{2\alpha + d}},
\end{equation*}
where \(\widehat{X}_0\) is given by Algorithm~\ref{algo:DE}.
\end{theorem}

\begin{proof}
    The proof follows the same approach as Theorem~\ref{thm:TV}. The quantity \(\mathbb{E}\left(\mathrm{TV}(X_0, \widehat{X}_0)\right)^2\) is bounded by 
    \[
    \int_{0}^\infty \int_{\mathbb{R}^d} \mathbb{E}\left(|s(x, t) - \widehat{s}(x, t)|^2\right) p(x, t) \, \mathrm{d}x \, \mathrm{d}t  + \frac{1}{\sqrt{n}}.
    \]
    We divide the integral \(\int_{0}^\infty \int_{\mathbb{R}^d} \mathbb{E}\left(|s(x, t) - \widehat{s}(x, t)|^2\right) p(x, t) \, \mathrm{d}x \, \mathrm{d}t\) into three parts: 
    \[
    \int_{1}^\infty \int_{\mathbb{R}^d} \mathbb{E}\left(|s(x, t) - \widehat{s}(x, t)|^2\right) p(x, t) \, \mathrm{d}x \, \mathrm{d}t,
    \]
    \[
    \int_{t_*}^1 \int_{\mathbb{R}^d} \mathbb{E}\left(|s(x, t) - \widehat{s}(x, t)|^2\right) p(x, t) \, \mathrm{d}x \, \mathrm{d}t,
    \]
    and 
    \[
    \int_{0}^{t_*} \int_{\mathbb{R}^d} \mathbb{E}\left(|s(x, t) - \widehat{s}(x, t)|^2\right) p(x, t) \, \mathrm{d}x \, \mathrm{d}t,
    \]
    where \( t_* = n^{-\frac{2}{2\alpha+d}} \). 

    When \(d > 1\), the bound for \(\int_{\mathbb{R}^d} \mathbb{E}\left(|s(x, t) - \widehat{s}(x, t)|^2\right) p(x, t) \, \mathrm{d}x\) is \(n^{-1}t^{-2}\) for \(t > 1\), \(n^{-1}t^{-\frac{d}{2}-1}\) for \(t_* < t < 1\), and \(t^{\alpha-1} + n^{-\frac{2(\alpha-1)}{2\alpha+d}}\) for \(t < t_*\). By plugging in these bounds, we obtain \(\mathbb{E}\left(\mathrm{TV}(X_0, \widehat{X}_0)\right) \lesssim n^{-\frac{\alpha}{2\alpha+d}}\).
\end{proof}

\begin{theorem}
\label{thm:Wasserstein2}
For \(\alpha \ge 1\),  there exists a constant \(C = C(\alpha, L,d)\) depending only on \(\alpha\), \(L\) and \(d\) such that
\[
\sup_{f \in \mathcal{F}_\alpha} \mathbb{E}\left(\mathrm{W_1}(X_0, \widehat{X}_0)\right) \leq \left\{
\begin{array}{ll}
Cn^{-\frac{1}{2}} & \text{if } d = 1, \\
Cn^{-\frac{1}{2}}\log(n) & \text{if } d = 2, \\
Cn^{-\frac{\alpha+1}{2\alpha+d}} & \text{if } d \ge 3.
\end{array}
\right.
\]
where \(\widehat{X}_0\) is given by Algorithm \ref{algo:DE}.
\end{theorem}

\begin{proof}
    Similar to the proof of Theorem~\ref{thm:Wasserstein}, we only need to bound 
    \[
    \sum_{i=0}^N \sqrt{t_{i+1}} \cdot \sqrt{\int_{t_i}^{t_{i+1}} \int_{\mathbb{R}^d} \mathbb{E}(|s(x, t) - \widehat{s}(x, t)|^2) \, p(x,t) \, \mathrm{d}x \, \mathrm{d}t} =: \sum_{i=0}^N\mathrm{E}_i,
    \]
    where \(t_0 = 0\), \(t_{N+1} = 1\), \(t_i = \frac{2^i}{n}\) for \(i \in [1,N]\), and \(N = \lfloor \log_2(n) \rfloor\). Let \(t_* = n^{-\frac{2}{2\alpha+d}}\). When \(d=2\), note that \(\mathrm{E}_i \lesssim n^{-\frac{\alpha-1}{2\alpha+2}} t_{i+1} \lesssim n^{-\frac{1}{2}}\) when \(i \in [0, \lfloor \log_2(nt_*) \rfloor - 1]\) and \(\mathrm{E}_i \lesssim n^{-\frac{1}{2}}\) when \(i \in [\lfloor \log_2(nt_*) \rfloor, N]\). We have
    \[
    \sum_{i=0}^N \mathrm{E}_i \lesssim \frac{N+1}{\sqrt{n}} \lesssim \frac{\log(n)}{\sqrt{n}}.
    \]
    For \(d \geq 3\), we have \(\mathrm{E}_i \lesssim n^{-\frac{\alpha-1}{2\alpha+d}} t_{i+1}\) when \(i \in [0, \lfloor \log_2(nt_*) \rfloor - 1]\) and \(\mathrm{E}_i \lesssim t_i^{\frac{1}{2}-\frac{d}{4}} n^{-\frac{1}{2}}\) when \(i \in [\lfloor \log_2(nt_*) \rfloor, N]\). It holds that
    \[
    \sum_{i=0}^{\lfloor \log_2(nt_*) \rfloor - 1} \mathrm{E}_i \lesssim n^{-\frac{\alpha-1}{2\alpha+d}} t_* \lesssim n^{-\frac{\alpha+1}{2\alpha+d}},
    \]
    \[
    \sum_{i=\lfloor \log_2(nt_*) \rfloor}^{N} \mathrm{E}_i \lesssim t_*^{\frac{1}{2}-\frac{d}{4}} n^{-\frac{1}{2}} \lesssim n^{-\frac{\alpha+1}{2\alpha+d}},
    \]
    which concludes the proof.
\end{proof}